%% file: main.tex
\documentclass[preprint,review,12pt]{elsarticle}

\usepackage{amsthm}
\usepackage{hyperref}
\usepackage{moreverb}
\usepackage{graphicx,subfig,wrapfig}
\usepackage{algorithm}
\usepackage{algpseudocode}
\usepackage{color}
\usepackage{amsmath}
\usepackage{amssymb}
\usepackage{bm}
\usepackage{afterpage}
\usepackage[margin=1in]{geometry}
\usepackage{mathrsfs}
\usepackage{tikz}
\usepackage{pstricks}
\usepackage{diagbox}
\usepackage{arydshln}
\usepackage{multirow}
\usepackage{appendix}
\usepackage{textcomp}

\newtheorem{remark}{Remark}
\newtheorem{myprop}{Proposition}
\newtheorem{definition}{Definition}[section]

\renewcommand{\bm}{\boldsymbol}

\usepackage{tikz}
\usetikzlibrary{shapes,arrows,
positioning, 
calc, 
bayesnet 
}

\input{Definitions}

\makeatletter
\def\ps@pprintTitle{%
  \let\@oddhead\@empty
  \let\@evenhead\@empty
  \let\@oddfoot\@empty
  \let\@evenfoot\@oddfoot
}
\makeatother

\begin{document}

\begin{frontmatter}


\title{Bayesian  multiscale deep generative model for the solution of high-dimensional inverse problems}

\author[label1,label2,label3,label4]{Yingzhi Xia}
\ead{xiayzh@shanghaitech.edu.cn}
 \address[label1]{School of Information Science and Technology, ShanghaiTech University, Shanghai, China}

\author[label2]{Nicholas Zabaras\corref{cor1}}
\ead{nzabaras@gmail.com}
\ead[url]{https://www.zabaras.com/}

\address[label2]{Scientific Computing and Artificial Intelligence (SCAI) Laboratory, University of Notre Dame, 311 Cushing Hall, Notre Dame, IN 46556, USA}
\address[label3]{Shanghai Institute of Microsystem and Information Technology, Chinese Academy of Sciences, Shanghai, China} \address[label4]{University of Chinese Academy of Sciences, Beijing, China}

\cortext[cor1]{Corresponding author}

\begin{abstract}
Estimation of spatially-varying parameters for computationally expensive forward models governed by partial differential equations is addressed. A novel multiscale Bayesian inference approach is introduced based on deep probabilistic generative models. Such generative models provide a flexible representation by inferring on each scale a low-dimensional latent encoding while allowing  hierarchical parameter generation from coarse- to fine-scales. Combining the multiscale generative model with Markov Chain Monte Carlo (MCMC), inference across scales is achieved enabling us to efficiently obtain  posterior parameter samples at various scales. The estimation of coarse-scale parameters using a low-dimensional latent embedding  captures global and notable parameter features using an inexpensive but inaccurate solver. MCMC sampling of the fine-scale parameters is enabled by  utilizing the posterior information in the immediate coarser-scale. In this way, the global features are identified in the coarse-scale with inference of  low-dimensional variables and inexpensive forward computation, and the local features are refined and corrected in the  fine-scale. The developed method is demonstrated with two types of permeability estimation for flow in heterogeneous media. One is a Gaussian random field (GRF) with uncertain length scales, and the other is channelized permeability with the two regions defined by different GRFs. The obtained results indicate that the  method  allows  high-dimensional parameter estimation while exhibiting stability, efficiency and accuracy.

\end{abstract}


\end{frontmatter}


\section{Introduction}
\label{section_intro}
Inverse problems are important but challenging in many fields like geophysics, medical imaging, groundwater flows, and other. They address the estimation of model parameters from partial and noisy observations~\cite{tarantola2005inverse}.  Two  approaches for addressing  inverse problems are typically employed. The deterministic methods convert parameter identification to an optimization problem that involves minimizing the misfit between model predictions and observations. Since limited observations are insufficient to identify  the underlying parameters, regularization methods~\cite{engl1996regularization, zhdanov2002geophysical, hansen2010discrete} are used to address this ill-posed problem. On the other hand, Bayesian inference approaches play a fundamental role in inverse problems allowing us to quantify the uncertainty of the solution and providing natural regularization via  prior knowledge~\cite{stuart2010inverse}. They treat parameters as random variables to highlight the uncertainty in their estimation. The non-uniqueness of the solution is addressed by computing the posterior of the parameters rather than a single point estimate. Variational inference (VI) and Monte Carlo (MC) methods are two main approximation methods to deal with the computation of the intractable posterior distribution. VI~\cite{blei2017variational, bruder2018beyond, jin2012variational,atkinson2019structured} is easy to implement but limited by the family of variational 
distributions. Most Bayesian approaches emphasize Markov Chain Monte Carlo (MCMC) methods that aim to generate samples from the posterior distribution that subsequently are used to produce statistics of the quantities of interest. 
\par
  
MCMC methods have  the appealing property that they are asymptotically exact. Thus, many previous works have studied the MCMC method or its variants for Bayesian inverse problems (BIPs). However, there are two main difficulties for these methods. First, the dimensionality of the spatially-varying parameters can be high (e.g. equal to  the number of grid points)  leading to the so called  curse of dimensionality. Second, these sampling-based approaches require multiple evaluations of the forward model (likelihood evaluation). Each evaluation involves a full forward simulation, which is computationally prohibiting for many practical problems governed by partial differential equations (PDEs).
\par
 
For the first problem identified above, given prior information, parameterization methods are often used to provide a low-dimensional embedding of the unknown spatially-varying parameter. The common method in BIPs is the truncated  Karhunen-Lo\`{e}ve expansion (KLE) for the estimation of Gaussian random fields (GRFs)~\cite{liao2019adaptive, mo2019deep}, where inference is performed over a small number of expansion coefficients. With limitations and strong assumptions on the mean and covariance functions, the KLE cannot reflect the true prior information, and is not a good choice for fields with nontrivial correlation structure. To address this, sparse grid interpolation~\cite{ma2009efficient,wan2011bayesian} and wavelet-based~\cite{ellam2016bayesian} methods have been proposed. However, such methods still have difficulties in the  parameterization of  complex parameters such as multi-modal or non-Gaussian random fields~\cite{laloy2017inversion,mo2019integration}. 
 
 Deep generative models (DGM)~\cite{goodfellow2014generative,kingma2013auto,rezende2015variational} provide a good choice for  parameterization. DGMs are much more flexible and scalable, where the prior information is naturally incorporated into the training data without strong assumptions. Once the DGM is trained, one can sample latent variables from a low-dimensional simple distribution (like a Gaussian), and then generate the spatially-varying parameter using the pre-trained neural network. Many recent studies integrated the DGM-based parameterization method with various inference methods to tackle  non-Gaussian parameter estimation problems, including conditional invertible neural networks~\cite{padmanabha2020solving}, variational autoencoder (VAE) with MCMC or ensemble smoother~\cite{laloy2017inversion,canchumuni2019towards}, generative adversarial network (GAN) with MCMC or Metropolis-adjusted Langevin algorithm (MALA)~\cite{laloy2018training,mosser2020stochastic}, adversarial autoencoder (AAE) with iterative local updating ensemble smoother (ILUES)~\cite{mo2019integration} and so on.
 \par
 
A potential remedy of the requirement of MCMC methods for multiple calls to the forward model solver is to build a surrogate forward model, such as polynomial chaos~\cite{li2014adaptive,marzouk2007stochastic}, Gaussian process \cite{chen2020anova,bilionis2013multi}, or deep neural networks~\cite{li2019hierarchical,zhu2018bayesian,mo2018deep}. However, the surrogate model often introduces epistemic uncertainty that will result in broadening of the posterior for parameter estimation~\cite{bilionis2013solution}. Furthermore, it is still a difficult task to construct an accurate surrogate for forward models with high-dimensional input  using limited data. To reduce the computational burden of the simulation,  multiscale~\cite{ferreira2007multiscale,koutsourelakis2009multi,higdon2002bayesian} and multi-fidelity methods~\cite{peherstorfer2018survey,yan2019adaptive} have been applied to accelerate the Bayesian computation without sacrificing accuracy. The two-stage MCMC~\cite{efendiev2006preconditioning} designed a preconditioned Metropolis-Hastings algorithm to improve the acceptance rate in the fine-scale model. Inspired by the multilevel Monte Carlo, Multilevel MCMC methods~\cite{dodwell2015hierarchical, hoang2013complexity, beskos2017multilevel, cui2019multilevel} are proposed for BIPs to accelerate the estimation of the posterior distribution. All these methods are indeed promising for BIPs  by leveraging the advantages of the accurate fine-scale model and the efficiency of the coarse-scale model.
\par

In this work,  we propose a multiscale deep generative model (MDGM) exploiting the multiscale nature of the parameter of 
interest. This extends  existing DGMs and allows us to generate parameters on various scales with different discretization/resolution. Since GANs are notorious on training stability and mode collapse, and flow-based models~\cite{dinh2014nice,tang2020dde} require an identical-dimensional latent space to the parameter, we derive the MDGM based on VAE. In the MDGM, we design a specific latent space that includes two latent variables, a low-dimensional latent variable that controls global and salient features and a higher-dimensional latent variable that defines  local and detailed features. Utilizing the hierarchical representation of the parameter and latent spaces, the multiscale inference is performed in the low-dimensional latent space rather than the original parameter space. This allows us to explore the posterior of the parameter from coarse- to fine-scales with a significant computational saving. Once most of the salient features are identified in the coarse-scale  using a computationally inexpensive coarse-solver, the fine-scale estimation requires only few fine-scale simulations to refine the coarse-scale parameter estimation.

The main contributions of this work are summarized as follows. 
(1) Based on the vanilla VAE, we extend and derive the MDGM, which can generate spatial parameters at various scales with an appropriately designed latent space. (2) The proposed multiscale inference method performs efficiently inference across scales based on the MDGM. 
(3) A flexible scheme allows efficient estimation of rough/global parameter features with  coarse-scale inference and parameter refinement with fine-scale inference. 
(4) The proposed method is demonstrated in Gaussian and non-Gaussian inversion tasks.
\par

The rest of the paper is organized as follows. Section~\ref{sec:Definition} provides the definition of the inverse problem and addresses the limitations of standard Bayesian approaches for distributed parameter estimation.  Section~\ref{sec:Hierarchicalparameterization} introduces the  generation of the multiscale training datasets.  The big picture of the multiscale estimation problem using hierarchical generative models is addressed in  Section~\ref{sec:Model_def}.  The one-scale and multiscale generative models are 
derived in Sections~\ref{sec:vdgm} and~\ref{sec:MDGM}, respectively. The Bayesian inversion using the multiscale generative model is discussed in Sections~\ref{sec:MH} and~\ref{sec:MultiscaleMH}. Section~\ref{sec:Examples} presents the results of various numerical examples in the estimation of Gaussian and channelized permeability in porous media flows and Section~\ref{sec:Conclusions} summarizes this work.

\section{Problem Definition}
\label{sec:Definition}
\subsection{Bayesian inverse problems}
In this section, we introduce the inverse problems of interest and  briefly discuss the limitations of standard Bayesian inference approaches to inverse problems.  
We consider a spatially-varying parameter $\bx(\bm{s})$   usually represented as a random field $\bx(\bm{s},\omega)$, where $\bm{s}$ is spatial location in the domain $\mathcal{S}$ and $\omega$ is a random event in the sample space $\bm{\Omega}$. This random field is discretized by a random vector $\bx\in \mathbb{R}^{M}$ using standard finite element or finite difference discretization approaches. In our inverse problem setting, $\bx(\bm{s})$, will be considered as our primary quantity of interest.  
\par 

Let us consider a physical system governed by PDEs in a given spatial domain. We assume $\bx(\bm{s},\omega)$ to be an input parameter (e.g. material property) of this model.   Of interest to this work are distributed properties with multiscale features. The forward model concerning  this physical system is usually considered as a function $\mathcal{F} : \mathbb{R}^{M} \rightarrow  \mathbb{R}^{D}$, which maps the unknown parameters $\bx$ to the observable output $\mathcal{D}_{obs} \in \mathbb{R}^{D}$ with a measurement  noise $\bm{\xi} \in \mathbb{R}^{D}$:

\begin{equation}
    \mathcal{D}_{obs} = \mathcal{F}(\bx) + \bm{\xi}.
\end{equation}
The inverse problem is to infer the unknown parameters $\bx$ based on   these noisy data $\mathcal{D}_{obs}$.  In the particular problem we will focus in 
Section~\ref{sec:Examples}, our goal is to estimate the permeability
field in a porous media flow using pressure measurements.

\par

Without prior information about the measurement system and/or model evaluation, we assume that  $\bm{\xi}$ is a zero-mean Gaussian noise with covariance matrix $\bm{\Sigma}$, i.e.,  $\bm{\xi} \sim \mathcal{N}\left(\bm{0}, \bm{\Sigma}\right)$. Since often $dim(\mathbb{R}^{D})\ll dim(\mathbb{R}^{M})$, the inverse problem  is highly ill-posed and  identification of the parameter $\bx$ is  highly-sensitive to this noise. The Bayesian paradigm~\cite{stuart2010inverse} provides a general and natural way to treat  the unknown parameter $\bx$ as random variable to highlight the uncertainty in the inference process. Given the observation data $\mathcal{D}_{obs}$, one calculates the posterior probability $\pi(\bx | \mathcal{D}_{obs})$ via  Bayes' formula as follows: 

\begin{equation}
    \pi(\bx| \mathcal{D}_{obs})=\frac{\mathcal{L}_e(\mathcal{D}_{obs} | \bx) \pi(\bx)}{\int \mathcal{L}_e(\mathcal{D}_{obs} | \bx) \pi(\bx) \mathrm{d} \bx},\label{eq: Bayes rule}
\end{equation}
where $\pi(\bx)$ is the prior distribution, and $\mathcal{L}_e(\mathcal{D}_{obs} | \bx)$ is the likelihood function which evaluates the discrepancy between the forward predictions and observations. For the assumed case of Gaussian noise, we can define the likelihood function as

\begin{equation}
   \mathcal{L}_e(\mathcal{D}_{obs} | \bx) \propto \text{exp}\left(-\frac{1}{2}\left(\mathcal{D}_{obs}-\mathcal{F}(\bx)\right)^T \bm{\Sigma}^{-1} \left(\mathcal{D}_{obs}-\mathcal{F}(\bx)\right)\right).
\end{equation}
As the parameter  $\bx$ of interest is  high-dimensional,  the normalization constant in Eq.~\eqref{eq: Bayes rule} involves computing a high-dimensional integral that is often an intractable process. Thus approximate inference for the posterior $\pi(\bx | \mathcal{D}_{obs})$ is performed using the unnormalized density, i.e.,

\begin{equation}
    \pi(\bx | \mathcal{D}_{obs}) \propto \mathcal{L}_e(\mathcal{D}_{obs} | \bx) \pi(\bx).
    \label{eq:PosteriorOfX}
\end{equation}

\subsection{Multiscale inference with MDGM}
Without a closed-form expression, the posterior distribution in Eq.~\eqref{eq:PosteriorOfX} must be computed numerically.  To this end, MCMC~\cite{liao2019adaptive,xu2020gaussian} or other approximation methods like Ensemble Kalman filter (EnKF)~\cite{mo2019integration,mo2018deep} are often employed. However,  there are still two main difficulties for these methods. MCMC and EnKF implementations will often fail to directly approximate the posterior of the high-dimensional spatially-varying parameter $\bx$. Moreover, for complex parameters (e.g. channelized permeability), the prior information cannot be easily cast as an explicit probability distribution. However, one often has access to a historical dataset $\bX\equiv\{\bxi\}_{i=1}^{N}$~\cite{mo2019integration,liu2019deep,tang2020deep} where $\bxi$ can be seen as samples from the underlying prior distribution $\pi(\bx)$. One could use $\bX$ to approximate $\pi(\bx)$ with its empirical measure. However, in this work, we will use this dataset to approximate the prior distribution with a generative model as follows:

 \begin{equation}
 p(\bx|\btheta) = \int p_{ \btheta}(\bx | \bz) p(\bz) d \bz,
 \label{eq:marginal_target}
 \end{equation}
 where $p(\bz)$ is a simple distribution (e.g. Gaussian) for the latent variable $\bz \in \mathbb{R}^{d}$, and $p_{ \btheta}(\bx|\bz)$ is a generative model parameterized by $ \btheta$ (decoder). In a DGM like a VAE, one can choose a Gaussian distribution $\mathcal{N}(\bm{0}, \bm{I})$ for $p(\bz)$, and $\mathcal{N}(\mu_{ \btheta}(\bz), \sigma^2 \bm{I})$ for $p_{ \btheta}(\bx|\bz)$, where $\mu_{ \btheta}(\bz)$ is the  output of the decoder neural network, and $\sigma$ is a  hyperparameter that does not depend on the latent variable $\bz$.   It is common practice in the literature~\cite{laloy2017inversion,mo2019integration,dorta2018structured,dorta2018training} to  ignore the noise and approximate the density
$p_{ \btheta}(\bx | \bz)$ with the point estimate $\mu_{ \btheta}(\bz)$.   Once the model is trained, we can thus define a mapping from the latent space to the original parameter space, i.e. $\bx = \mu_{ \btheta}(\bz)$.
 
The  parameters $\btheta$ of the generative model $p_{ \btheta}(\bx|\bz)$ can be computed using the given training dataset $\bX$ by minimizing the Kullback-–Leibler (KL) divergence $ D_{KL}\left(\pi(\bx)||p(\bx| \btheta)\right)$~\cite{schoberl2019predictive}, where $\bxi \stackrel{i.i.d}{\sim} \pi(\bx)$. This leads to the equivalent problem of maximizing the marginal log-likelihood:
 \begin{eqnarray}
 \log p(\bX| \btheta) &=& \sum_{i=1}^N \log p(\bxi| \btheta) \nonumber\\
 &=&\sum_{i=1}^N \log \ \int p_{ \btheta}(\bxi | \bz^{(i)}) p(\bz^{(i)}) d \bz^{(i)}.
 \label{eq:maginal log_likelihood}
 \end{eqnarray}
The above marginalization is potentially very difficult to compute involving an intractable integration. Using Expectation-Maximization is also intractable as that will require the posterior $p_{ \btheta}(\bz|\bx)$ that is also computationally intractable. The marginal likelihood for the dataset $\{\bxi\}_{i=1}^{N}$ can be reformulated using a variational density $q_{\bphi}(\bz^{(i)}|\bxi)$ (encoder) parameterized by $\bphi$. The details of these calculations will be given in Section~\ref{sec:vdgm}.
 
 To approximate the posterior distribution in Eq.~\eqref{eq:PosteriorOfX} using the MCMC method, we are interested to generate realizations sampled from the underlying prior distribution $\pi(\bx)$. 
 To sample realizations from $\pi(\bx)$ using the generative model, one can sample $\bz_i$ from the simple and low-dimensional distribution $p(\bz)$ and then obtain the realization $\bxi$ using the decoder model $\mu_{ \btheta}(\bz)$. 
 Alternatively, instead of approximating the posterior $\pi(\bx | \mathcal{D}_{obs})$ in Eq.~\eqref{eq:PosteriorOfX}, one can instead evaluate the low-dimensional posterior $p(\bz | \mathcal{D}_{obs})$ 
using the following unnormalized density:

\begin{equation}
    p(\bz | \mathcal{D}_{obs}) \propto \mathcal{L}_e(\mathcal{D}_{obs} | \bz) p(\bz),
    \label{eq:PosteriorOfZ}
\end{equation}
where $p(\bz)$ is an explicit distribution, and the likelihood can be evaluated using the decoder model $\mu_{ \btheta}(\bz)$ and the forward model $\mathcal{F}(\bx)$. The evaluation of  the posterior of the low-dimensional latent variable $\bz$ using MCMC or EnKF is computationally tractable~\cite{laloy2017inversion,mo2019integration}.

To further improve the efficiency of the inference process, we will introduce a multiscale version of the above highlighted generative model to perform inference in each scale $l=1,2,\ldots,L$ from the coarsest-scale ($l=1$) to the desired finest-scale ($l=L$). This multiscale scheme based on the MDGM  contains a hierarchical simple distribution $p(\bz_l)$ at each scale $l$ and  a conditional distribution $p_{ \btheta_l}(\bx_l | \bz_l)$ that can generate the spatially-varying parameter $\bx$ in each scale. Correspondingly, one can assess the posterior $p(\bz_l | \mathcal{D}_{obs})$ using MCMC with $p(\bz_l)$, $\mu_{ \btheta_l}(\bz_l)$, and $\mathcal{F}_l(\bm{x_l})$. The details of this multiscale model are given next.
 
\section{Methodology}
\label{sec:Methodology}
\subsection{Multiscale dataset}
\label{sec:Hierarchicalparameterization}
Our physical systems of interest are governed by a system of PDEs, and the spatially-varying property of interest is a material property appearing e.g. in the constitutive equations. The forward problem defines the well-posed solution of the PDEs (with some boundary conditions) given appropriate material properties. Such problems are often solved in a discretized fashion with  finite element or finite difference or spectral  approximations for different levels of discretization of the spatial domain $\mathcal{S}$. In this work, we are interested in a hierarchical parameterization of the spatially-varying parameters $\bx_l$  with different spatial discretization or resolutions at each scale $l$. If the forward model is performed in the $2$-D space, the parameter random fields $\bx_l$ at the $l$-th scale are treated as images, e.g.  $\bx_l \in \mathbb{R}^{H_l \times W_l}$ ($M_l = H_l \times W_l$), where $H_l, W_l$ denote the number of the pixels in the horizontal and vertical directions, respectively.

For notational convenience, we assume that the finest scale parameters $\bx_L$ represent our ``true parameter model".
The noisy observations $\mathcal{D}_{obs}$ in our numerical studies are taken from this discretization level. In the inverse problem of interest, our task is to compute $\bx_L$ given a finite number of observations. For the solution of this inverse problem,  prior knowledge can provide useful information for $\bx_L$ before any observations. As prior information for our model, we assume that we are given a dataset $\bX=\{\bxi_{L}\}_{i=1}^{N}$. To obtain images for training the generative model at different discretization levels, we will need to obtain a multiscale training dataset. 

This can be accomplished by upscaling the fine-scale training dataset~\cite{lu2016improved,wen2003upscaling,wen1996upscaling}. 
For example, with an upscaling (deterministic) operator $\mathcal{U}:\mathbb{R}^{M_l} \rightarrow \mathbb{R}^{M_{l-1}}~(M_{l-1}\ll M_l)$, where $M_l$ and $M_{l-1}$ are the dimensions of $\bx_l$ and $\bx_{l-1}$, respectively.
The dataset $\{\bxi_{l-1}\}_{i=1}^{N}$ in the coarse-scale $(l-1)$ is obtained by

 \begin{equation}
     \bxi_{l-1} = \mathcal{U} (\bxi_{l}). \label{eq:deterministic upscaling}
 \end{equation}
The datasets $\{\bxi_{l}\}_{i=1}^{N},~l=1,2,\dots,L$ in different scales are obtained by adopting recursively $\mathcal{U}$ in Eq.~\eqref{eq:deterministic upscaling} starting with the finest-scale $l=L$. we assume $\bxi_{l}$ is sampled from $l$-th scale underlying prior distribution $\pi_l(\bm{x}_l)$.  The operator $\mathcal{U}$ used in this paper is deterministic, which    leads to a one-to-one correspondence between the elements in    $\{\bxi_{1},\dots,\bxi_{L-1},\bxi_{L}\}_{i=1}^{N}$. One can choose different operators $\mathcal{U}$ depending on the particular parameter of interest. In this paper, we employ the arithmetic average\footnote{http://www.epgeology.com/static-modeling-f39/how-upscale-permeability-t6045.html}:

\begin{equation}
\begin{aligned}\bx_{l-1}(e) = \frac{1}{n_e}\sum_{i=1}^{n_e} \bx_l(e_i),
    \end{aligned}
    \label{arithmetic average}
\end{equation}
where $n_e$ denotes the number of elements in the fine-scale $l$ corresponding to one element in the coarse-scale $(l-1)$. The value at the coarse-grid element $e$ is the mean of the values in the spatially corresponding elements $e_i$ in the fine-scale. Spatial correspondence between two adjacent  scales with $50\%$ coarsening in each direction is illustrated in Fig.~\ref{upscaling_diagram}.

\begin{figure}[h]
	\centering
	\includegraphics[width=0.6\linewidth, height = 0.22\linewidth]{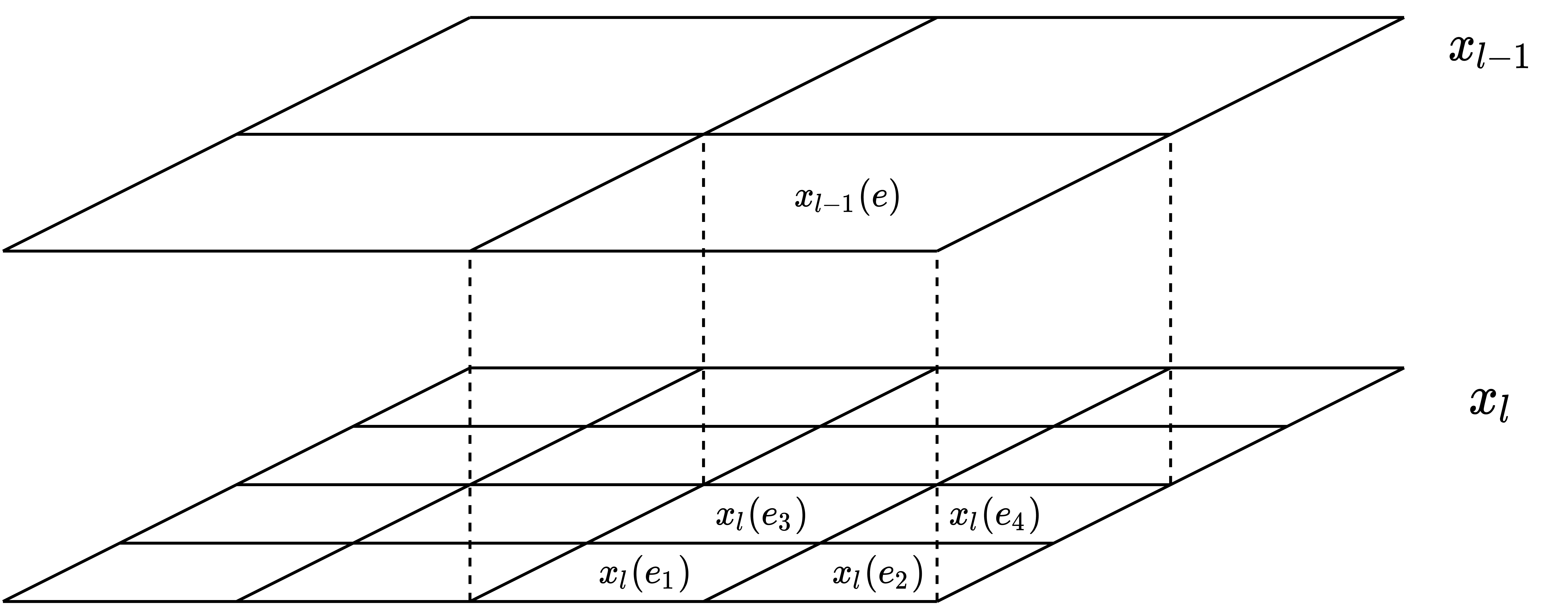}
	\caption{Illustration of the spatial correspondence in deterministic upscaling~\cite{ferreira2007multiscale}. The parameter $\bx_{l-1}(e)$  in the coarse-scale element $e$ is  equal to $\mathcal{U}\left(\bx_l(e_i)\right)$, where $e_i$ are the  spatially corresponding fine-scale elements  to the coarse-element $e$.  
	The number of fine- to coarse-elements in each direction is
	proportional to $\frac{h_{l-1}}{h_l}$, where $h_{l-1}$, $h_l$ are  the mesh sizes in the coarse- and fine-scales, respectively. }
	\label{upscaling_diagram}
\end{figure} 

An example illustrating this deterministic upscaling for channelized permeability using Eq.~\eqref{arithmetic average} is given in Fig.~\ref{upscaling}. The coarse-scale image provides a blurry representation of the fine-scale image but overall its features are consistent with those of the fine-scale image. It can be noticed that the coarse-scale image manifests itself with a checkerboard pattern that misses a lots of local information. 
\par 
 
\begin{figure}[h]
	\centering
	\includegraphics[width=0.8\linewidth, height = 0.28\linewidth]{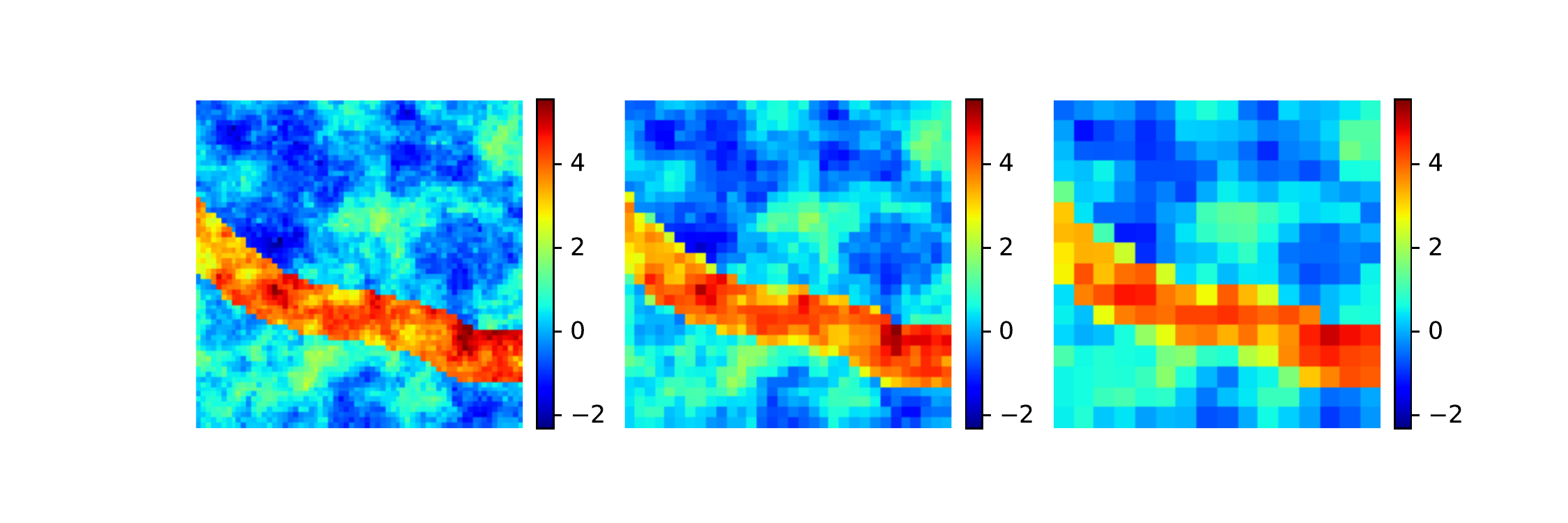}
	\caption{Upscaling channelized log-permeability samples ($3$ scales) using the upscaling technique in Fig.~\ref{upscaling_diagram} and Eq.~\eqref{arithmetic average}. From left to right: (a) original and the finest-grid $\bx_3 \in \mathbb{R}^{64 \times 64}$ realization with resolution $64 \times 64$, (b) the 
	coarser-scale $\bx_2$ with resolution $32 \times 32$ after applying the operator in Eq.~\eqref{arithmetic average} with $n_e=4$ on $\bx_3$,  (c) the coarsest-scale $\bx_1$ with resolution $16 \times 16$ after applying the operator in Eq.~\eqref{arithmetic average} with $n_e=4$ on $\bx_2$.}
	\label{upscaling}
\end{figure}

With the  scales of interest  pre-determined and the training dataset defined at each scale, we are ready to train the MDGM and perform inference on each scale using the proposed hierarchical multiscale framework.

\subsection{Model specification}
\label{sec:Model_def}
Given the training dataset $\{\bxi\}_{i=1}^{N}$, we seek to learn a DGM that can approximate   $\pi(\bx)$ with  Eq.~\eqref{eq:marginal_target}. The DGM involves the original parameter space and the latent space and the mappings between these spaces. For any $\bx$ sampled from the underlying distribution $\pi(\bx)$, the corresponding $\bz$ is sampled from the conditional distribution  $q_{\bphi}(\bz|\bx)$, where $q_{\bphi}(\bz|\bx)$ is called the recognition model or probabilistic encoder model, and $\bphi$ are its model parameters. In the reverse direction, one can sample a $\bz$ from a simple and low-dimensional distribution $p(\bz)$, and obtain the corresponding $\bx$ from the generative model or probabilistic decoder model $p_{ \btheta}(\bx|\bz)$. 
\
\begin{definition} [Probabilistic Generative Model]
\label{def:GenerativeModel} Given a set of
training input data $\{\bxi\}_{i=1}^{N}$, where $\bxi \sim \pi(\bx) $, select an appropriate distribution $p(\bz)$ for the latent variable $\bz$ and learn the models $p_{ \btheta}(\bx | \bz)$ and $q_{\bphi}(\bz|\bx)$, respectively, such that $\pi(\bx)$ can be approximated using  Eq.~\eqref{eq:marginal_target},  where $ \btheta$ and $\bphi$ denote the parameters of the generative and recognition models, respectively.\end{definition}
\par

For high-dimensional inversion tasks, direct inference in the fine-scale  is prohibited due to the  computational  cost of the forward model. In addition,  inference of the latent parameters $\bz$ that lead to good estimates of $\bx$ through a generative model requires a high-dimensional $\bz$. This  will lead to long exploration costs for MCMC and  requires a large number of forward model evaluations. To this end, given the multiscale dataset as discussed in Section~\ref{sec:Hierarchicalparameterization}, we propose a multiscale scheme for  posterior estimation by introducing a hierarchy of generative models from coarse- to fine-scales. In this scheme, the coarse-scale generative models have  low-dimensional latent spaces. This together with inexpensive forward model evaluations in the coarse-scales would allow MCMC to explore the posterior  in  coarse-scales with much reduced cost. The computational savings can be even higher if the latent representation on a given scale utilizes the latent information that was inferred in the immediately coarser-scale. 
\par

In our construct, the latent variables $\bz_l$ at level $l$ of the hierarchy are given as $\bz_l = (\bz_{l-1},\bz_l^{\star})$, where the latent variables $\bz_{l-1}$ and $\bz_l^{\star}$ are encoded from the coarse-scale parameter $\bx_{l-1}$ and the fine-scale parameter $\bx_{l}$, respectively.  The latent variable $\bz_{l-1}$ can generate $\bx_{l-1}$ through the generative model at scale $(l-1)$. It also impacts the generation of $\bx_l$ in the fine-scale $l$ by way of dominating its salient features since $\bz_{l-1}$ captured the information from the immediately coarser-scale. In this setting with $\bz_{l-1}$ encoded from $\bx_{l-1}$, it is anticipated that $\bx_l$  generated by $p_{ \btheta_l}(\bx_l |\bz_{l-1},\bz_l^{\star})$ would sustain most of the features of  $\bx_{l-1}$. We extend the Definition~\ref{def:GenerativeModel} to a multiscale scenario as follows.

 \par

\begin{definition}[Multiscale Deep Generative Model]\label{def:MultiscaleGenerative}  Given a set of training input data $\{\bxi_{l}\}_{i=1}^{N},   l=1, 2, \dots, L$, select  appropriate distributions $p(\bz_l)$, where $\bz_{l}=(\bz_{l-1}, \bz_{l}^{\star})$ are 
the latent variables 
at scale $l$, and learn the models $p_{ \btheta_1}(\bx_1 | \bz_{1})$, $p_{ \btheta_2}(\bx_2 | \bz_1,\bz_2^{\star})$, $\dots$, $p_{ \btheta_l}(\bx_l | \bz_{l-1},\bz_l^{\star})$ and $q_{\bphi_1}( \bz_{1}|\bx_1)$, $q_{\bphi_2}( \bz_1,\bz_2^{\star}|\bx_2 )$, $\dots$, $q_{\bphi_l}( \bz_{l-1},\bz_l^{\star}|\bx_l)$ recursively,  such that $\pi_l(\bx_l)$ can be approximated using  Eq.~\eqref{eq:marginal_target} in each scale. Here,  $ \btheta_l$ and $\bphi_l$ denote the parameters of the generative   and recognition models at scale $l$, respectively.
\end{definition}
 
 Once the MDGM is established,  the parameters $\bx_{l}$ are encoded by the latent variable $\bz_l$ so that one can perform  inference of $\bz_l$. Given the prior distribution $p(\bz_l)$,
 the decoder model $\mu_{ \btheta}(\bz_l)$, the forward model $\mathcal{F}_{l}$, and observations $\mathcal{D}_{obs}$, 
 inference of the posterior of  $\bz_l$  is performed as follows:

  \begin{equation}
     p(\bz_l | \mathcal{D}_{obs}) \propto \mathcal{L}_e(\mathcal{D}_{obs} | \bz_l) p(\bz_l).
     \label{eq:l-th posterior of z}
 \end{equation}

Note that MCMC converges and captures prominent and valuable features quickly in the coarse-scale $(l-1)$ since $\bz_{l-1}$ is  low-dimensional and the forward model $\mathcal{F}_{l-1}$ is less expensive in comparison to $\mathcal{F}_{l}$. For an efficient Bayesian inference at each scale $l$,  we are interested in using the posterior distribution at coarse-scale  $(l-1)$ to provide an informative prior information or improve sampling efficiency in the next finer-scale.   This  avoids relying completely on  inference in a high-dimensional latent space where direct computation of fine-scale details would increase the model complexity. The purpose of inference on fine-scale is to correct the details rather than run long exploration for capturing all appropriate features. A related idea was implemented earlier using hierarchical structured sparse grids  in~\cite{wan2011bayesian}. In summary, the inverse problem is divided into a multiscale posterior estimation, with the inference of parameters proceeding from coarse- to fine-scale. The definition of the multiscale posterior estimation problem is given next.

 \begin{definition}[Multiscale Posterior Estimation] \label{def:MultiscalePosterior}  Given observations $\mathcal{D}_{obs}$, forward models $\mathcal{F}_l$, probabilistic encoder model $p_{\bphi_l}(\bz_l|\bx_l)$, decoder models $\mu_{ \btheta}(\bz_l)$, and prior distributions $p(\bz_l)$ on different scales $l~(l = 1,2,\ldots,L)$, explore the posterior distribution $p(\bz_l|\mathcal{D}_{obs})$ in Eq.~\eqref{eq:l-th posterior of z} recursively from coarse- to fine-scales by using MCMC or other posterior   modeling techniques. 
\end{definition}

\subsection{Probabilistic generative model}\label{sec:vdgm}
The MDGM is used to generate parameters with different discretization/resolution. We construct such a model based on the variational autoenoder (VAE). The coarsest generative model that involves a single-scale ($l=1$) is vanilla VAE. It employs the dataset $ \{ \bx_1^{(i)}\}_{i=1}^{N}$ (generated as discussed in Section~\ref{sec:Hierarchicalparameterization}) sampled from the underlying distribution $\pi(\bx_1)$, i.e. $\bx_1^{(i)} \stackrel{i.i.d}{\sim} \pi(\bx_1)$. We consider below the probabilistic generative model on a single-scale before  deriving the multiscale formulation in Section~\ref{sec:MDGM}. For simplicity of the notation, we drop the subscript $1$ in the equations below even though this model will be used in the scale $l=1$. 

Given the training dataset $ \{ \bx^{(i)}\}_{i=1}^{N}$, one can introduce a variational family $q\inphi(\bz|\bx)$ to convert the intractable computation of maximizing the marginal log-likelihood in Eq.~\eqref{eq:maginal log_likelihood} into an optimization problem, where $\bphi$ denotes the model parameters. It can be written as follows:

\begin{align}
	\label{eqn:lowerbound}
	\log \ p(\bX| \btheta) &= \sum_{i=1}^N \log p(\bxi | \btheta) \nonumber \\
	&= \sum_{i=1}^N \log \int  p\inth(\bxi|\bzi) p\inth(\bzi)~d\bzi \nonumber \\
	&= \sum_{i=1}^N \log \int  q\inphi(\bzi|\bxi) \frac{p\inth(\bxi|\bzi)  p\inth(\bzi)}{q\inphi(\bzi|\bxi)}~d\bzi  \nonumber \\
	&\geq \sum_{i=1}^N \underbrace{\int   q\inphi(\bzi|\bxi) \log \frac{p\inth(\bxi|\bzi)  p\inth(\bzi)}{q\inphi(\bzi|\bxi)}~d\bzi }_{\mathcal L (\btheta,\bphi; \bxi)}, 
\end{align}
where the last step is the application of Jensen's inequality. The above lower bound is called the variational lower bound. For a given training dataset, one can maximize the variational lower bound rather than the marginal log-likelihood. This is an optimization problem with respect to   model parameter $\btheta$ and $\bphi$. The variational lower bound in Eq.~(\ref{eqn:lowerbound}) can be written as follows,

\begin{align}
	\mathcal L(\btheta, \bphi; \bX) = &\sum_{i=1}^N \mathbb E_{q\inphi(\bzi|\bxi)} [\log p\inth(\bxi,\bzi) - \log q\inphi(\bzi|\bxi) \nonumber] \nonumber \\
	= &\sum_{i=1}^N \mathbb E_{q\inphi(\bzi|\bxi)} [\log p\inth(\bxi|\bzi) ] -\sum_{i=1}^N D_{KL}\left(q\inphi(\bzi|\bxi) || p\inth(\bzi)\right).
	\label{eqn:decomposedlowerbound}
\end{align}
The minimization of the $- \mathcal{L}( \btheta,\bphi)$  balances the optimization of both the  recognition  and  generative models. Thus the recognition model parameters $ \bphi$ are learned jointly with the generative model parameters $\btheta$~\cite{kingma2013auto}. The graphical model is shown in Fig.~\ref{fig:pgm_vae}.

\begin{figure}[h]
\centering
\resizebox{0.2\textwidth}{!}{%
      \tikz{ %
        \node[obs] (x) {$\bxi$} ; %
        \node[latent, left=of x] (z) {$\bzi$};
        \node[below=of z] (phi) {$\bphi$};
        \node (zxmiddle) at ($(z)!0.5!(x)$) {};
        \path let \p1 = (zxmiddle), \p2 = (phi) in node (theta) at (\x1,-\y2) {$\btheta$};
        \plate[inner sep=0.25cm, xshift=-0.12cm, yshift=0.12cm] {plate3} {(x) (z)} {$N$}; %
        \edge {z} {x};
        \edge {theta} {z};
        \edge {theta} {x};
        \edge[dashed] {phi} {z};
        \draw[dashed,->, bend left] (x) edge (z);
}
}
\caption{The directed graphical model for the probabilistic model~\cite{kingma2013auto}. The latent variable  $\bzi$ of each configuration $\bxi$ is obtained by the probabilistic recognition model $q\inphi(\bzi|\bxi)$. The variational approximation is indicated with dashed edges and the generative model $p\inth(\bx|\bz)p(\bz)$ with solid edges. }
\label{fig:pgm_vae}
\end{figure}

To evaluate $\mathcal{L}( \btheta,\bphi)$ in Eq.~\eqref{eqn:decomposedlowerbound}, there are   three probability distributions to be identified. As mentioned in Definition~\ref{def:GenerativeModel}, we shall select appropriate simple distributions for the latent variables $\bz$. For example, in this paper, we let $p(\bz) \sim \mathcal{N}(\bm{0}, \bm{I})$. 
The $\mathbb{E}_{\bz^{(i)} \sim q_{\bphi}(\bz^{(i)} | \bxi)}[\log p_{ \btheta}(\bxi | \bz^{(i)})]$ in Eq.~\eqref{eqn:decomposedlowerbound} is the expected log-likelihood. It encourages the reconstructed data $\hat{\bx}$ by the decoder to approximate the original data $\bx$. We assume $p_{ \btheta}(\bx| \bz)$ is modeled by a Gaussian distribution $\mathcal{N}(\mu(\bz), \sigma^2 \bm{I})$, where the mean $\mu(\bz)$ is the  output of a decoder neural network and  $\sigma$ is a constant hyperparameter that does not depend on the decoder so that it can be ignored during optimization. Let $\hat{\Sigma} = \sigma^2 \bm{I}$, the first term  in Eq.~\eqref{eqn:decomposedlowerbound} then takes the form based on the minibatches:

\begin{eqnarray}
&&    \sum_{i=1}^N \mathbb{E}_{\bz^{(i)} \sim q_{\bphi}(\bz^{(i)} | \bxi)}   (\log p_{ \btheta}(\bxi | \bz^{(i)}))   \nonumber \\ &=&\sum_{i=1}^N\mathbb{E}_{\bz^{(i)} \sim q_{\bphi}(\bz^{(i)} | \bxi)}\left[-\frac{1}{2} \left(\bxi-\mu(\bz^{(i)})\right)^{T} \hat{\Sigma}^{-1} \left(\bxi-\mu(\bz^{(i)})\right) \right]+constant \nonumber \\
    &\propto& -\sum_{i=1}^N\mathbb{E}_{\bz^{(i)} \sim q_{\bphi}(\bz^{(i)} | \bxi)} \left[ \left(\bxi-\mu(\bz^{(i)})\right)^{T}\left(\bxi-\mu(\bz^{(i)})\right)\right] \nonumber \\
    &\approx&   -\frac{N}{n}\frac{1}{m} \sum_{i=1}^{n}\sum_{j=1}^{m} \left\|\bxi-\mu(\bz^{(i,j)})\right\|^{2}, \quad \bz^{(i,j)} \sim q_{\bphi}(\bz^{(i,j)}| \bxi), 
    \label{eq:recon_loss}
\end{eqnarray}
where $n$ denotes the number of training samples of $\bx$ (also referred to as the batch size in the training of deep neural networks). For each epoch, there are $\frac{N}{n}$ minibatches, each batch uniformly sampled from the dataset $\{\bxi\}_{i=1}^{N}$. $m$ is the number of $\bz$ samples from $q_{\bphi}(\bz| \bx)$ for an expectation approximation using the Monte Carlo method. One can also refer to Eq.~\eqref{eq:recon_loss} as the reconstruction error.

For the $D_{KL}(q_{\bphi}(\bz^{(i)} | \bxi) \| p(\bz^{(i)}))$ in Eq.~\eqref{eqn:decomposedlowerbound}, $q_{\bphi}(\bz^{(i)} | \bxi)$ is taken as a Gaussian distribution, where the mean $\hat{\mu}(\bx)$ and the variance $\hat{\sigma}^2(\bx)$ are outputs of the encoder network. The KL-divergence   can be analytically computed  when both   distributions are Gaussian~\cite{kingma2013auto}. The KL-divergence works as an objective function for the optimization problem with respect to the encoder parameters $\bphi$. Based on the minibatches, it can be written as: 

\begin{equation}
    \sum_{i=1}^N D_{KL}(q_{\bphi}(\bz^{(i)} | \bxi) \| p(\bz))= \frac{N}{2n}\sum_{i=1}^{n}\sum_{k=1}^{d}\left(\hat{\mu}_k^{2}(\bxi)+\hat{\sigma}_k^{2}(\bxi)-\log \hat{\sigma}_k^{2}(\bxi)-1\right), 
     \label{eq:KL_loss}
\end{equation}
where $\hat{\mu}_k(\bx)$ and $\hat{\sigma}_k(\bx)$ denote the $k$-th element of the mean and standard deviation, respectively, which are outputs of the encoder network. The results of Eqs.~\eqref{eq:recon_loss} and~\eqref{eq:KL_loss}  summarize the objective function for optimization of the encoder and decoder neural networks.

\par

Inspired from $\beta$-VAE~\cite{higgins2017beta},  we slightly modify the underlying loss function. This modification of the VAE results by adding an extra hyperparameter $\beta$ to the KL divergence.   This hyperparamter can constrict  the  capacity of the latent bottleneck and encourage a disentangled representation. We would like the individual dimensions of the  latent variable $\bz$ to be interpretable or to correspond to some features of parameters thus disentangling the true variation of data~\cite{higgins2017beta,chen2016infogan}.  A good interpretability and factorized representation for the latent variable will facilitate the feature exploration in MCMC. Note that~\cite{higgins2017scan} has shown that $\beta$-VAE is an important method in disentangled representation for learning  the compositional and hierarchical visual concept. Choosing an appropriate hyperparameter $\tilde{\beta}$ for loss function is important in MDGM. We can define the loss function $\tilde{\mathcal{L}}\equiv-\frac{n}{N}\mathcal{L}( \btheta,\bphi)$ in each training iteration, the loss function for the single-scale probabilistic generative model  as follows:

\begin{equation}\tilde{\mathcal{L}}( \btheta,\bphi) = \frac{\tilde{\beta} }{2}\sum_{i=1}^{n}\sum_{k=1}^{d}\left(\hat{\mu}_k^{2}(\bx^{(i)})+\hat{\sigma}_k^{2}(\bx^{(i)})-\log \hat{\sigma}_k^{2}(\bx^{(i)})-1\right)+\frac{1}{m} \sum_{i=1}^{n}\sum_{j=1}^{m} \left\|\bx^{(i)}-\mu(\bz^{(i,j)})\right\|^{2}.
    \ \label{eq:vae_loss}
\end{equation}

\begin{remark}
In the vanilla VAE, the objective function considers the reconstruction error and the KL-divergence to be of  equivalent importance. The hyperparameter $\beta$ is introduced to break this balance. More specifically, high values of $\beta$ put more emphasis on the latent space approximation than on the reconstruction, expediting the learning of notable feature variations but bringing blurred minutiae. For example, in the channelized permeability experiment, we noted that high $\beta$ is conducive to capturing the continuous channel structures while losing much fidelity in local variations.
\end{remark}

To optimize the parameters $ \btheta$ and $\bphi$ in neural networks, one could employ stochastic gradient descent (SGD) or other gradient-based optimization algorithms related to back propagation.  Note that the second term in Eq.~\eqref{eq:vae_loss} is an expectation approximation using the Monte Carlo method that  samples $\bz^{(i,j)}$ from $q_{\bphi}(\bz^{(i,j)}| \bxi)$. But the expectation computation involving sampling with a high variance will reflect on the gradient estimation, which leads to an unfavorable influence on optimization. To make it trainable and back propagate the gradient correctly, the reparameterization trick is introduced. The latent variables $\bz$ are expressed by a differentiable transformation $g_{\bphi}(\epsilon, \bx)$ with respect to an auxiliary independent random variable $\epsilon$. In the Gaussian distribution case, we let $ \epsilon \sim \mathcal{N}(0,I)$, sampling $\bz$ via such a $g_{\bphi}(\epsilon, \bx)$:

\begin{equation}
\bz = \hat{\mu}(\bx) +\hat{\sigma}(\bx) \odot \epsilon,
\label{eq:repara}
\end{equation}
where $\odot$ refers to an element-wise product. The illustration of forward and back propagation computation with the reparameterization trick is shown in Fig.~\ref{rt}.
\begin{figure}[!htp]
	\centering
	\includegraphics[width=0.9\linewidth, height = 0.25\linewidth]{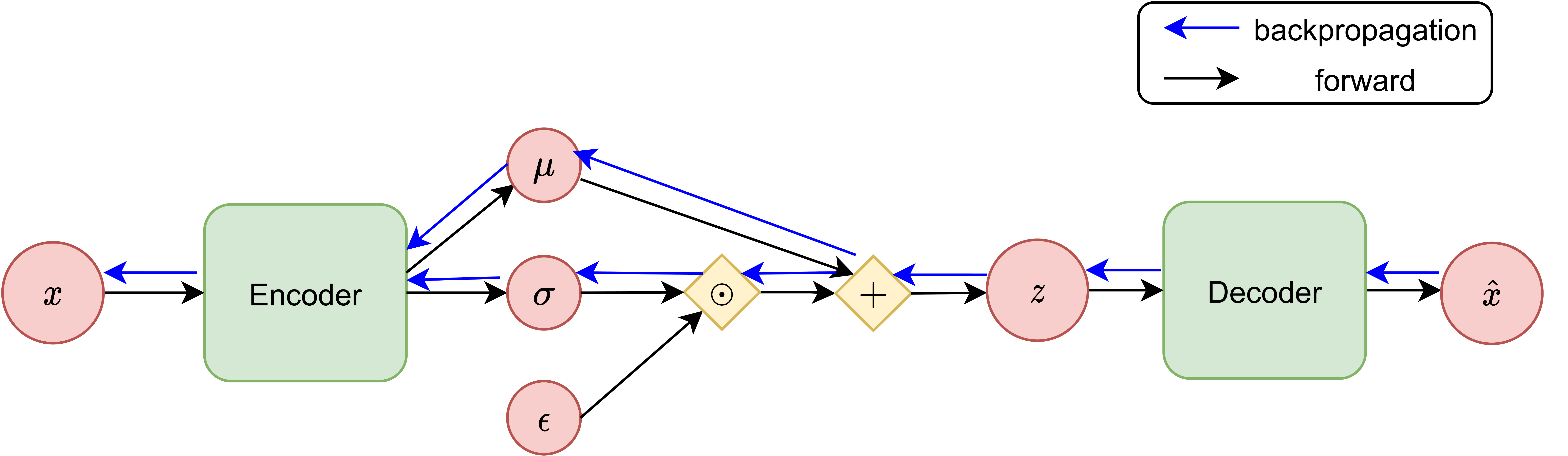}
	\caption{Illustration of  the network architecture. The black arrows denote forward computation in the encoder/decoder networks and the blue arrows indicate the feasible implementation of back propagation under the reparameterization trick.}
	\label{rt}
\end{figure}

\begin{remark}
The reparameterization trick makes the random variable $\bz$ to only depend on two deterministic variables by introducing an auxiliary random variable $\epsilon$ sampled from the standard Gaussian distribution. It scales  $\epsilon$ by the variance $\hat{\sigma}(\bx)$ and shifts it by the mean $\hat{\mu}(\bx)$. The operators $+$ and $\odot$ in Eq.~\eqref{eq:repara} are differentiable, which makes the gradient computation achievable.  Numerical experiments indicate that $m$ in Eq.~\eqref{eq:vae_loss} can be set to $1$ when  $n$ is large enough.
\end{remark}

As discussed in   Definition~\ref{def:GenerativeModel}, we constructed the probabilistic models $p_{ \btheta}(\bx | \bz)$ and $q_{\bphi}(\bz|\bx)$ to sample $\bx$ given its corresponding latent variable $\bz$ and to map the parameters $\bx$ to the latent space, respectively. Based on the objective function in Eq.~\eqref{eq:vae_loss} and neural network in Fig.~\ref{rt} (detailed architecture see~\ref{app:nn}), one can optimize the network parameters $\bphi$ and $ \btheta$. The implementation procedure is summarized in Algorithm~\ref{DGM}. 

\begin{algorithm}[h]
	\caption{Training probabilistic generative model}
	\label{DGM}
	\begin{algorithmic}[1]
		\Require Dataset $\{\bxi\}_{i=1}^{N}$, training epoch $E$, batch size $n$, learning rate $\eta$, hyperparameter $\tilde{\beta}$, $m=1$.
		\State Initialize $ \bphi,  \btheta \leftarrow \text { Initialize parameters }$ 
		\While {$epoch < E$}
		\State $\bx^n \leftarrow$ Sample minibatch $n$ datapoints from $\{\bxi\}_{i=1}^{N}$
		\State  $\epsilon^n \leftarrow$ Sample $n$ noise from Gaussian distribution $\mathcal{N}(0, I )$
		\State  $\bz^n \leftarrow$ Compute by encoder network with Eq.~\eqref{eq:repara}
		\State $\nabla_{ \btheta} \tilde{\mathcal{L}}, \nabla_{\bphi} \tilde{\mathcal{L}} \leftarrow$ \text{Calculate gradients of }$\tilde{\mathcal{L}}\left( \btheta, \bphi ;\bz^n, \bx^{n}, \tilde{\beta}, m \right)$ in Eq.~\eqref{eq:vae_loss}
		\State $ \btheta =  \btheta-\eta \nabla_{ \btheta} \tilde{\mathcal{L}} \leftarrow$ update $ \btheta$ using gradient-based optimization algorithm (e.g. SGD or Adam)
		\State $\bphi = \bphi-\eta \nabla_{\bphi} \tilde{\mathcal{L}} \leftarrow$ update $\bphi$ using gradient-based optimization algorithm (e.g. SGD or Adam)
		\EndWhile
		\Ensure probabilistic encoder  $q_{\bphi}(\bz|\bx)$, probabilistic decoder  $p_{ \btheta}(\bx|\bz)$. 
		
	\end{algorithmic}
\end{algorithm}

This model only involves a single-scale parameter representation learning, while MDGM is a multi-stage and recursive training procedure from coarse to fine, i.e. the training output in the coarse-scale model  is the input to the next finer-scale model.  The coarsest-scale 
probabilistic models $p_{ \btheta_1}(\bx_1 | \bz_1)$ and  $q_{\bphi_1}(\bz_1|\bx_1)$ are outputs of Algorithm~\ref{DGM}. The generative model $p_{ \btheta_1}(\bx_1 | \bz_1)$ is used for the estimation of the posterior  $\pi(\bx_1|\mathcal{D}_{obs})$ in the coarsest-scale by standard MCMC (see Section~\ref{sec:MH}). The recognition model $q_{\bphi_1}(\bz_1|\bx_1)$ is the input (as a pre-trained model) to the second-scale ($l=2$) generative model training   (see Section~\ref{sec:MDGM}).

\subsection{Multiscale deep  generative model (MDGM)}\label{sec:MDGM}
In Section~\ref{sec:vdgm}, we presented  the single-scale parameter generative model in a probabilistic perspective and discussed how to use deep neural networks for its implementation. 
In this section, we enhance this model to a  multiscale framework. Based on the parameter data generation procedure in Section~\ref{sec:Hierarchicalparameterization}, we start from the training data in the finest-scale and coarse grain them recursively to represent the training data in the coarsest-scale. 
We demonstrate the MDGM by explaining how to train it in the $l$-th scale as an example based on the assumption that we obtained the pre-trained encoder and decoder networks in the $(l-1)$-th scale.  

To construct the connection among various scales, we design a special latent space for the finer-scale. In particular, its latent variable $\bz_l$ inherits the latent variable $\bz_{l-1}$ of the previous scale and  augments it with an additional latent variable $\bz_{l}^{\star}$. These two variables are independent. This is in principle  similar to the Bayesian approach followed in~\cite{wan2011bayesian} where a hierarchical sparse grid approximation was used to represent an unknown parameter field at different scales.

As we discussed in Definition~\ref{def:MultiscaleGenerative}, we also need to train the recognition model $q_{\bphi}(\bz_{l}|\bx_{l})$ and generative model $p_{ \btheta}(\bx_l|\bz_l)$ in scale $l$ with the same objective function as in Eq.~\eqref{eq:maginal log_likelihood}. The difference with standard generative models is that the recognition model is divided into two parts in order to encode the parameter $\bx$ with the  independent latent variables $\bz_{l-1}$ and $\bz_{l}^{\star}$. As shown in Fig.~\ref{multi_vae1},    the encoder network includes a pre-trained encoder network and an augmented encoder network. $q_{\bphi}(\bz_{l-1}|\bx_l)$ is computed by adopting the upscaling operator over $\bx_l$ firstly and then using the previous scale encoder network. It can be modeled by 

\begin{figure}[!htp]
	\centering
	\includegraphics[width=0.9\linewidth, height = 0.24\linewidth]{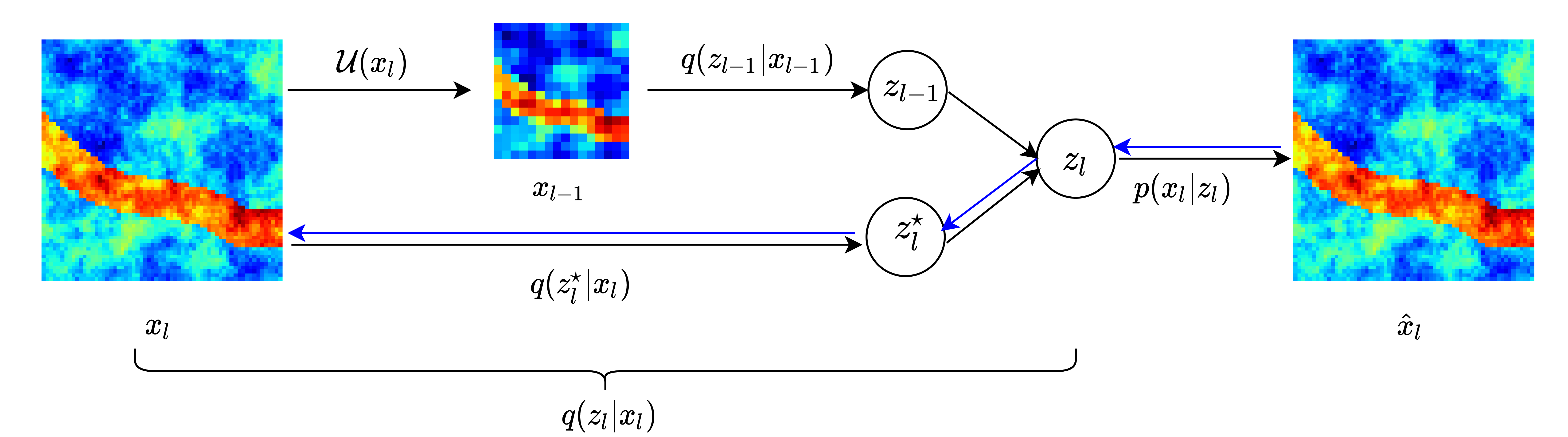}
	\caption{Schematic illustration of the probabilistic generative model in the $l$-th scale. Black arrows above denote the forward computation, and blue arrows in reverse direction denote the back propagation in gradient-based optimization. $q(\bz_{l-1}|\bx_{l-1})$ as an input is a pre-trained network, whereas the model parameters in $q(\bz_l^{\star}|\bx_{l})$ and $p(\bx_{l}|\bz_{l})$ need to be learned. For details on how to concatenate $\bz_{l-1}$ with $\bz_l^{\star}$ to obtain the latent variable $\bz_{l}$ refer to~\ref{app:cat}. }
	\label{multi_vae1}
\end{figure}

\begin{equation}
q_{\bphi}(\bz_{l-1}|\bx_{l})
= \int q_{\bphi}(\bz_{l-1}|\bx_{l-1})  \pi(\bx_{l-1}|\bx_{l}) d\bx_{l-1}.
\end{equation}
Since the upscaling operator $\mathcal {U}$ is deterministic, we can write the following:

\begin{equation}
\begin{aligned}q_{\bphi}(\bz_{l-1}|\bx_{l})
= q_{\bphi}(\bz_{l-1}|\bx_{l-1}).
\label{eq: convert}
\end{aligned}
\end{equation}
As in the single-scale probabilistic generative model in the last section, the variational lower bound $\mathcal{L}( \btheta,\bphi; \bx)$ consists of two parts as described in Eq.~\eqref{eqn:decomposedlowerbound}, i.e. $\mathbb E_{q\inphi(\bzi|\bxi)} [\log p\inth(\bxi|\bzi) ]$ and $D_{KL}\left(q\inphi(\bzi|\bxi) || p\inth(\bzi)\right)$. The first part in the $l$-th scale is formulated as the reconstruction error  introduced in Eq.~\eqref{eq:recon_loss}, which involves the decoder model $p_{ \btheta_l}(\bx_l | \bz_l)$ and its model parameter $\btheta_l$. The second part is much distinct with $p(\bz_l)$ being the target distribution in the KL-divergence.  It is an isotropic multivariate Gaussian distribution as well, but it can be viewed as the joint distribution of $(\bz_{l-1},\bz_l^{\star})$, where $\bz_{l-1}$, $\bz_l^{\star}$ are independent. Thus $p(\bz_{l-1})$ and $p(\bz_{l}^{\star})$ are also isotropic multivariate Gaussian distributions. Combining  with Eq.~\eqref{eq: convert}, we rewrite the KL-divergence as follows:

\begin{eqnarray}
D_{KL}(q_{\bphi}(\bz_l | \bx_l) \| p(\bz_l)) & = & \int q_{\bphi}(\bz_l | \bx_l) \log \frac{q_{\bphi}(\bz_l | \bx_l)}{ p(\bz_l)} d\bz_l \nonumber \\
&=& \int \int q_{\bphi}(\bz_{l-1} | \bx_l) q_{\bphi}(\bz_l^{\star} | \bx_l)  \log \frac{q_{\bphi}(\bz_{l-1} | \bx_l) q_{\bphi}(\bz_l^{\star} | \bx_l)}{ p(\bz_{l-1})p(\bz_l^{\star})} d\bz_{l-1} d\bz_l^{\star} \nonumber \\
&=& \int \int q_{\bphi}(\bz_{l-1} | \bx_{l-1}) q_{\bphi}(\bz_l^{\star} | \bx_l)  \left[\log \frac{q_{\bphi}(\bz_{l-1} | \bx_{l-1}) }{ p(\bz_{l-1})}+\log \frac{q_{\bphi}(\bz_l^{\star} | \bx_l)}{p(\bz_l^{\star})} \right] d\bz_{l-1} d\bz_l^{\star} \nonumber \\
&=& \int  q_{\bphi}(\bz_{l-1} | \bx_{l-1})  \log \frac{q_{\bphi}(\bz_{l-1} | \bx_{l-1}) }{ p(\bz_{l-1})}d\bz_{l-1}\int q_{\bphi}(\bz_l^{\star} | \bx_l)  d\bz_l^{\star} \nonumber \\
&+& \int  q_{\bphi}(\bz_l^{\star} | \bx_l)  \log \frac{q_{\bphi}(\bz_l^{\star} | \bx_l) }{ p(\bz_l^{\star})}d\bz_l^{\star}\int q_{\bphi}(\bz_{l-1} | \bx_{l-1}) d\bz_{l-1} \nonumber \\
&=&  D_{KL}(q_{\bphi}(\bz_{l-1}| \bx_{l-1}) \| p(\bz_{l-1}))+D_{KL}(q_{\bphi}(\bz_l^{\star} | \bx_l) \| p(\bz_l^{\star})),
\label{eq:multi_KL_loss}
\end{eqnarray}
where $\bphi$ in the $l$-th scale is $\bphi_l$ containing  $\bphi_{l-1}$ and $\bphi_{l}^{\star}$, where $\bphi_{l-1}$ denotes the parameters of the encoder network $q(\bz_{l-1}| \bx_{l-1})$ in the $(l-1)$-th scale and $ \bphi_{l}^{\star}$ denotes the parameters of the augmented encoder network $q(\bz_l^{\star} | \bx_l)$ in the $l$-th scale.  Recall that we need to employ the pre-trained encoder network $q_{\bphi_{l-1}}(\bz_{l-1} | \bx_{l-1})$ directly in the $l$-th scale training because it should ensure that the finer-scale generative model shares common latent variables $\bz_{l-1}$ with the coarser-scale, so

\begin{equation}
\begin{aligned}
D_{KL}(q_{\bphi}\left(\bz_l | \bx_l) \| p(\bz_l)\right) = D_{KL}\left(q_{\bphi_l^{\star}}(\bz_l^{\star} | \bx_l\right) \| p(\bz_l^{\star}))+\text{constant}.
\label{eq:opt_multi_KL_loss}
\end{aligned}
\end{equation}
We also use gradient-based optimization algorithm to train the MDGM where the back propagation is only applied in the decoder network and augmented encoder network as blue arrows illustrated in Fig.~\ref{multi_vae1}. We can formulate the loss function for each training iteration like Eq.~\eqref{eq:vae_loss} as follows:
\begin{eqnarray}\tilde{\mathcal{L}}( \btheta_l,\bphi^{\star}_l) &=& \frac{\tilde{\beta}}{2}\sum_{i=1}^{n}\sum_{k=1}^{d_l^{\star}}\left(\hat{\mu}_{lk}^{2}(\bx_l^{(i)})+\hat{\sigma}_{lk}^{2}(\bx_l^{(i)})-\log \hat{\sigma}^2_{lk}(\bx_l^{(i)})-1\right) \nonumber \\ &+& \frac{1}{m} \sum_{i=1}^{n}\sum_{j=1}^{m} \left\|\bx_l^{(i)}-\mu_l(\bz_l^{(i,j)})\right\|^{2},
\label{eq:MDGMloss}
\end{eqnarray}
where $\hat{\mu}_{lk}(\cdot)$ and $\hat{\sigma}_{lk}(\cdot)$ denote the $k$-th element in the mean and standard deviation, respectively, which are outputs of the $l$-th scale  augmented encoder network, and $\mu_l(\cdot)$ is the output of the $l$-th scale decoder network. Here, $m$ is still set to be $1$ when  the batch size $n$ is large enough. $d_l^{\star}$ is the dimension of the latent variable $\bz_l^{\star}$. The hyperparameter $\tilde{\beta}$ is determined by balancing the disentanglement of the parameter features and reconstruction.

In summary, assume that we obtained the probabilistic encoder and decoder models in the previous $(l-1)$ level. We can then train the augmented encoder network and decoder networks using the objective function in Eq.~\eqref{eq:MDGMloss} so that the probabilistic encoder  $q_{\bm{\phi_l}}(\bz_l|\bx_l)$ and the probabilistic decoder  $p_{ \btheta_l}(\bx_l|\bz_l)$ in the $l$-th scale are acquired. For details see Algorithm~\ref{multi_vae_algorithm}. The architectures of the augmented encoder and decoder networks  are the same as those in the  encoder and decoder networks for the single-scale model (refer to~\ref{app:nn}). 

\par
We have introduced so far the probabilistic generative model in a single-scale in the last section and the concept of generative modeling across two scales in this section. Since these unsupervised learning models only need the unlabeled data, we only provide the finest-scale dataset $\{\bxi_{L}\}_{i=1}^{N}$ sampled from the underlying prior distribution $\pi(\bx_L)$ and adopt the deterministic upscaling operator $\mathcal{U}$ to generate the dataset in the scale of interest. The training of the  generative models proceeds  from the coarsest- to the finest-scale. The coarsest-scale generative model is trained by Algorithm~\ref{DGM}. From the second- to the finest-scale, the generative model is trained recursively by Algorithm~\ref{multi_vae_algorithm}. Once the generative models are trained in the scales of interest, they can be integrated with MCMC to estimate the posterior of the parameters in each scale.

\begin{remark}
The training of the $l$-th scale generative model employs the $(l-1)$-th scale encoder network, where $\bz_{l-1}$ is sampled from $q_{\bphi_{l-1}}(\bz_{l-1}|\bx_{l-1})$ using the  reparameterization trick. A $\bz_{l-1}$ with high-variance will   impact the training stability and convergence inflicting big noise in the latent variable $\bz_l$. We address this problem in the training procedure by replacing $\bz_{l-1}$ with its mean $\hat{\mu}_{l-1}(\bx_{l-1})$.
\end{remark}

\begin{algorithm}[h]
	\caption{Training of the multiscale probabilistic generative model in the $l$-th scale}
	\label{multi_vae_algorithm}
	\begin{algorithmic}[1]
	\Require $ \{\bx_{i}\}_{i=1}^{N}$, pre-trained encoder $q_{\bphi_{l-1}}(\bz_{l-1}|\bx_{l-1})$, training epoch $E$, batch size $n$, learning rate $\eta$, hyperparameter $\tilde{\beta}$, $m=1$
		\State Initialize $ \bphi_l,  \btheta_{l}^{\star} \leftarrow \text { Initialize parameters }$ 
		\While {$epoch < E$}
		\State $\bx^n \leftarrow$ Sample minibatch $M$ datapoints from $ \{\bx_{i}\}_{i=1}^{N}$
		\State $\bz_{l-1}^n \leftarrow$ Compute  by $q_{\bphi_{l-1}}(\bz_{l-1}|\bx_{l-1})$ for each data
		\State  $\epsilon \leftarrow$ Sample $n$ noise variables from the  Gaussian distribution $p(\epsilon)$
		\State $\bz_{l}^n \leftarrow$ Compute $\bz_{l}^{\star}$ by the  augmented encoder network in Eq.~\eqref{eq:repara}, and concatenate with $\bz_{l-1}$ for each data
		\State $\nabla_{  \btheta_{l}^{\star}} \tilde{\mathcal{L}}, \nabla_{\bphi_l} \tilde{\mathcal{L}} \leftarrow$ \text{Calculate gradients of }$\tilde{\mathcal{L}}\left( \btheta_{l}^{\star}, \bphi_l ;\bz_{l}^n, \bx^{n}, \tilde{\beta}, m \right)$ in Eq.~\eqref{eq:MDGMloss} with respect to $ \btheta_{l}^{\star}$  and $\bphi_l$
		\State $  \btheta_{l}^{\star}= \btheta-\eta \nabla_{  \btheta_{l}^{\star}} \tilde{\mathcal{L}} \leftarrow$ using gradient-based optimization algorithm (e.g. SGD or Adam)
		\State $\bphi_l=\bphi-\eta \nabla_{\bphi_l} \tilde{\mathcal{L}} \leftarrow$ using gradient-based optimization algorithm (e.g. SGD or Adam)
		\EndWhile
		\State $q_{\bphi_l}(\bz_l|\bx_l) \leftarrow$ combine the pre-trained encoder $q_{\bphi_{l-1}}(\bz_{l-1}|\bx_{l-1})$ with $q_{\bphi_l^{\star}}(\bz_l^{\star} | \bx_l)$
		\Ensure probabilistic encoder  $q_{\bphi_l}(\bz_l|\bx_l)$, probabilistic decoder  $p_{\btheta_l}(\bx_l|\bz_l)$.  
	\end{algorithmic}
\end{algorithm}

\subsection{Bayesian inversion with the DGM}\label{sec:MH}
In Bayesian inversion, the Metropolis-Hastings (MH) algorithm~\cite{metropolis1953equation, hastings1970monte} is a popular MCMC method to approximate the  posterior distribution~\cite{andrieu2003introduction}.
As discussed in Definition~\ref{def:MultiscalePosterior}, we design a multiscale scheme aimed to realize the high-dimensional parameter estimation in inverse problems where  the posterior distribution is approximated recursively. We utilize the previous coarser-scale estimation information as we proceed with inference in the next immediate finer-scale. Note that the coarsest-scale inference only involves single-scale information in which the parameterization employs the DGM given in Algorithm~\ref{DGM}. 

For the coarsest-scale posterior distribution, we are interested in combining the standard MH algorithm with the single-scale generative model to estimate the parameters $\bx_1\in \mathbb{R}^{M_1}$ using the low-dimensional latent variable $\bz_1 \in \mathbb{R}^{d_1}$, where $M_1$ is the dimension of the coarsest-scale parameter $\bx_1$ that also determines the spatial discretization of the coarsest-scale forward model $\mathcal{F}_1$, and $d_1$ is the dimension of the coarsest-scale latent space. MCMC based on the pre-trained model $p_{\btheta_1}(\bx_1 | \bz_1)$ allows us to explore the posterior approximation of $\bz_1$. Recall that the model $p_{\btheta_1}(\bx_1 | \bz_1)$ is a Gaussian distribution with mean $\mu_{\btheta_1}(\bz_1)$ and standard derivation $\sigma$, where $\mu_{\btheta_1}(\bz_1)$ is the output of the decoder network in Algorithm~\ref{DGM} and $\sigma$ is a hyperparameter that  determines the noise  level. In the test procedure, we can thus define $\bx_1 = \mu_{\btheta_1}(\bz_1)$ as a mapping from the latent space to the original parameter space. Then one can sample in the continuous latent space to produce various parameters $\bx_1$ using this mapping.
In this way, the problem of inferring  $\bx_1$ becomes equivalent to the one of  inferring $\bz_1$.

Given the observed data $\mathcal{D}_{obs}$, the forward function $\mathcal{F}_1$, the decoder model $\bx_1 = \mu_{\btheta_1}(\bz_1)$, and the prior distribution $p(\bz_1)$, the unnormalized posterior distribution in the coarsest-scale is as follows:

\begin{equation}
   \pi(\bz_1 | \mathcal{D}_{obs}) \propto \mathcal{L}_e(\mathcal{D}_{obs} | \bz_1)p(\bz_1), 
   \label{eq:postrior of z1}
\end{equation}
where the prior distribution $p(\bz_1)$ is $\mathcal{N}(\bm{0},\bm{I})$. 
 
For the above target distribution in the coarsest-scale, we only need to create a Markov chain with length $N_{ite}$ for the latent variable $\bz_1$.  
The proposal distribution $\pi_q(\bz_1^{\prime}|\bz_1^{(j)})$ can be the simple random walk~\cite{gelman1996efficient} or a preconditioned Crank-Nicholson (pCN) algorithm~\cite{dodwell2015hierarchical,cotter2013mcmc, hairer2014spectral}. To initialize the first state $\bz_1^{(1)}$ for this Markov chain, one can sample  $\bz_1^{(1)}$  from the prior distribution $p(\bz_1)$. For each $j \geq 1$, we first sample a candidate $\bz_1^{\prime}$ from the proposal distribution $\pi_q(\cdot|\bz_1^{(j)})$ and then reconstruct the spatially-varying parameter $\bx_1$ by the  decoder model $\bx_1=\mu_{\phi_1}( \bz_1$). One can compute the likelihood function and the Metropolis acceptance ratio $\alpha$ using the forward model $\mathcal{F}_1(\bx_1)$. The acceptance ratio $\alpha$ is defined as in MH algorithm:

\begin{equation}
\alpha = min\left(1,\frac{\pi(\bz_1^{\prime}|\mathcal{D}_{obs}) \pi_q(\bz_1^{(j)}|\bz_1^{\prime})}{\pi(\bz_1^{(j)}|\mathcal{D}_{obs})\pi_q(\bz_1^{\prime}|\bz_1^{(j)})}\right).
\label{eq:ratio_MH_DGM}
\end{equation}
With probability $\alpha$, the candidate sample $\bz_1^{\prime}$ is accepted, otherwise rejected. Specifically, in each iteration, the state includes not only the latent variable $\bz_1$ but also its corresponding spatially-varying parameter $\bx_1$. The Markov chain typically takes some iterations to reach its stationary distribution, which depends on the target distribution. The first $n_{b}$ states of the Markov chain (burn-in stage) are thrown-away to ensure a valid  approximation of the target distribution. Once the latent variable $\bz_1$ is estimated, its corresponding parameter samples $\{\bx_1^{(j)}\}_{j=n_{b}}^{N_{ite}}$ can be viewed as samples from the posterior distribution $\pi(\bx_1 | \mathcal{D}_{obs})$. A summary of the MH algorithm for Bayesian inverse modeling based on the deep generative model is given in Algorithm~\ref{mcmc}. This algorithm can be applied to solve any high-dimensional parameter estimation in the single-scale. 

In the multiscale context, the coarsest posterior  was approximated by Algorithm~\ref{mcmc}. But there are still two main problems to be resolved by invoking a finer-scale model. One is that the high-resolution parameter estimation that can present richer details is desired; the other is the inaccurate forward model introduces  an epistemic uncertainty or model error, which will reflect on the parameter estimation. 
The model error $\delta_l$ in the $l$-th scale is defined by

\begin{equation}
   \delta_l = \mathcal{F}_L(\bx_L) - \mathcal{F}_l(\bx_l).
   \label{eq:model_error}
\end{equation}
This model error $\delta_l$  will be zero if and only if we adopt the ``true model'' $\mathcal{F}_L$. The posterior distribution $\pi(\bx_L | \mathcal{D}_{obs})$  is the ultimate goal of our inverse problem.

\begin{algorithm}[h]
	\caption{ The Metropolis-Hastings with Deep Generative Model (MH-DGM) Algorithm for the first-scale posterior distribution approximation}
	\label{mcmc}
	\begin{algorithmic}[1]
		\Require  unnormalized distribution $\pi(\bz| \mathcal{D}_{obs})$, proposal distribution $\pi_q(\cdot)$, iteration number $N_{ite}$, decoder model $\mu( \bz)$, burn-in length $n_{b}$
		\State Initialize $\bz^{(1)},  \bz^{(1)} \sim \mathcal{N}(\bm{0},\bm{I})$
		\For {$j = 1:N$}
		\State Draw $\bz^{\prime} \sim \pi_q(\cdot|\bz^{(j)}) $ 
		\State Compute the spatially-varying parameter $ \bx^{\prime} = \mu(\bz^{\prime})$ by the decoder model
		\State Compute the likelihood function by solving the forward model $\mathcal{F}(\bx^{\prime})$
		\State Compute the acceptance ratio
		$$\alpha = min\left(1,\frac{\pi(\bz^{\prime}|\mathcal{D}_{obs}) \pi_q(\bz^{(j)}|\bz^{\prime})}{\pi(\bz^{(j)}|\mathcal{D}_{obs})\pi_q(\bz^{\prime}|\bz^{(j)})}\right)$$
		\State Draw $\rho$ from uniform distribution $\mathcal{U}[0,1]$
 		\If {$\rho < \alpha$}
		\State $\text { Let } \bz^{(j+1)}=\bz^{\prime}, \bx^{(j+1)}=\bx^{\prime}$
		\Else
		\State {$\text { Let }\bz^{(j+1)}=\bz^{j}, \bx^{(j+1)}=\bx^{(j)}$}
		\EndIf
		\EndFor
		\Ensure posterior samples $\{\bx^{(i)}\}_{i=n_b}^{N_{ite}}$
	\end{algorithmic}
\end{algorithm}

\subsection{Multiscale Bayesian inversion with the MDGM}\label{sec:MultiscaleMH}

To eliminate the model error and estimate the fine-scale parameter, we integrate the  MDGM  with the MH algorithm.
We refer to this algorithm as MH-MDGM. Such a procedure involves inference across scales and allows us to generate samples on each scale for estimating the unknown parameter proceeding from the coarsest- to the finest-scale. Leveraging the latent space construction in the MDGM, the difference with the single-scale method is that the fine-scale estimation integrated with the MDGM method can inherit the coarse estimation resulting in a highly-efficient sampling procedure. This approach is applied from the second-coarsest to the finest-scale where inference in each scale proceeds recursively utilizing information at the previous coarser-scale.
\par
For any $2\leq l \leq L$, let us suppose that the approximate posterior distribution $\pi(\bz_{l-1}|\mathcal{D}_{obs})$ in the  $(l-1)$-th scale has been obtained. We thus have obtained the posterior samples $\{\bz_{l-1}^{(i)}\}_{i=n_{l-1}}^{N_{l-1}}$ and their corresponding parameter samples $\{\bx_{l-1}^{(i)}\}_{i=n_{l-1}}^{N_{l-1}}$. As defined in Eqs.~\eqref{eq:PosteriorOfZ} and~\eqref{eq:postrior of z1}, we are concerned about the finer-scale $\pi(\bz_{l}|\mathcal{D}_{obs})$ posterior estimation. Using the MCMC method, one can sample $\bz_l$ to generate the parameter $\bx_l$ by the decoder model $\bx_{l} = \mu_{\btheta_{l}}( \bz_{l})$ in MDGM, and then 
solve the forward model $\mathcal{F}_l(\bx_l)$ to evaluate the likelihood function and the acceptance ratio.

Suppose that the $j$-th state of the Markov chain in the $l$-th scale is $\bz_l^{(j)}$. The acceptance ratio of the reject/accept scheme  for the $(j+1)$-th state in the MH-MDGM is still the same as that in the single-scale MH algorithm, i.e.

\begin{equation}
    \alpha = min\left(1,\frac{\pi(\bz_l^{\prime}|\mathcal{D}_{obs}) \pi_q(\bz_l^{(j)}|\bz_l^{\prime})}{\pi(\bz_l^{(j)}|\mathcal{D}_{obs}) \pi_q(\bz_l^{\prime}|\bz_l^{(j)})}\right), 
    \label{ratio}
\end{equation}
where the proposal $\pi_q(\bz_l^{\prime}|\bz_l^{(j)})$ is a particular distribution since $\bz_l = (\bz_{l-1},\bz_l^{\star})$ contains two independent random variables, which have different connotation in the MH-MDGM. The proposal is taken to have a factorized form  as follows~\cite{ dodwell2015hierarchical, dodwell2019multilevel}:

\begin{equation}
    \pi_q(\bz_l^{\prime}|\bz_l^{(j)})=\pi_q(\bz_{l-1}^{\prime}|\bz_{l-1}^{(j)},\bz_{l}^{\star(j)}) \pi_q(\bz_l^{\star\prime}|\bz_{l-1}^{(j)},\bz_{l}^{\star(j)}).
    \label{fac}
\end{equation}
In the MDGM, 
$\bz_{l}^{\star}$ is sampled from an isotropic  Gaussian  distribution, which is independent of $\bz_{l-1}$. The proposal distribution for $\bz_{l}^{\star(j)}$ can be simplified as 

\begin{equation}
\pi_q(\bz_l^{\star\prime}|\bz_{l-1}^{(j)},\bz_{l}^{\star(j)}) = \pi_q(\bz_l^{\star\prime}|\bz_{l}^{\star(j)}).
\label{eq: lth poposal}
\end{equation}
This can be a simple random walk or the pCN algorithm. 
Exploring $\bz_l^{\star}$ provides additional information to that obtained at the $(l-1)$-scale allowing us to enrich the obtained local details and  correct global features. One can set a big step size for $\bz_l^{\star}$ in the proposal distribution to explore and detect various local patterns. 
In the MH-MDGM, the initial state $\bz_l^{\star}$ can be sampled  from the prior distribution $p(\bz_l^{\star})$.

Based on the independence assumption, one can define a proposal distribution in the $l$-th scale for the latent variable $\bz_{l-1}$ that connects adjacent scales as follows:

\begin{equation}
\pi_q(\bz_{l-1}^{\prime}|\bz_{l-1}^{(j)},\bz_{l}^{\star(j)}) = \pi_q(\bz_{l-1}^{\prime}|\bz_{l-1}^{(j)}).
\label{eq: l-1 th proposal}
\end{equation}
Once $\bz_{l-1}$ has been estimated in the previous-scale, some multiscale estimation methods like~\cite{wan2011bayesian} will not correct and estimate it again in the finer-scale. However, note that the model error that resulted from the coarse-scale solver can impact the posterior estimation. We treat the latent variable $\bz_{l-1}$ as a random variable in the $l$-th scale estimation and update it invoking the accurate forward model. To exploit the $(l-1)$-th scale posterior result, we  assign an initial state for $\bz_{l-1}$ in the $l$-th scale estimation, which can inherit the $(l-1)$-th scale Markov state, greatly reducing the time to obtain the stationary distribution. The obvious approach is to use the last state in the $(l-1)$-th scale, while another reasonable choice is the mean of the $(l-1)$-th posterior estimation. Note that we use a set of posterior samples $\{\bx_{l-1}^{(i)}\}_{i=n_{l-1}}^{N_{l-1}}$ to approximate the posterior distribution $\pi(\bx_{l-1}|\mathcal{D}_{obs})$, where $n_{l-1}$, $N_{l-1}$ denote  the lengths of  burn-in and Markov chain, respectively. $\bz_{l-1}$ encodes the coarse-scale estimation information via the encoder network $\pi_{\btheta_{l-1}}(\bz_{l-1}|\bx_{l-1})$. We can assign the initial state $\bz_{l-1}^{(1)}$ in the $l$-th scale estimation to be $\bz_{l-1}^{(1)} = \mathop{\arg\max} \pi_{\btheta_{l-1}}(\bz_{l-1}|\bm{\overline{\bx}}_{l-1})$, where $\bm{\overline{\bx}}_{l-1}$ is the mean of the posterior samples $\{\bx_{l-1}^{(i)}\}_{i=n_{l-1}}^{N_{l-1}}$. Since the latent variable $\bz_{l-1}$ has been estimated and can impact the crucial global features of the parameter $\bx_{l}$ generation using the decoder model $\bx_{l} = \mu_{\btheta_{l}}( \bz_{l})$, it should be assigned a small step size in the proposal distribution.

Combining Eq.~\eqref{eq: l-1 th proposal} with Eq.~\eqref{eq: lth poposal}, we obtain the conclusion in  Proposition~\ref{detailed balance}. The MH-MDGM algorithm for the $l$-th scale posterior distribution approximation is presented in Algorithm~\ref{multi_mcmc}.

\begin{remark}
In the $(l-1)$-th scale posterior approximation, the Markov chain generates samples $\{\bz_{l-1}^{(i)}\}_{i=n_{l-1}}^{N_{l-1}}$ and their corresponding parameter samples $\{\bx_{l-1}^{(i)}\}_{i=n_{l-1}}^{N_{l-1}}$. One cannot directly use the posterior samples $\bz_{l-1}^{(i)}$  in the $l$-th fine-scale inference. The reason is that these samples are only related to the decoder model $\mu_{\btheta_{l-1}}( \bz_{l-1})$, which can generate parameters $\bx_{l-1}$ to evaluate the likelihood function for the coarse-scale posterior approximation. Note that the training of the MDGM in the $l$-th scale only uses $\bx_{l}$ and $\bx_{l-1}$ rather than the latent variables. Message passing across scales depends on the encoder network $\pi_{\btheta_{l-1}}(\bz_{l-1}|\bx_{l-1})$. In order to generate $\bx_{l}$  for the fine-scale inference, we need to  use $\{\bx_{l-1}^{(i)}\}_{i=n_{l-1}}^{N_{l-1}}$ as the  approximate $(l-1)$-th scale posterior result and then encode them into the latent variables $\bz_{l-1}$,  rather than directly using the samples $\{\bz_{l-1}^{(i)}\}_{i=n_{l-1}}^{N_{l-1}}$.
\end{remark}
\par

\begin{myprop}
\label{detailed balance}
 For any $2\leq  l \leq L$, the acceptance ratio $
 \alpha(\bz_l^{\prime},\bz_l^{(j)})$ for the $(j+1)$-th state in the $l$-th scale MH-MDGM algorithm with the factorized proposal is  

 \begin{equation}
 \label{ac}
      \alpha(\bz_l^{\prime},\bz_l^{(j)}) = min\left(1,\frac{\pi(\bz_l^{\prime}|\mathcal{D}_{obs})\pi_q(\bz_{l-1}^{(j)}|\bz_{l-1}^{\prime})\pi_q(\bz_{l}^{\star(j)}|\bz_l^{\star\prime})}{\pi(\bz_l^{(j)}|\mathcal{D}_{obs}) \pi_q(\bz_{l-1}^{\prime}|\bz_{l-1}^{(j)})\pi_q(\bz_l^{\star\prime}|\bz_{l}^{\star(j)})}\right),
 \end{equation}
 which satisfies the detailed balance condition.
\end{myprop}
\begin{proof}
The transition kernel $K(\bz_l^{\prime},\bz_l^{(j)})$ with the proposal distribution $\pi_q(\bz_l^{\prime}|\bz_l^{(j)})$ in the MH algorithm is:

\begin{equation}
    K(\bz_l^{\prime},\bz_l^{(j)}) = \pi_q(\bz_l^{\prime}|\bz_l^{(j)})\alpha(\bz_l^{\prime},\bz_l^{(j)}) + \delta(\bz_l^{\prime},\bz_l^{(j)})\int\pi_q(\bz_l|\bz_l^{(j)})(1-\alpha(\bz_l,\bz_l^{(j)})) d\bz_l, 
\end{equation}
where $\delta(\bz_l^{\prime},\bz_l^{(j)})$ is the Dirac delta function. When $\bz_l^{\prime}=\bz_l^{(j)}$, it is  obvious that  the below detailed balance condition is satisfied:

\begin{equation}
\pi(\bz_l^{(j)}|\mathcal{D}_{obs})K(\bz_l^{\prime},\bz_l^{(j)}) = \pi(\bz_l^{\prime}|\mathcal{D}_{obs})K(\bz_l^{(j)},\bz_l^{\prime}).
\end{equation}
When $\bz_l^{\prime} \neq \bz_l^{(j)}$, $\delta(\bz_l^{\prime},\bz_l^{(j)}) = 0$, we obtain $K(\bz_l^{\prime},\bz_l^{(j)}) = \pi_q(\bz_l^{\prime}|\bz_l^{(j)})\alpha(\bz_l^{\prime},\bz_l^{(j)})$. Based on Eqs.~\eqref{eq: lth poposal} and~\eqref{eq: l-1 th proposal}, the proposal distribution can be written as $\pi_q(\bz_l^{\prime}|\bz_l^{(j)}) = \pi_q(\bz_{l-1}^{\prime}|\bz_{l-1}^{(j)})\pi_q(\bz_l^{\star\prime}|\bz_{l}^{\star(j)})$. Then we can write the following:

\begin{eqnarray}
&&\pi(\bz_l^{(j)}|\mathcal{D}_{obs})K(\bz_l^{\prime},\bz_l^{(j)}) \nonumber \\
&=& \pi(\bz_l^{(j)}|\mathcal{D}_{obs})\pi_q(\bz_{l-1}^{\prime}|\bz_{l-1}^{(j)})\pi_q(\bz_l^{\star\prime}|\bz_{l}^{\star(j)}) \min\left(1,\frac{\pi(\bz_l^{\prime}|\mathcal{D}_{obs})\pi_q(\bz_{l-1}^{(j)}|\bz_{l-1}^{\prime})\pi_q(\bz_{l}^{\star(j)}|\bz_l^{\star\prime})}{\pi(\bz_l^{(j)}|\mathcal{D}_{obs}) \pi_q(\bz_{l-1}^{\prime}|\bz_{l-1}^{(j)})\pi_q(\bz_l^{\star\prime}|\bz_{l}^{\star(j)})}\right)\nonumber \\
&=& \min\left( \pi(\bz_l^{(j)}|\mathcal{D}_{obs})\pi_q(\bz_{l-1}^{\prime}|\bz_{l-1}^{(j)})\pi_q(\bz_l^{\star\prime}|\bz_{l}^{\star(j)}),\pi(\bz_l^{\prime}|\mathcal{D}_{obs})\pi_q(\bz_{l-1}^{(j)}|\bz_{l-1}^{\prime})\pi_q(\bz_{l}^{\star(j)}|\bz_l^{\star\prime}) \right)\nonumber \\
&=& \pi(\bz_l^{\prime}|\mathcal{D}_{obs})\pi_q(\bz_{l-1}^{(j)}|\bz_{l-1}^{\prime})\pi_q(\bz_{l}^{\star(j)}|\bz_l^{\star\prime})\min\left(1,\frac{\pi(\bz_l^{(j)}|\mathcal{D}_{obs})\pi_q(\bz_{l-1}^{\prime}|\bz_{l-1}^{(j)})\pi_q(\bz_l^{\star\prime}|\bz_{l}^{\star(j)})}{\pi(\bz_l^{\prime}|\mathcal{D}_{obs})\pi_q(\bz_{l-1}^{(j)}|\bz_{l-1}^{\prime})\pi_q(\bz_{l}^{\star(j)}|\bz_l^{\star\prime})}\right) \nonumber \\
&=& \pi(\bz_l^{\prime}|\mathcal{D}_{obs})K(\bz_l^{(j)},\bz_l^{\prime}),
\end{eqnarray} 
and the detailed balance condition is satisfied.
\end{proof}

\begin{algorithm}
	\caption{ The Metropolis-Hastings with Multiscale Deep Generative Model (MH-MDGM) Algorithm for the $l$-th scale posterior distribution approximation}
	\label{multi_mcmc}
	\begin{algorithmic}[1]
		\Require unnormalized distribution $\pi(\bz_l| \mathcal{D}_{obs})$, proposal distribution $\pi_q(\cdot|\bz_{l-1})$ and $\pi_q(\cdot|\bz_{l}^{\star})$,  decoder model $\mu_{\btheta_l}( \bz_l)$,  recognition model $\pi_{\btheta_{l-1}}(\bz_{l-1}|\bx_{l-1})$, $(l-1)$-th scale posterior samples $\{\bx_{l-1}^{(j)}\}_{j=n_{l-1}}^{N_{l-1}}$, iteration number $N_l$, burn-in length $n_l$.
		\State Compute the mean of $(l-1)$-th scale posterior estimation $\bm{\overline{\bx}}_{l-1}$
		\State Initialize $\bz_{l-1}^{(1)},  \bz_{l-1}^{(1)} = \mathop{\arg\max} \pi_{\btheta_{l-1}}(\bz_{l-1}|\bm{\overline{\bx}}_{l-1})$.
		\State Initialize $\bz_l^{\star(1)}, \bz_l^{\star(1)} \sim \mathcal{N}(\bm{0},\bm{I})$.
		\For {$j = 1:N$}
		\State Draw $\bz_{l-1}^{\prime}  \sim \pi_q(\cdot|\bz_{l-1}^{(j)})$.
		\State Draw $\bz_l^{\star\prime} \sim \pi_q(\cdot|\bz_l^{\star(j)})$, then $\bz_l^{\prime} = (\bz_{l-1}^{\prime},\bz_l^{\star\prime})$
			\State Compute the spatially-varying parameter $ \bx^{\prime}_l = \mu_{\btheta_l}(\bz^{\prime}_l)$ by the decoder model
		\State Compute the likelihood function by solving the forward model $\mathcal{F}_l(\bx^{\prime}_l)$
		\State Compute the acceptance ratio 
		$$\alpha = min\left(1,\frac{\pi(\bz_l^{\prime}|\mathcal{D}_{obs})\pi_q(\bz_{l-1}^{(j)}|\bz_{l-1}^{\prime})\pi_q(\bz_{l}^{\star(j)}|\bz_l^{\star\prime})}{\pi(\bz_l^{(j)}|\mathcal{D}_{obs}) \pi_q(\bz_{l-1}^{\prime}|\bz_{l-1}^{(j)})\pi_q(\bz_l^{\star\prime}|\bz_{l}^{\star(j)})}\right)$$
		\State Draw $\rho$ from uniform distribution $\mathcal{U}[0,1]$
 		\If {$\rho < \alpha$}
		\State $\text { Let } \bz_{l-1}^{(j+1)}=\bz_{l-1}^{\prime} \text { and } \bz_l^{\star(j+1)}=\bz_l^{\star\prime}, \bx_l^{(j+1)} = \bx_l^{\prime}$
		\Else
		\State {$\text { Let } \bz_{l-1}^{(j+1)}=\bz_{l-1}^{(j)} \text { and }\bz_l^{\star(j+1)}=\bz_l^{\star(j)}, \bx_l^{(j+1)}=\bx_l^{(j)}$}
		\EndIf
		\EndFor
		\Ensure posterior samples $\{\bx_l^{(i)}\}_{i=n_l}^{N_l}$
	\end{algorithmic}
\end{algorithm}

\section{Numerical Examples}\label{sec:Examples}
The code and data for reproducing all examples reported in this paper can be found  at
 \href{https://github.com/zabaras/MH-MDGM}{https://github.com/zabaras/MH-MDGM}.
In this section, we discuss and compare  the performance of the proposed multiscale method with that of the single-scale inference method. Our focus is on the estimation of the permeability field in a single phase, steady-state Darcy flow.  For any permeability field $\bm{K}$ on a 2D unit square domain $\mathcal{S}=[0,1]^2$, the pressure field $p$ and velocity field $\bm{v}$ are governed by the following equations:  
\begin{eqnarray}
\bm{v}(\bm{\bm{s}}) &=&-\bm{K}(\bm{s}) \nabla p(\bm{\bm{s}}), \quad \bm{\bm{s}} \in \mathcal{S}, \label{eq:darcyA}\\
\nabla \cdot \bm{v}(\bm{s}) &=& f(\bm{s}), \quad \bm{s} \in \mathcal{S},  
\label{eq:darcyB}
\end{eqnarray}
\noindent  
with boundary conditions
\begin{eqnarray}
\bm{v}(\bm{s}) \cdot \hat{\bm{n}}&=&0, \quad \bm{s} \in \Gamma_{N}, \label{eq:BoundaryA}\\
  p(\bm{s}) &=& 0, \quad \bm{s} \in \Gamma_{D},
\label{eq:BoundaryB}
\end{eqnarray} 
where $\hat{\bm{n}}$ is the unit normal vector to the Neumann boundary $\Gamma_{N}$. The Neumann boundary $\Gamma_{N}$ consists of the top and bottom boundaries and the Dirichlet boundary $\Gamma_{D}$ consists of the right and left  walls. We consider a source term $f(\bm{s}) = 3$. For the forward model $\mathcal{F}$, given any permeability field $\bm{K}$,   Eqs.~\eqref{eq:darcyA} and~\eqref{eq:darcyB} are solved by a mixed finite element formulation implemented in FEniCS~\cite{alnaes2015fenics} with third-order Raviart-Thomas elements for the velocity $\bm{v}$, and forth-order discontinuous elements for the pressure $p$. The computational cost for solving this forward model with different discretizations is reported in Table~\ref{forward cost}.

\begin{table}[h]
	\caption{Computational cost of solving the forward model in  Eqs.~\eqref{eq:darcyA} and~\eqref{eq:darcyB} with different discretizations.} 
	\centering	
	\begin{tabular}{cccc}  
		\hline
		Discretization & Output dimension   & Seconds &  Normalized time\\ \hline
		$64 \times 64$ & $3 \times 4096$ & $3.10$ &  $23.85$\\
		$32 \times 32$ & $3 \times 1024$ & $0.67$ &  $5.15$\\
		$16 \times 16$ & $3 \times 256 $ & $0.13$ &  $1.0$\\
		\hline
	\end{tabular}
	\label{forward cost}
\end{table}

To enforce the non-negative permeability constraint, we consider the log-permeability  as the main parameter of interest in the inverse problem with $\bx = log(\bm{K})$. The inverse problem for the above model is to infer the unknown log-permeability field given noisy pressure measurements at some sensor locations. In this paper, the exact log-permeability field $\bx_{exact}$ is defined in a $64 \times 64$ 
uniform grid.
The generated data by the generative model are usually smooth and blurry compared to the original data due to information compression. We consider $64$ pressure observations that are uniformly located in $[0.0625+0.125i, 0.0625+0.125i], i=0,1,2,\ldots,7$. A $5\%$ independent additive Gaussian random noise is considered on these $64$ pressure observations which are obtained by the above forward model given the exact log-permeability field. 

\subsection{Test problem 1: Gaussian Random Field (GRF)} 
For GRF based log-permeability data, one can adopt the KLE method for a reduced-order representation. Bayesian inference often works well on such low-dimensional inversion tasks.
However, the KLE expansion requires a-priori knowledge of the length scale in the GRF covariance function. This information is of course not available in most practical applications. 
Data-driven methods like DGM do not require such restrictive assumptions. The only prior knowledge available is based on the given training dataset. 

In this section, we apply the proposed method on a two-scales scenario and provide a  comparison with the reference case of the single-scale method. In the two-scales test case, we consider the $16 \times 16$ uniform grid as the coarsest-scale and the $64 \times 64$ uniform grid as the finest-scale. The finest-scale is the same with the scale of the exact log-permeability field. Using the proposed method, the generative models are trained on $16 \times 16$ and $64 \times 64$ $ (16-64)$ resolutions. We use the pre-trained MDGM with MCMC for Bayesian inference from coarse to fine, and compare the results with the single scale ($64 \times 64$ grid) method with focus of our investigation on computational efficiency, accuracy and convergence.

\subsubsection{Multiscale dataset}
As discussed in Section~\ref{sec:MDGM}, we adopt an unsupervised learning method for parameter representation before addressing the solution of the inverse problem. In particular, one can train a generative model for the sampling of log-permeability fields. For such an unsupervised learning problem, we only need to obtain the dataset $\{\bxi\}_{i=1}^{N}$ based on historical sample information or the underlying prior distribution $\pi(\bx)$.
In this example, we assume that the log-permeability field is a Gaussian random field, i.e.
$ \bx(\bm{s}) \sim \mathcal{GP}\left( m(\bm{s}), k\left(\bm{s}_1, \bm{s}_2\right)\right)$, where $ m(\bm{s})$ and $k(\bm{s}_1, \bm{s}_2)$ are the mean and covariance functions, respectively. $\bm{s}_1 = (x_1, y_1)$ and $\bm{s}_2 = (x_2, y_2)$ denote two arbitrary spatial locations. The covariance function $k(\bm{s}_1, \bm{s}_2)$ in this paper is taken as:

\begin{equation}
    k(\bm{s}_1, \bm{s}_2)= \sigma_{\log (K)}^{2} \exp \left(-\sqrt{\left(\frac{x_1-x_2}{l_{1}}\right)^{2}+\left(\frac{y_1-y_2}{l_{2}}\right)^{2}}\right),
    \label{eq_cov}
\end{equation}
where $ \sigma_{\log (K)}^{2}$ is the variance,  and $l_{1}$ and $l_{2}$ are the length scales along the $x$ and $y$ axes, respectively.  In this example, we set $ m(\bm{s}) = 1$ and $ \sigma_{\log (K)}^{2} = 0.5$. As   discussed earlier, it is hard to tackle the multiple length scales setup using the KLE method, while the deep generative model has a distinct superiority on the prior assumption since the dataset embodies these assumptions or information in a natural way. We assume that the length scales are not  fixed in the prior distribution.  Let the length scales be $l_1 = l_2 = 0.2 +0.01i, i= 0,1,2,\ldots,9$. The finest-scale parameter $\bx_2$ is uniformly discretized into an $H \times W = 64 \times 64$ grid. For each length scale, we generate $2500$ samples for the training dataset $\{\bx_2^{(i)}\}_{i=1}^{N}$, where $N$ is $25000$. The exact log-permeability field $\bx_{exact}$ for the Bayesian inversion has never been seen in the training procedure. We set  the $16 \times 16$ grid as the coarsest-scale in the parameter estimation.   We adopt the upscaling operator in Eq.~\eqref{arithmetic average} with $n_e=16$ for each realization in $\{\bx_2^{(i)}\}_{i=1}^{N}$, and then obtain the $16 \times 16$ grid dataset $\{\bx_1^{(i)}\}_{i=1}^{N}$. This  constitutes  the training dataset for the generative model in the coarse-scale.

\subsubsection{Training and results of the MDGM}
With the training datasets available on each scale, we can train the generative model from coarse- to fine-scale. The coarsest-scale generative model only involves the single-scale parameters $\bx_1$. The training procedure of the single-scale generative model is illustrated in Algorithm~\ref{DGM}. The schematic network architecture is shown in Fig.~\ref{rt}. The detailed encoder and decoder networks can be seen in \ref{app:nn}. All of the encoder and decoder neural networks in this paper are trained on a NVIDIA GeForce GTX $1080$ Ti GPU card. The loss function is defined in Eq.~\eqref{eq:vae_loss} for training the single-scale generative model. For all training procedures in this paper, we set  the batch size $n$ in the loss function to $64$, and the sampling size $m$ to $1$.  For the optimization of all neural networks, the Adam optimizer~\cite{kingma2014adam} was adopted with a learning rate of $2 \times 10^{-4}$. The above setup is kept consistent for all training models in this paper. The only differences are the training epochs and hyperparameters $\tilde{\beta}$ in the loss functions, which will be specified in the remaining cases.  All the models in the GRF case are trained with $30$ epoches, and the hyperparameter $\tilde{\beta}$ is $0.5$. It takes  about $24$ minutes and $73$ minutes for the training of a single-scale generative model on $16 \times 16$ and $64 \times 64$ resolutions, respectively. 

For the $64 \times 64$ single-scale generative model, we use the pre-trained model $p_{ \btheta}(\bx|\bz)$ for MCMC exploration to estimate directly the parameter in the $64 \times 64$ resolution with $\bz \in \mathbb{R}^{256}$.  For the $16 \times 16$ single-scale generative model, the latent variables $\bz_1 \in \mathbb{R}^{16}$ serve as the bridge  to the fine-scale estimation. We use the pre-trained model $p_{ \btheta_1}(\bx_1|\bz_1)$ for MCMC exploration of the posterior in the coarse-scale. For the training of the MDGM, the pre-trained model $q_{\bphi_1}(\bz_1|\bx_1)$ is part of the encoder network as illustrated in Fig.~\ref{multi_vae1}, and an input to  Algorithm~\ref{multi_vae_algorithm}. The loss function in Eq.~\eqref{eq:MDGMloss} is used for training the models $q_{\bphi_2^{\star}}(\bz_2^{\star}|\bx_2)$ and $p_{ \btheta_2}(\bx_2|\bz_1,\bz_2^{\star})$, where $\bz_2^{\star} \in \mathbb{R}^{256}$. The training time  for  the models $q_{\bphi_2^{\star}}(\bz_2^{\star}|\bx_2)$ and $p_{ \btheta_2}(\bx_2|\bz_1,\bz_2^{\star})$ is  about $78$ minutes. 

Once   the generative models are obtained in the different scales, the parameters can be generated in each scale by first sampling the corresponding latent variable, and then by decoding it using the corresponding generative model. Here, we illustrate the MDGM results in Fig.~\ref{Gaussian_MDGM}. For the pre-trained generative model $p_{ \btheta_1}(\bx_1|\bz_1)$, we can sample the latent variables $\bz_1$ from $\mathcal{N}(\bm{0},\bm{I})$, and use these latent variables as input to the model $p_{ \btheta_1}(\bx_1|\bz_1)$. The first row of Fig.~\ref{Gaussian_MDGM} shows $4$ realizations with $16 \times 16$ resolution that are output of this decoder network. It is  apparent that all of these images have distinct features in the spatial distribution of high- and low-value regions. However, they   are very smooth so that one cannot capture any local information. Fortunately, this is a good choice to highlight the obvious global features, which are of great importance in Bayesian inversion. It is a fundamental and necessary requirement for spatially-varying parameter estimates to be consistent with the exact parameter in capturing global features. 

The performance of the $64 \times 64$ resolution generative model $p_{ \btheta_2}(\bx_2|\bz_1,\bz_2^{\star})$ in this two-scales MDGM is shown in the second row of Fig.~\ref{Gaussian_MDGM}. The latent variable $\bz_1$ is the mean of the distribution $q_{\bphi_1}(\bz_1|\bx_1)$, i.e. $\bz_1 = \mathop{\arg\max} q_{\bphi_1}(\bz_1|\bx_1)$, where $\bx_1 \in \mathbb{R}^{16 \times 16}$ is the first image in this row. The latent variable $\bz_2^{\star}$ is sampled from the Gaussian distribution $\mathcal{N}(\bm{0},\bm{I})$. We observe that that the three samples generated by the model $p_{ \btheta_2}(\bx_2|\bz_1,\bz_2^{\star})$ have some particularities. Unlike samples in the first row, they keep similar spatial distribution of high- and low-value regions that inherit from the first $16 \times 16$ resolution image $\bx_1$ while they are refined locally in a diverse manner. The model $q_{\bphi_1}(\bz_1|\bx_1)$ acts as a messenger in the MDGM. The information of low-resolution $\bx_1$ that captures global features is encoded by $\bz_1$, which together with the random variables $\bz_2^{\star}$ that is used to supplement local details are decoded by the generative model $p_{ \btheta_2}(\bx_2|\bz_1,\bz_2^{\star})$. This demonstrates that these two latent variables have different missions in the generative model, in particular $\bz_1$ plays an important role in the global feature generation, while $\bz_2^{\star}$ contributes to local details. This disentangled representation~\cite{higgins2018towards,higgins2017beta,bengio2013representation,suter2019robustly} for local and global features assists in an accurate and efficient Bayesian multiscale estimation.

\begin{figure}[!htbp]
	\centering
	
      \includegraphics[width=1.0\textwidth]{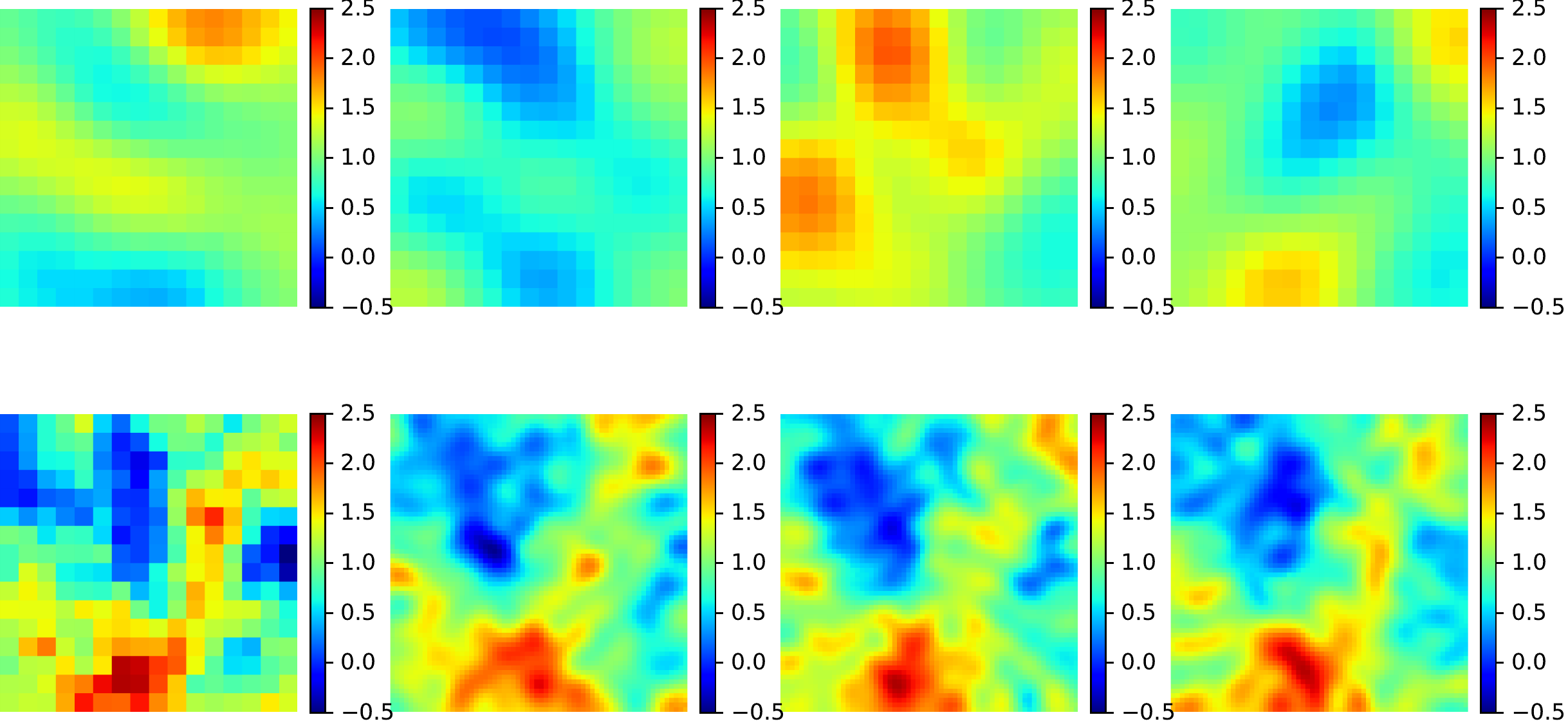}
    \caption{First row: $16 \times 16$ resolution  random samples using the generative model $p_{ \btheta_1}(\bx_1|\bz_1)$, where  $\bz_1\in \mathbb{R}^{16}$ is sampled from   $\mathcal{N}(\bm{0},\bm{I})$. Second row:  $64 \times 64$ resolution random samples using the generative model $p_{ \btheta_2}(\bx_2|\bz_2)$, where   $\bz_2 = (\bz_1,\bz_2^{\star})$.  We let $\bz_1 = \mathop{\arg\max} q_{\bphi_1}(\bz_1|\bx_1)$, where  $\bx_1$ is the first image in this row  and $\bz_2^{\star} \in \mathbb{R}^{256}$ is sampled from   $\mathcal{N}(\bm{0},\bm{I})$.}
    \label{Gaussian_MDGM}
\end{figure}

\subsubsection{The inversion results and discussion} 
The reference experiment considered is the single-scale method. We use Algorithm~\ref{mcmc} with the pre-trained model $\bx = \mu_{ \btheta}(\bz)$ to estimate directly $\pi(\bz|\mathcal{D}_{obs})$, where $\bx \in \mathbb{R}^{64 \times 64}$ and $\bz \in \mathbb{R}^{256}$. All the implementations using MCMC treat the latent variables as random variables, and use the pre-trained model to generate $\bx$ for evaluating the likelihood function. The proposal distribution for the random variables applied in this paper is preconditioned Crank-Nicolson (pCN) that is defined below: 

 \begin{equation}
 \label{pCN}
     \bz^{\prime} = \sqrt{1-\gamma^2} \bz + \gamma \zeta,
 \end{equation}
 where $\bz$ and $\bz^{\prime}$ are current and proposed next states, respectively, and $\zeta \sim \mathcal{N}(\bm{0},\bm{I})$. The step size of the random movement from the current state to a new position is controlled by the free parameter $\gamma$. We set $\gamma$ be $0.08$ for the first $50\%$  and $0.04$ for the last $50\%$ states in the Markov chain for all implementations using Algorithm~\ref{mcmc}. We run $10000$ iterations in MCMC to ensure its convergence. For all implementations of the MCMC algorithm, we collected the last $2000$ states as the posterior samples. Fig.~\ref{Gaussian-reference64} shows the estimation results using the reference method. It can be seen that the integration of the deep generative model with MCMC leads to a reasonable estimation of the spatially varying parameter.
 
 \begin{figure}[h]
	\centering
	\includegraphics[width=0.8\linewidth, height = 0.5\linewidth]{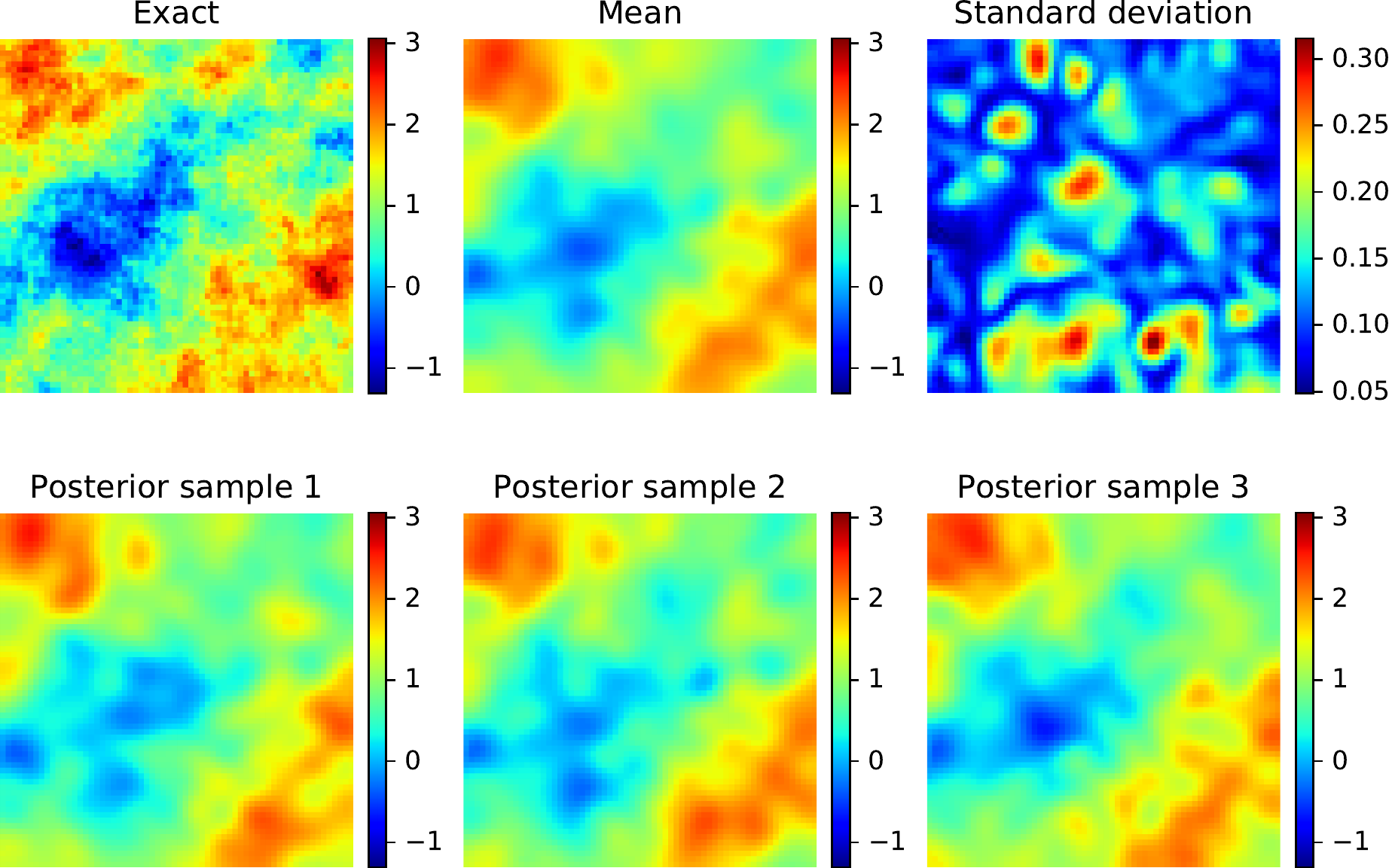}
	\caption{The reference single-scale estimation result with the desired $64 \times 64$ grid. The mean and standard deviation of the estimated posterior distribution that are computed using the posterior samples are shown in the first row. Realizations  from the posterior  are shown in the second row. Inference with respect to  $\bz \in \mathbb{R}^{256}$ is performed using  Algorithm~\ref{mcmc}.}
	\label{Gaussian-reference64}
\end{figure}

The multiscale parameter estimation is implemented next from coarse- to fine-scale. We use Algorithm~\ref{mcmc} with pre-trained model $\bx_1 = \mu_{ \btheta_1}(\bz_1)$ to estimate $\pi(\bz_1|\mathcal{D}_{obs})$, where $\bx_1 \in \mathbb{R}^{16 \times 16}$ and $\bz_1 \in \mathbb{R}^{16}$. A Markov chain with length $7000$ is constructed for $\bz_1$ to estimate the coarse-scale Gaussian log-permeability field. This result  explores various global patterns of the log-permeability. Fig.~\ref{Gaussian-16} shows that the estimation result is as expected. It is clear that the global spatial distribution of low- or high-values is located at three different regions, consistent with the exact log-permeability field. The estimated log-permeability with $16 \times 16$ grid   captures the important features efficiently since it uses the forward model with only a $16 \times 16$ grid. However, this occurs at the  sacrifice of local information. In addition, the coarse-scale forward model introduces  computational error. To resolve these  issues, one needs to infer the fine-scale parameter using the high-resolution generative model and the corresponding precise forward solver.

\begin{figure}[h]
	\centering
	\includegraphics[width=0.8\linewidth, height = 0.5\linewidth]{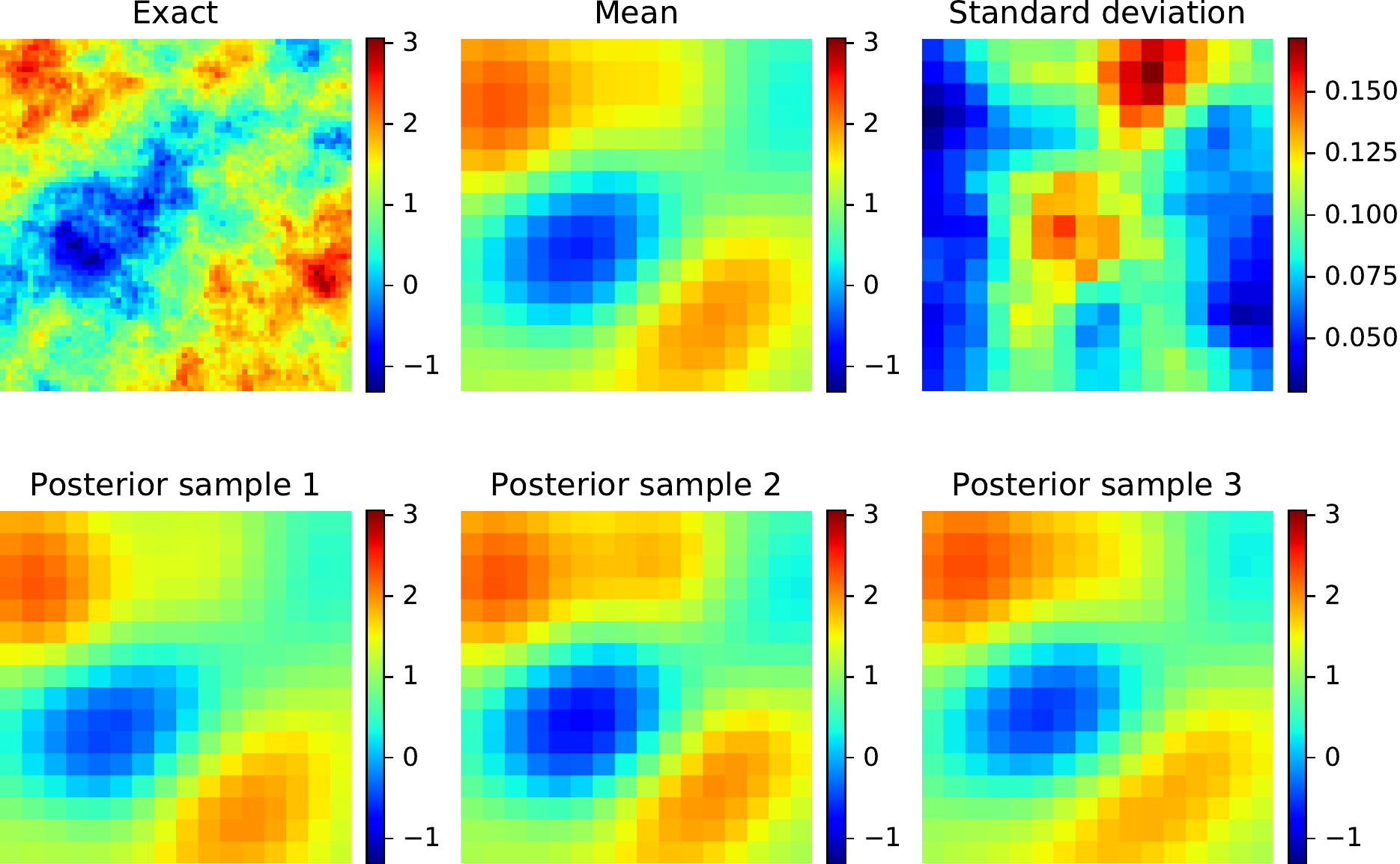}
	\caption{The two-scale estimation result with $16 \times 16$ coarse-grid and $64 \times 64$ fine-grid. The coarse-scale results are shown here. Inference with respect to   $\bz_1 \in \mathbb{R}^{16}$ is performed  using  Algorithm~\ref{mcmc}.}
	\label{Gaussian-16}
\end{figure}

To correct and refine the parameter estimation, we use Algorithm~\ref{multi_mcmc} with the  pre-trained model $\bx_2 = \mu_{ \btheta_2}(\bz_2)$ to estimate $\pi(\bz_2|\mathcal{D}_{obs})$, where $\bx_2 \in \mathbb{R}^{64 \times 64}$, $\bz_2 = (\bz_1, \bz_2^{\star})$, and $\bz_1 \in \mathbb{R}^{16}$, and $\bz_2^{\star} \in \mathbb{R}^{256}$. Since we have obtained the posterior samples in the coarse-scale, the estimation in the fine-scale  takes advantage of the coarse estimation. We need to assign two proposal distributions for $\bz_1$ and $\bz_2^{\star}$. The two proposal distributions we used in this paper for Algorithm~\ref{multi_mcmc} are the pCN in Eq.~\eqref{pCN} with different step sizes. The fixed step size $\gamma$ for the  low-dimensional latent variable $\bz_{l-1}$ is $0.01$, while the adaptive step size $\gamma$ for the high-dimensional latent variable $\bz_{l}^{\star}$ is $0.08$ for the first $50\%$ and $0.04$ for the last $50\%$ of the states in the Markov chain. The results in Fig.~\ref{Gaussian-16-64} indicate that the estimated parameter with $64 \times 64$ grid using the proposed method has even better performance than the reference single-scale results in Fig.~\ref{Gaussian-reference64}. The details of the low-value region are closer to the exact log-permeability field. 

\begin{figure}[h]
	\centering
	\includegraphics[width=0.8\linewidth, height = 0.5\linewidth]{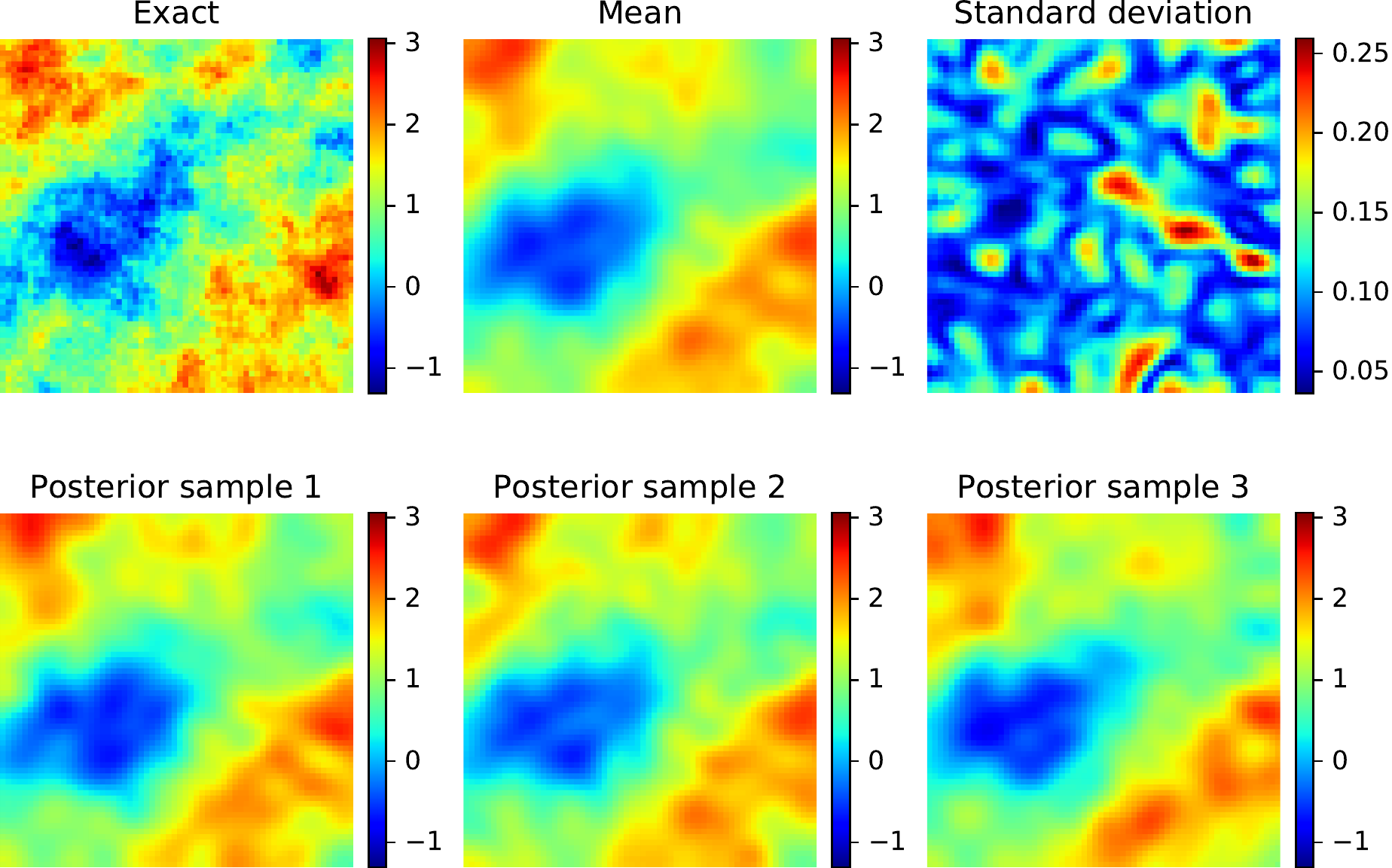}
	\caption{The two-scale estimation result with $16 \times 16$ coarse-grid and $64 \times 64$ fine-grid. The fine-scale results are shown here. Inference with respect to two latent variables $\bz_1 \in \mathbb{R}^{16}$ and $\bz_2^{\star} \in \mathbb{R}^{256}$ is performed using  Algorithm~\ref{multi_mcmc}.}
	\label{Gaussian-16-64}
\end{figure}

To compare the estimation result with the exact field, we provide an illustration in Fig.~\ref{Gaussian_error_bar} that shows the values of the log-parameter field from the left top corner to the right bottom corner. We notice that the true values curve (black line) is very sharp while the posterior mean is much smoother. Compared to other deep generative models~\cite{goodfellow2014generative, dinh2014nice}, the shortcoming of VAE~\cite{kingma2013auto} is that it generates blurry samples since the bottleneck layer captures a compressed latent encoding. However, we note that the result obtained from the proposed method is better than the reference method as its mean is much closer to the exact solution.

\begin{figure}[h]
	\centering
	\subfloat[]{
    \includegraphics[width=0.4\textwidth]{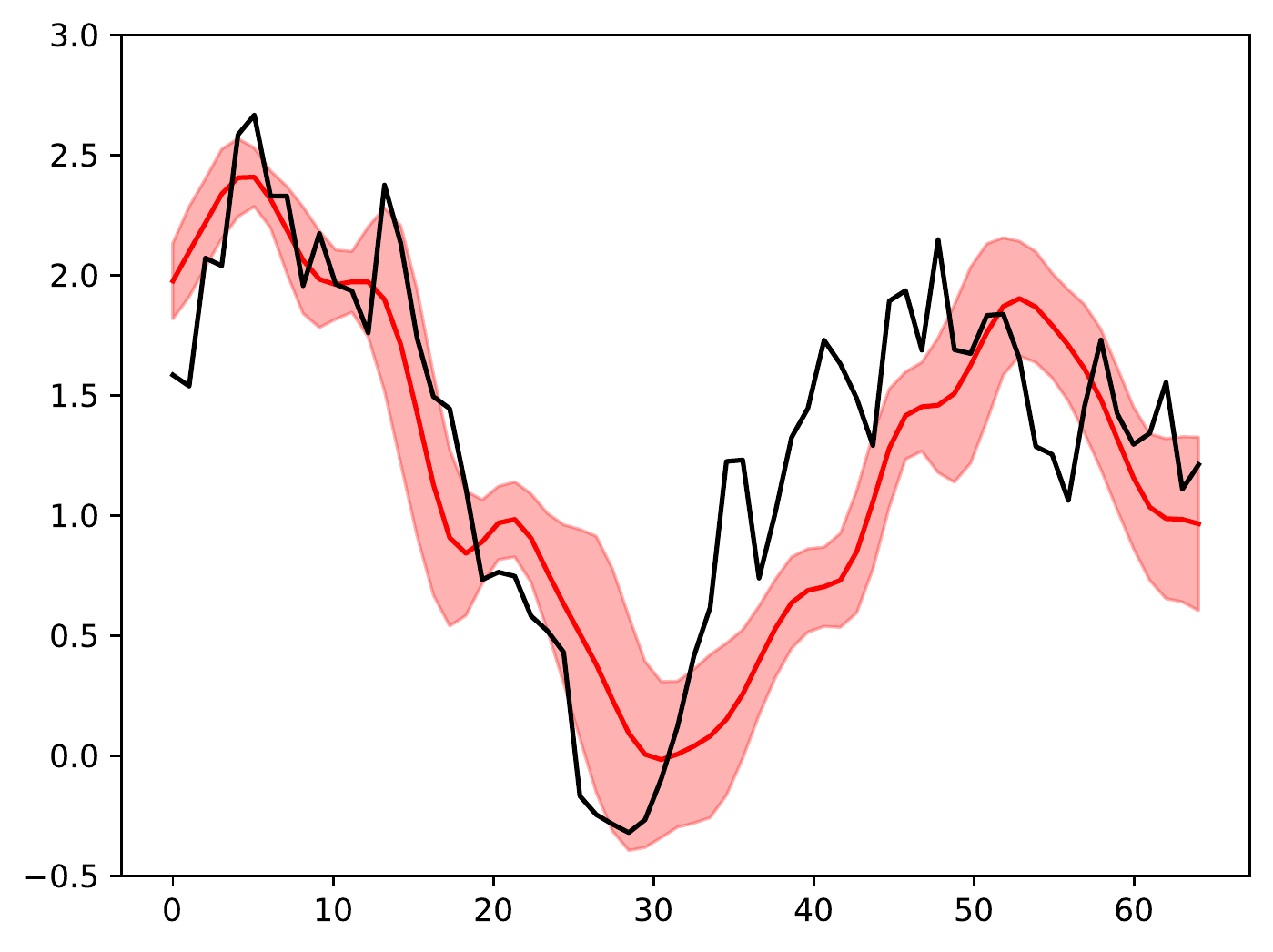}}
    \quad
    \subfloat[]{
    \includegraphics[width=0.4\textwidth]{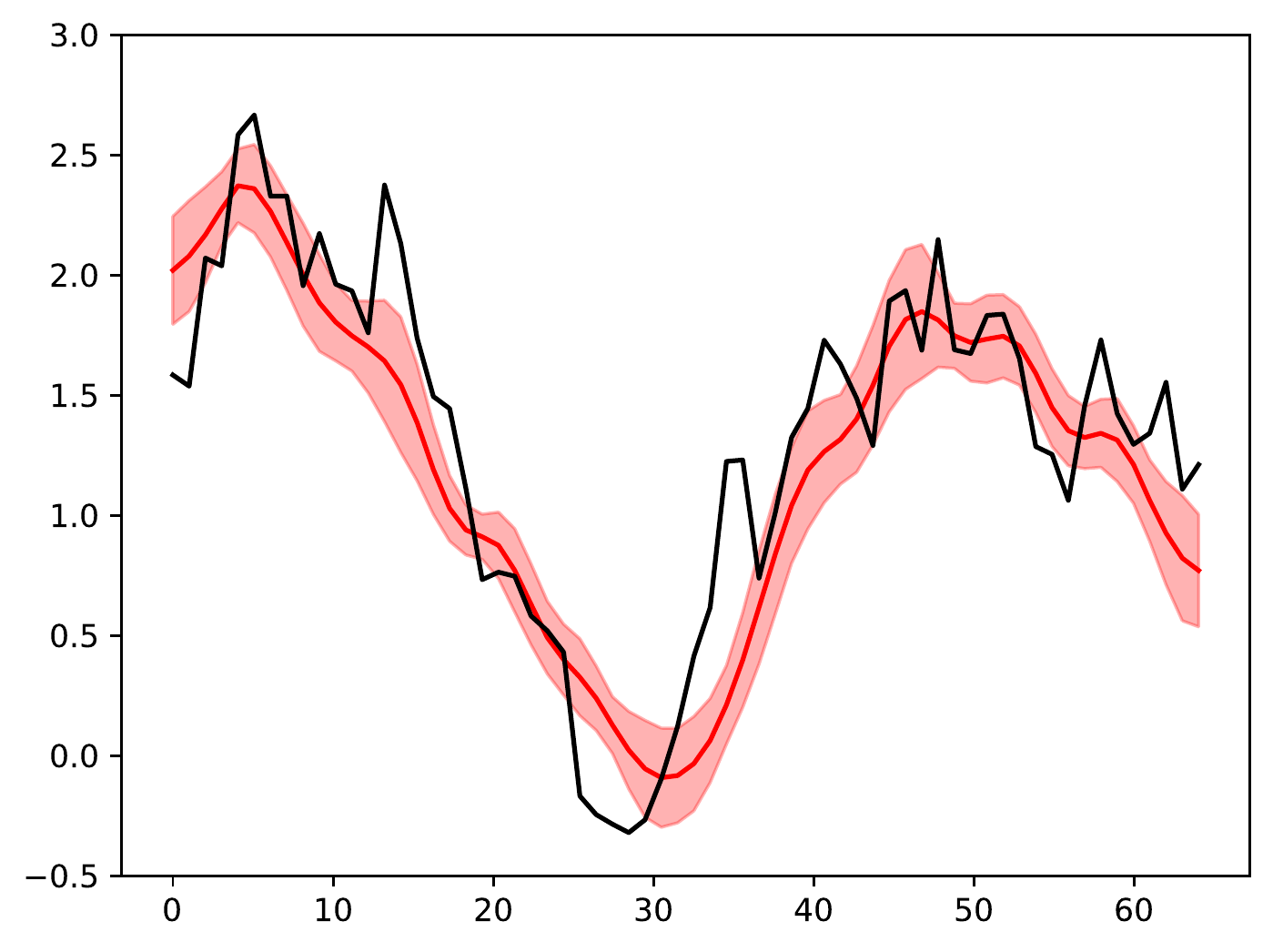}}
    \caption{The log-permeability from the left top corner to the right bottom corner on a $64 \times 64$ grid. The results from left to right are obtained by  (a) reference $(64)$ and (b) two scales $(16-64)$ estimation, respectively. The black and red lines show the true and the posterior mean log-permeability, respectively. The shaded region shows values within two standard deviations of the mean.}
    \label{Gaussian_error_bar}
\end{figure}
 
 The main cost of the Bayesian inference using MCMC comes from the forward model evaluation. The forward model's computational cost in Table~\ref{Gaussian_cost} suggests that the single-scale method takes about $1.9$ times the computational cost of the proposed two-scale method. The acceptance rate of MCMC for all implementations is shown in Table~\ref{Gaussian_acceptance_ratio}. We can note that the proposed method has a much  higher acceptance rate than the single-scale method in the desired scale with $64 \times 64$ grid. The reason is that the inference for the fine-scale parameter is facilitated  from the designed latent space of the MDGM that has two latent variables with different influences in the fine-scale parameter generation. The low-dimensional latent variable impacts the global features observed in the fine-scale. We assigned a small step size for its pCN proposal distribution since the global features that are inherited from the coarse-scale estimation only need slight adaption, while the high-dimensional latent variable estimated for the local refinement is assigned with a big step size. The reference method often rejects the proposed samples and gets trapped in local modes. However, in the fine-scale estimation in the proposed method,  the main changes in the local features will lead to small changes in the likelihood function so that most of the proposed samples are accepted. 
 
\begin{table}[h]
	\caption{The iterations (its) and approximated cpu time in seconds(s) for solving the forward model in different experiments for the GRF test problem.} 
	\centering	
	\begin{tabular}{ccccc}  
		\hline
		Experiment     & $16 \times 16$ &  $64 \times 64$ & Total\\ \hline
		\multirow{2}{*}{one scale ($64$)$^{\star}$}  & - & $10000$ its & $10000$ its \\
		 & - & $31000$ s & $31000$ s  \\\hline
		\multirow{2}{*}{two scales ($16-64$)} & $7000$ its  & $5000$ its & $12000$ its \\
		& $910$ & $15500$ s & $16410$ s  \\\hline
	\end{tabular}
	\label{Gaussian_cost}
\end{table}

\begin{table}[h]
	\caption{The acceptance ratio of the MH algorithm in different experiments for the GRF test problem.} 
	\centering	
	\begin{tabular}{cccc}  
		\hline
		Experiment & $16 \times 16$   & $64 \times 64$\\ \hline
		one scale ($64$) & - & 30.6\% \\ \hline
		two scales ($16-64$) & 37.5\% & 67.8\% \\
		\hline
	\end{tabular}
	\label{Gaussian_acceptance_ratio}
\end{table}

The convergence of the Markov chain used to estimate parameters with a $64 \times 64$ grid is of concern since the forward evaluation on such a scale is very expensive. 
To assess the convergence in the desired scale with a $64 \times 64$ grid, we employ three metrics using  variables available in the iterative process. Since the observation data (noisy pressure measurements) is the only basis for parameter estimation, we compute the misfit between the observation data and its corresponding prediction using the sum of squared residuals of observations ($SSR_{obs}$) as the iteration proceeds:

\begin{equation}
  SSR_{obs} =  \sum_{i=1}^{N_{obs}}\left(\mathcal{F}(\mu(\bz))^{(i)} - \mathcal{D}_{obs}^{(i)}\right)^2, 
\end{equation}
where $N_{obs} = 64$ in this paper, and $\mathcal{F}(\mu(\bz))^{(i)}$ and $\mathcal{D}_{obs}^{(i)}$ are the $i$--th predicted pressure value and its corresponding true observation, respectively. Ideally, we expect that the value of $SSR_{obs}$ is close to $0$, which suggests there is no discrepancy between the predictions and observations. But this cannot be realized even using the exact input for $\mathcal{F}$ due to measurements noise. To remove the impact of noise, we use an enhanced metric i.e. the normalized sum of squared weighted residual ($NSSWR$)~\cite{mo2019integration, laloy2012high}:

\begin{equation}
  NSSWR = \frac{1}{SSWR_{ref}} \sum_{i=1}^{N_{obs}}\left(\frac{\mathcal{F}(\mu(\bz))^{(i)} - \mathcal{D}_{obs}^{(i)})}{\sigma_{n}^{(i)}}\right)^2, 
\end{equation}
where $SSWR_{ref} = \sum_{i=1}^{N_{obs}}(\frac{\mathcal{F}(\bx_{exact})^{(i)} - \mathcal{D}_{obs}^{(i)})}{\sigma_{n}^{(i)}})^2 $, $\bx_{exact}$ is the exact log-permeability field, and $\sigma_{n}^{(i)}$ is the standard deviation of Gaussian random noise imposed in the $i$--th observation. The $SSWR$ metric is normalized by the $SSWR_{ref}$. Thus, the inversion process has converged when the $NSSWR$ value is close to $1$. Since the target is to estimate the log-permeability field based on the given observations, we also consider the evaluation of the mismatch between the predicted and the exact log-permeability fields using the sum of squared residuals of parameter ($SSR_{para}$) as shown below:

\begin{equation}
  SSR_{para} =  \sum_{i=1}^{M}\left(\mu(\bz)^{(i)} - \bx_{exact}^{(i)}\right)^2,
\end{equation}
where $M = 4096$, since the desired scale is discretized to $64 \times 64$, and $\mu(\bz)^{(i)}$ and $\bx_{exact}^{(i)}$ are the  $i$--th value of the predicted and the exact log-permeability fields, respectively. 

Fig.~\ref{convergence_Gaussian} shows the convergence results in the GRF case. The comparison of the proposed method with the single-scale method is also given in Fig.~\ref{convergence_Gaussian}. It can be seen that better performance in all evaluation metrics is produced by the proposed method. 
A big distinction takes place in the decrease of $SSR_{para}$ illustrated in Fig.~\ref{convergence_Gaussian}(c). We show the  sampled log-permeability field with $64 \times 64$ grid at different iterations in the Markov chain in Fig.~\ref{Gaussian-state}. For the single-scale method, each dimension of the estimated latent variable $\bz \in \mathbb{R}^{256}$ has equivalent importance for the generation of the log-permeability. The random walk in such a high-dimensional space has difficulty in efficiently exploring the space and  identifying a good estimate and its uncertainty. Together with a random initial state sampled from the prior distribution, the exploration takes a long time to reach the stationary distribution.   The second row in Fig.~\ref{Gaussian-state} presents the state evolution for the two-scale method.  The first-state that inherited the coarse-scale estimation captures well most of the non-local features of  the exact log-permeability field. Thus with the iterations shown,  we only need to correct the local features to explore the posterior distribution.

\begin{figure}[h]
	\centering
	\subfloat[]{
    \includegraphics[width=0.3\textwidth]{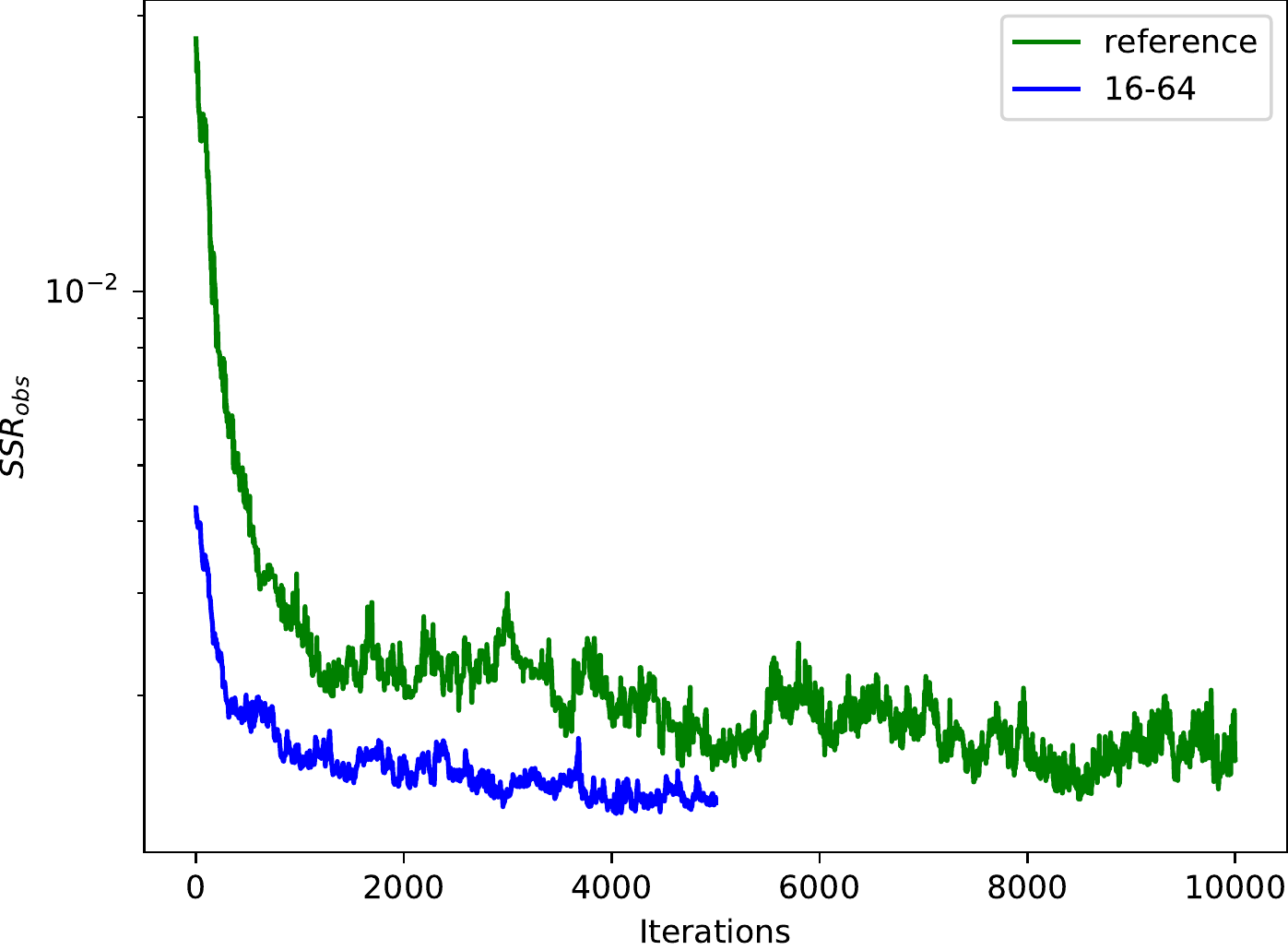}}
    \quad
    \subfloat[]{
    \includegraphics[width=0.3\textwidth]{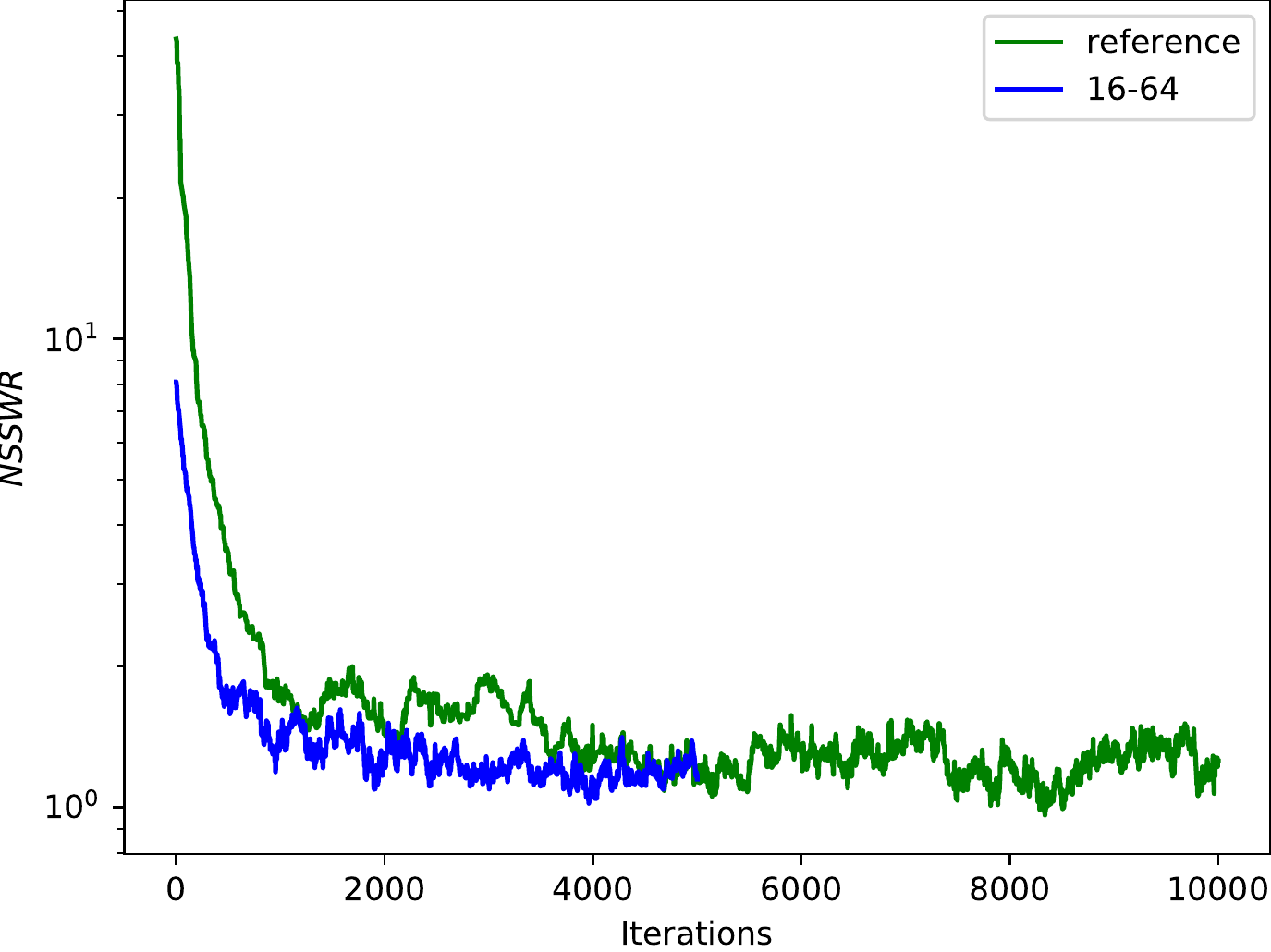}}
    \quad
    \subfloat[]{
    \label{convergence_Gaussian c}
    \includegraphics[width=0.3\textwidth]{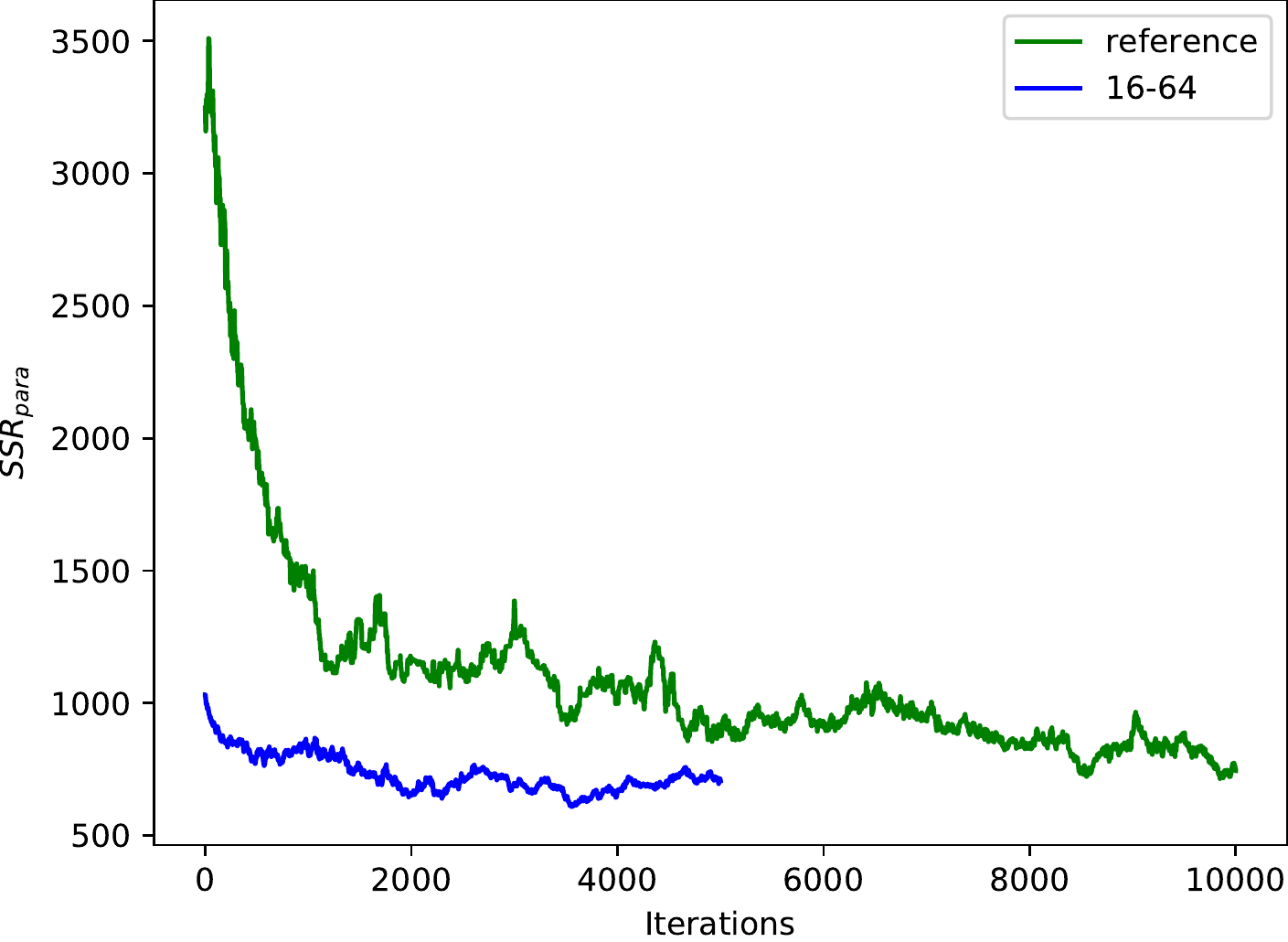}}
	\caption{The convergence of the Markov chain used for the Gaussian log-permeability field estimation with a $64 \times 64$ grid. The evaluation metrics from left to right are (a) the SSR of observable pressure values,  (b) NSSWR values, and (c) the SSR of parameter field.}
	\label{convergence_Gaussian}
\end{figure}

\begin{figure}[h]
	\centering
	\includegraphics[width=1.0\linewidth, height = 0.4\linewidth]{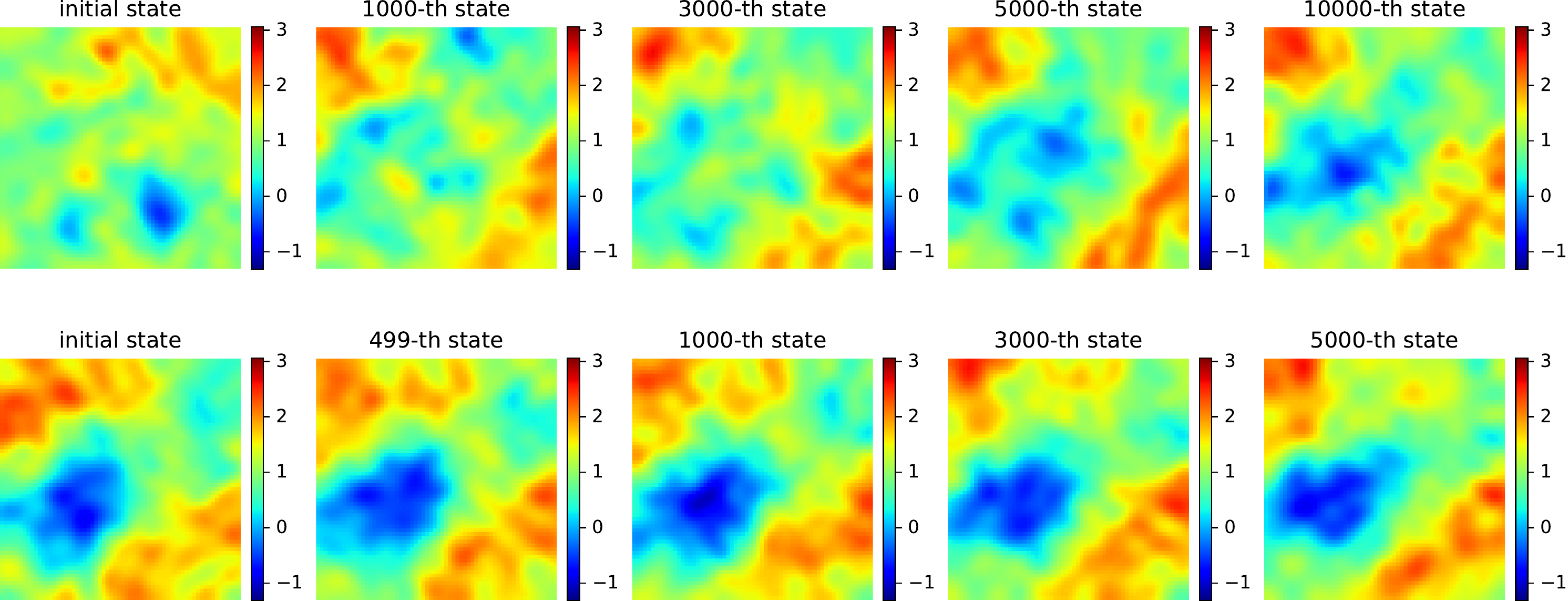}
	\caption{The states of the Gaussian log-permeability field with $64 \times 64$ grid during iterations in the Markov chain. The results from top to bottom row are obtained from (a) reference $(64)$ and (b) two scales $(16-64)$ experiment, respectively.}
	\label{Gaussian-state}
\end{figure}

\subsection{Test problem 2: Non-Gaussian Random Field} 
In the second test problem, we consider a channelized log-permeability field as the exact parameter in the inversion experiment. For such a non-Gaussian permeability field is often difficult to obtain a good parameterization using conventional methods such as sparse-grid approximations~\cite{wan2011bayesian}, wavelets~\cite{ellam2016bayesian}, or principal component analysis  (PCA)~\cite{emerick2017investigation,vo2015data,sarma2007new}. Furthermore, the Bayesian inference using random walk MCMC based on recent reported learning-based parameterization methods~\cite{mo2019integration,laloy2017inversion,tang2020deep, laloy2018training} for high-dimensional non-Gaussian parameters is not a good  choice. We show the
benefits of the multiscale method in the parameterization of non-Gaussian random fields and Bayesian inference. In this test example, we demonstrate the proposed method with two- and three-scales scenarios and provide comparisons with the reference case of a single-scale method. 

\subsubsection{Multiscale dataset}
 Suppose that the prior information for the channel location before any measurement is from the   image~\cite{laloy2018training} with size of $2500 \times 2500$ shown in Fig.~\ref{data_gene}(a). One can crop the large image with a fixed stride. With a $16$ stride in the horizontal and vertical directions, we obtained $23104$ training samples of size of $64 \times 64$. To provide sufficient data for the training of the generative model, we flip the entries in each row of the image in the left/right direction by the \textsl{fliplr} operation~\footnote{https://numpy.org/doc/1.18/reference/generated/numpy.fliplr.html} implemented in Numpy package to obtain a new image, and cropped this image to obtain additional $23104$ samples.  A sample cropped by this procedure is illustrated in Fig.~\ref{data_gene}(b). The binary image depicts channels with white regions. We assume that the channelized regions have different log-permeabilities 
resulting in high- and low-permeability values in white and black regions, respectively. We set the log-permeability values for each region by independently sampling  from two  Gaussian Random Fields (GRFs). The means of the GRFs for the high-permeability channelized regions and low-permeability regions are $4$ and $0$, respectively. The covariance function in Eq.~\eqref{eq_cov} with length scales $l_1$ and $l_2$ equal to $0.3$ is applied. The variance is $0.5$ for both GRFs. By imposing two GRFs for the binary image samples, we can generate the training data as shown in Fig.~\ref{data_gene}(c). We assume the generated samples are i.i.d.~sampled from the underlying prior distribution $\pi(\bx)$. We take $40000$ samples from generated $46208$ realizations as training data.
The exact log-permeability field $\bx_{exact}$ for Bayesian inversion is sampled from the remaining samples, which has never been seen in the training procedure. The $\bx_{exact}$ we used in this test problem is the image shown in Fig.~\ref{data_gene}(c). 

\begin{figure}[!htp]
	\centering
	\includegraphics[width=0.8\linewidth, height = 0.5\linewidth]{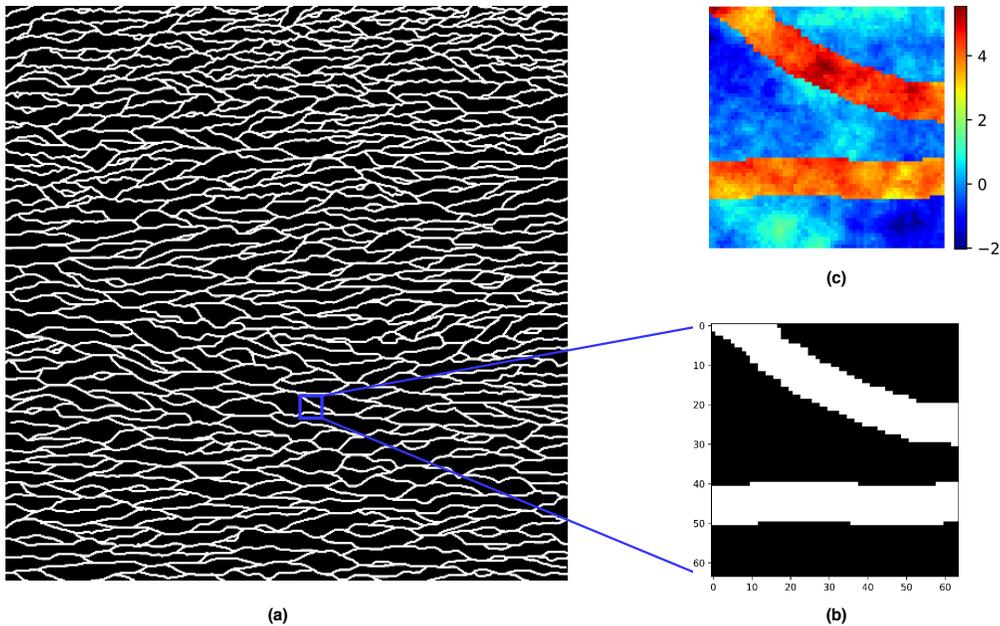}
	\caption{(a) The large image contains the prior information of the  channel location (b) cropped binary image samples from the large image (c) A sample from the underlying prior distribution $\pi(\bx)$ by assigning two different GRFs to the binary image.}
	\label{data_gene}
\end{figure}

The training of MDGM needs data with different discretizations. In this test example, we consider two types of MDGM, which are a two-scale model with $16 \times 16$ grid and $64 \times 64$ grid $(16-64)$ and a three-scale model with  $16 \times 16$ grid, $32 \times 32$ grid and $64 \times 64$ grid $(16-32-64)$. To generate these datasets, one can adopt the upscaling operator in Eq.~\eqref{arithmetic average} with $n_e=4$ over the original $64 \times 64$ grid data, and apply it again with $n_e=4$ in the generated $32 \times 32$ grid data to obtain the $16 \times 16$ grid data. An example with different discretizations has been illustrated in Fig.~\ref{upscaling}.

\subsubsection{Training and results of the MDGM}
Using the generated training datasets, we applied the same neural network used in the GRF case to train the MDGM. We will discuss the training procedure and performance below for the two-  and three-scales cases. The reference single-scale generative model $p_{ \btheta}(\bx|\bz)$ requires $3.15$ hours to train $50$  epochs using the loss function in Eq.~\eqref{eq:vae_loss} and Algorithm~\ref{DGM}, where $\bx \in \mathbb{R}^{64 \times 64}$, $\bz \in \mathbb{R}^{256}$, and $\tilde{\beta} = 0.5$.

\noindent \emph{Two scales $(16-64)$.} In this example, the $\bx_1, \bx_2$ denote the parameters with $16 \times 16$ grid and $64 \times 64$ grid, respectively. For the $\bx_1$, the generative model is trained using Algorithm~\ref{DGM} with $\tilde{\beta} = 1$ in Eq.~\eqref{eq:vae_loss}. It takes about $37$ minutes for $30$ training epochs. We obtained $p_{ \btheta_1}(\bx_1|\bz_1)$ and $q_{\bphi_1}(\bz_1|\bx_1)$, where $\bx_1 \in \mathbb{R}^{16 \times 16}$ and $\bz_1 \in \mathbb{R}^{16}$. The model $q_{\bphi_1}(\bz_1|\bx_1)$ is the input to  Algorithm~\ref{multi_vae_algorithm} used for training the finer-scale generative model. We use the pre-trained model $p_{ \btheta_1}(\bx_1|\bz_1)$ to reconstruct the parameters to compute the likelihood function when we use Algorithm~\ref{mcmc} to estimate $\pi(\bz_1|\mathcal{D}_{obs})$, where $\bz_1 \in \mathbb{R}^{16 }$. The generated samples using the model $p_{ \btheta_1}(\bx_1|\bz_1)$ are shown in the first row of Fig.~\ref{Channel_MDGM}(a). As expected, these samples present the most important features i.e. the location of the channels without much local information. The finer-scale generative model $p_{ \btheta_2}(\bx_2|\bz_1,\bz_2^{\star})$ is trained using Algorithm~\ref{multi_vae_algorithm} with the loss  function in Eq.~\eqref{eq:MDGMloss} for the estimation refinement, where $\tilde{\beta} = 2.5$, $\bx_2 \in \mathbb{R}^{64 \times 64}$, and $\bz_2^{\star} \in \mathbb{R}^{256}$. Training $50$ epochs takes about $3.3$ hours. The performance of this   is shown in the second row of Fig.~\ref{Channel_MDGM}(a), where the first image in this row is a field with a $16 \times 16$ grid. Encoding this image into a certain $\bz_1$ and  together with three randomly sampled latent variables $\bz_2^{\star} \sim \mathcal{N}(\bm{0},\bm{I})$ generated the last three samples in this row. We can see that the generated fine-scale samples   keep similar channels with the given coarse-scale image, while their local refinement shows sufficient diversity. This  indicates  that the coarse-scale information can be captured by the latent variable $\bz_1$ using the encoder model $q_{\bphi_1}(\bz_1|\bx_1)$. As before, in the finer-scale generative model, the low-dimensional latent variable $\bz_1$  dominates the global features (channels in this example), while the high-dimensional latent variables $\bz_2^{\star}$ are used to capture  local features.

\begin{figure}[!htbp]
	\centering
	\subfloat[]{\label{Channel_MDGM_16_64}%
    \includegraphics[width=1.0\textwidth]{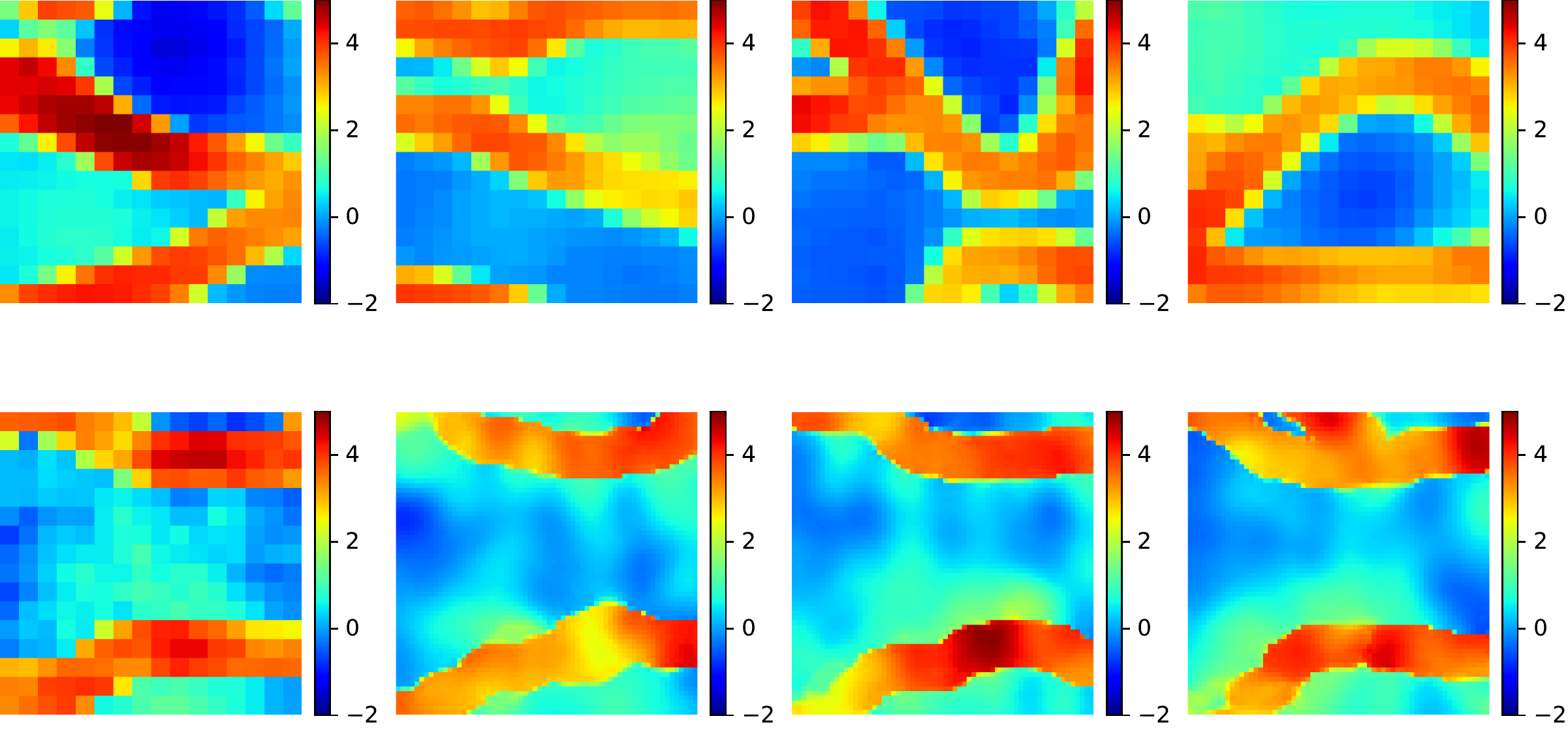}}
    \quad
    \subfloat[]{\label{Channel_MDGM_16_32_64}%
    \includegraphics[width=1.0\textwidth]{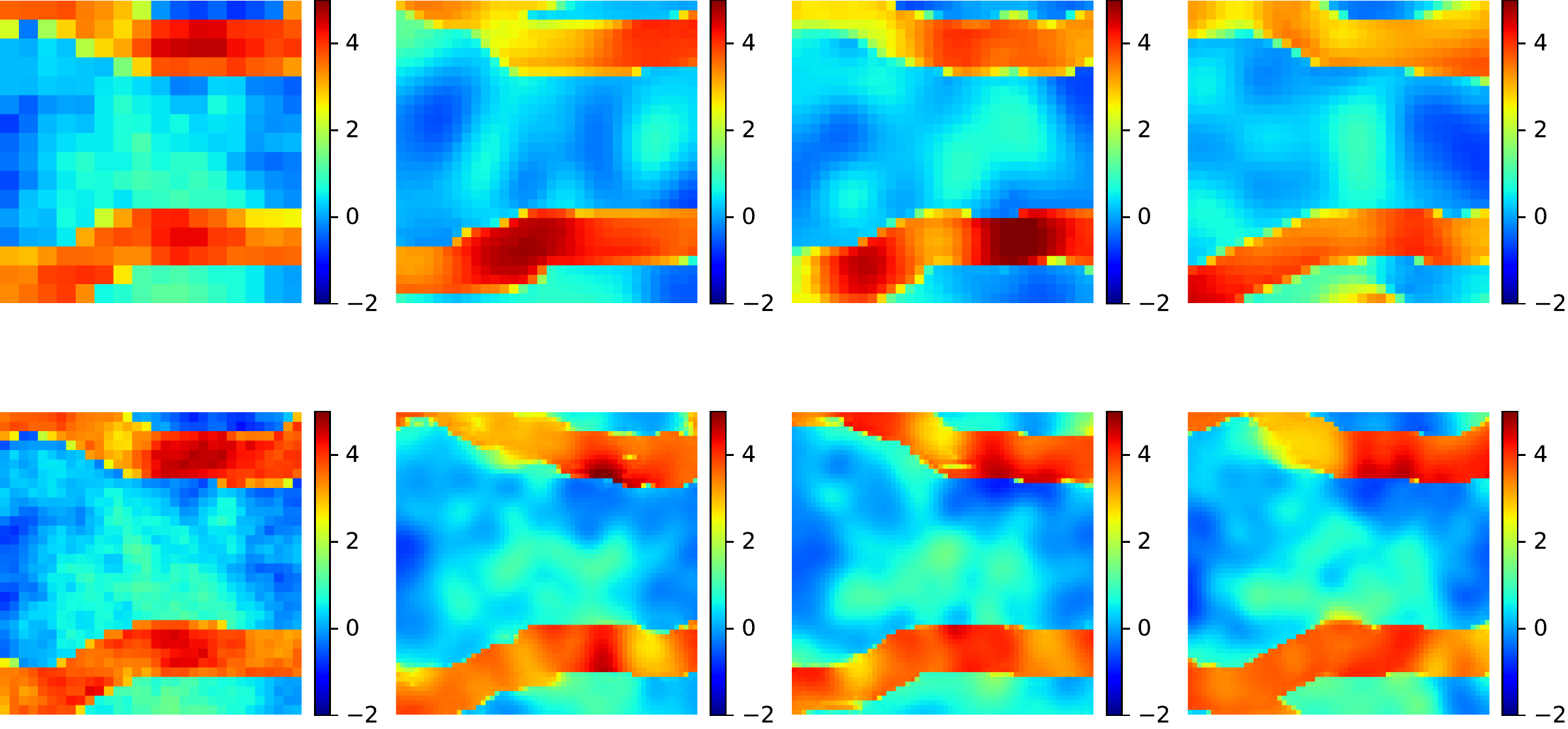}}
    \caption{(a) The $16-64$ MDGM. First row: $16 \times 16$ resolution  random samples using the generative model $p_{ \btheta_1}(\bx_1|\bz_1)$ by randomly sampling   $\bz_1\in \mathbb{R}^{16}$ from   $\mathcal{N}(\bm{0},\bm{I})$ as input. Second row: $64 \times 64$ resolution  random samples using the model $p_{ \btheta_2}(\bx_2|\bz_1,\bz_2^{\star})$. We let $\bz_1 = \mathop{\arg\max} q_{\bphi_1}(\bz_1|\bx_1)$, and $\bx_1$ is the first image in this row, and $z_2\in \mathbb{R}^{256}$ is randomly sampled from $\mathcal{N}(\bm{0},\bm{I})$. (b) The $16-32-64$ MDGM. The $16 \times 16$ resolution generative model is the same with the $16-64$ MDGM, shown in the first row of (a). First row: $32 \times 32$ resolution  random samples using the model $p_{ \btheta_2}(\bx_2|\bz_1,\bz_2^{\star})$. We let $\bz_1 = \mathop{\arg\max} q_{\bphi_1}(\bz_1|\bx_1)$, and $\bx_1$ is the first image in this row, and $z_2\in \mathbb{R}^{64}$ is randomly sampled from $\mathcal{N}(\bm{0},\bm{I})$. Second row: $64 \times 64$ resolution  random samples using the model $p_{ \btheta_3}(\bx_3|\bz_2,\bz_3^{\star})$. We let $\bz_2 = \mathop{\arg\max} q_{\bphi_2}(\bz_2|\bx_2)$, and $\bx_2$ is the first image in this row, and $z_3\in \mathbb{R}^{256}$ is randomly sampled from $\mathcal{N}(\bm{0},\bm{I})$.  }
    \label{Channel_MDGM}
\end{figure}

\noindent \emph{Three scales $(16-32-64)$.} In the second example with three scales, the parameters with $16 \times 16$, $32 \times 32$ and $64 \times 64$ grids are denoted by $\bx_1, \bx_2, \bx_3$, respectively. The model $p_{ \btheta_1}(\bx_1|\bz_1)$ is the same as the above two-scales example as shown in the first row of Fig.~\ref{Channel_MDGM}(a). The first row of  Fig.~\ref{Channel_MDGM}(b) gives  the results of the model $p_{ \btheta_2}(\bx_2|\bz_1,\bz_2^{\star})$, where $\bx_2 \in \mathbb{R}^{32 \times 32}$, $\bz_1 \in \mathbb{R}^{16}$ is encoded from the first image in this row using the model $q_{\bphi_1}(\bz_1|\bx_1)$, and $\bz_2^{\star} \in \mathbb{R}^{64}$ is sampled from $\mathcal{N}(\bm{0},\bm{I})$. It takes about $1.3$ hours to train $50$ epochs with $\tilde{\beta} = 0.7$ in Eq.~\eqref{eq:MDGMloss}. Compared with the two-scales case, one can notice that the generated samples have highly consistent channels that inherit from the first image but maintain local diversity. In order to sample the desired parameters with $64 \times 64$ grid, we also adopt Algorithm~\ref{multi_vae_algorithm} with $\tilde{\beta} = 0.7$ in Eq.~\eqref{eq:MDGMloss} to train the model $p_{ \btheta_3}(\bx_3|\bz_2,\bz_3^{\star})$, where $\bz_2 \in \mathbb{R}^{80}$, $\bz_3^{\star} \in \mathbb{R}^{256}$, and the training time is $3.75$ hours. The pre-trained model $q_{\bphi_2}(\bz_2|\bx_2)$ is the input to encode the given coarse-scale training data, where $\bz_2 = (\bz_1, \bz_2^{\star})$. The samples using the pre-trained model $p_{ \btheta_3}(\bx_3|\bz_2,\bz_3^{\star})$ are the last three samples shown in the second row of  Fig.~\ref{Channel_MDGM}(b), where $\bz_2$ is encoded from the first image in this row. As with the previous examples, the channels basically coincide with those in the first image. However, the difference is that the generated samples not only retain the global features of the encoded coarse-scale image but also discover local details. 

In summary, both the two- and three-scales models can generate samples with the correct spatial distribution of channels in the coarsest-scale with the $16 \times 16$ grid. For the fine-scale parameter generation, the low-dimensional latent variables  encode the information of the coarse-scale parameter and define the global features of the fine-scale parameter. The high-dimensional latent variables capture fine-details and provide diversity in the generated samples.

\subsubsection{The inversion results and discussion} 
Once all generative models are trained, we can use them in the Bayesian inversion process. To assess the efficiency and accuracy of the proposed multiscale method, we first  evaluate the posterior distribution using the reference single-scale method. Algorithm~\ref{mcmc} using the pre-trained model $ \bx=\mu_{\btheta}(\bz)$ as the input can perform the random walk in the latent space, where $\bx \in \mathbb{R}^{64 \times 64}$ and $\bz \in \mathbb{R}^{256}$. Although the dimension of the latent variable $\bz$ is still high, the convergence can be realized by constructing a Markov chain with the length of $30000$. The inferred results are depicted in Fig.~\ref{channel_reference}. The estimated mean and samples mostly match the exact channels and some important local details.

\begin{figure}[!htbp]
	\centering
	\includegraphics[width=0.8\linewidth, height = 0.5\linewidth]{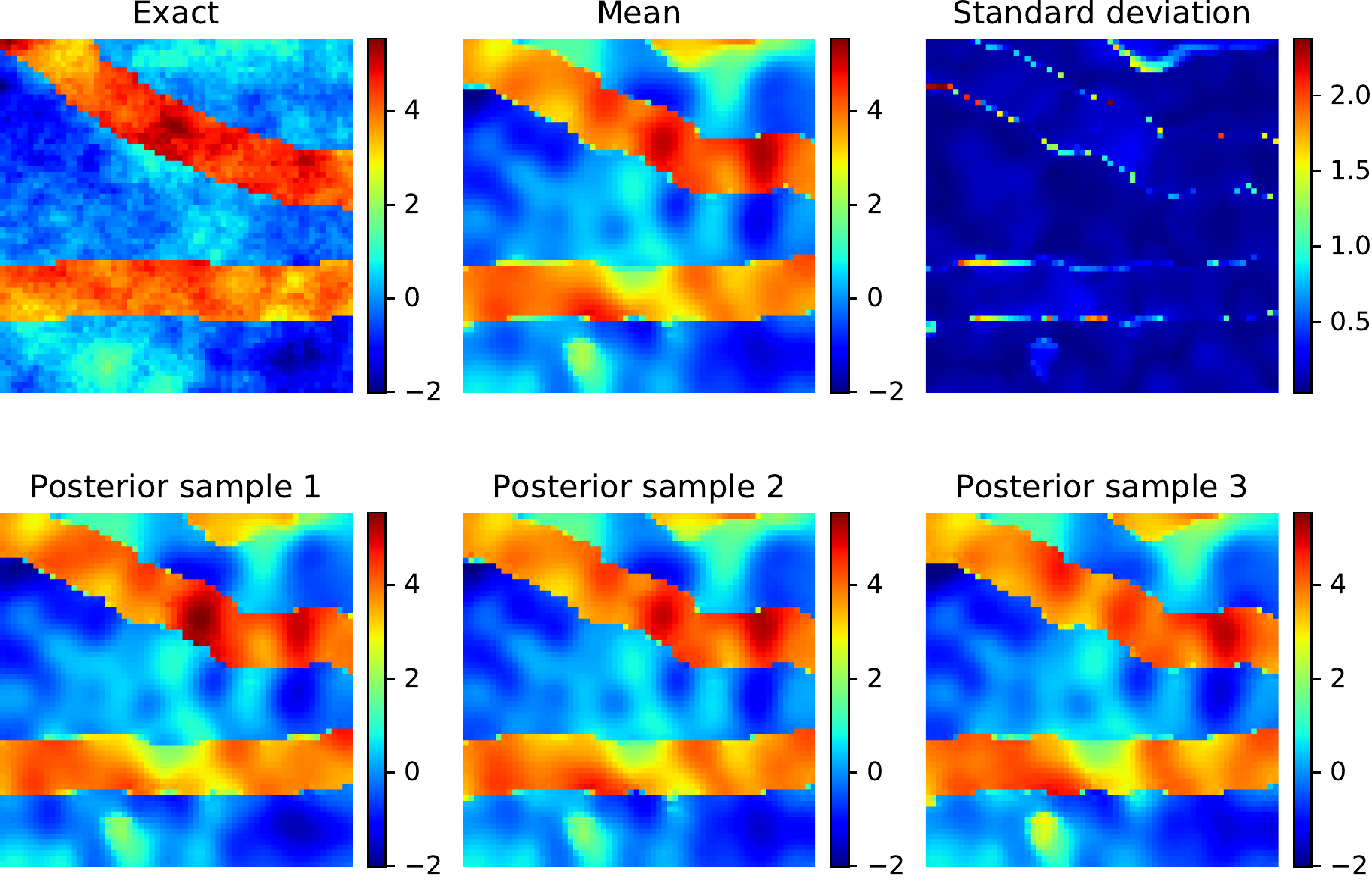}
	\caption{The referenced single-scale estimation result in the desired $64 \times 64$ grid.  Inference with respect to  $\bz \in \mathbb{R}^{256}$ is performed using Algorithm~\ref{mcmc}.}
	\label{channel_reference}
\end{figure}

\noindent \emph{Two scales $(16-64)$.}
For the two-scales inversion example, one needs to estimate successively the parameter with the pre-trained generative model in the coarse-scale with a $16 \times 16$ grid and fine-scale with a $64 \times 64$ grid. The coarse-scale estimation uses Algorithm~\ref{mcmc} and the pre-trained model $\bx_1 = \mu_{\btheta_1}(\bz_1)$, where $\bx_1 \in \mathbb{R}^{16 \times 16}$ and $\bz_1 \in \mathbb{R}^{16}$. The length of the Markov chain for  $\bz_1$ is $7000$. The posterior log-permeability fields are illustrated in Fig.~\ref{channel_16}. Obviously, the method identified all channel locations, which is the most important information in channelized parameter estimation. Similar to the Gaussian case shown in Fig.~\ref{Gaussian-16}, the estimated variance in the coarse-scale is very low since slightly varying a global feature of the log-permeability will greatly impact the value of the pressure field. Based on the coarse-scale estimation, the refinement results with $64 \times 64$ grid are shown in  Fig.~\ref{channel_16-64}. This only needs to run $7000$ iterations using  Algorithm~\ref{multi_mcmc} with the pre-trained model $\bx_2 =\mu_{\btheta_2}(\bz_2^{\star},\bz_1)$, where $\bx_2 \in \mathbb{R}^{64 \times 64}$, $\bz_2^{\star} \in \mathbb{R}^{256}$, and $\bz_1 \in \mathbb{R}^{16}$. The fine-scale local details are similar to those in the exact log-permeability. However, benefited from the coarse-scale estimation, the calculation saved a lots of computational cost requiring a reduced number of   forward model evaluations the  $64 \times 64$ grid.  

\begin{figure}[!htbp]
	\centering
	\includegraphics[width=0.8\linewidth, height = 0.5\linewidth]{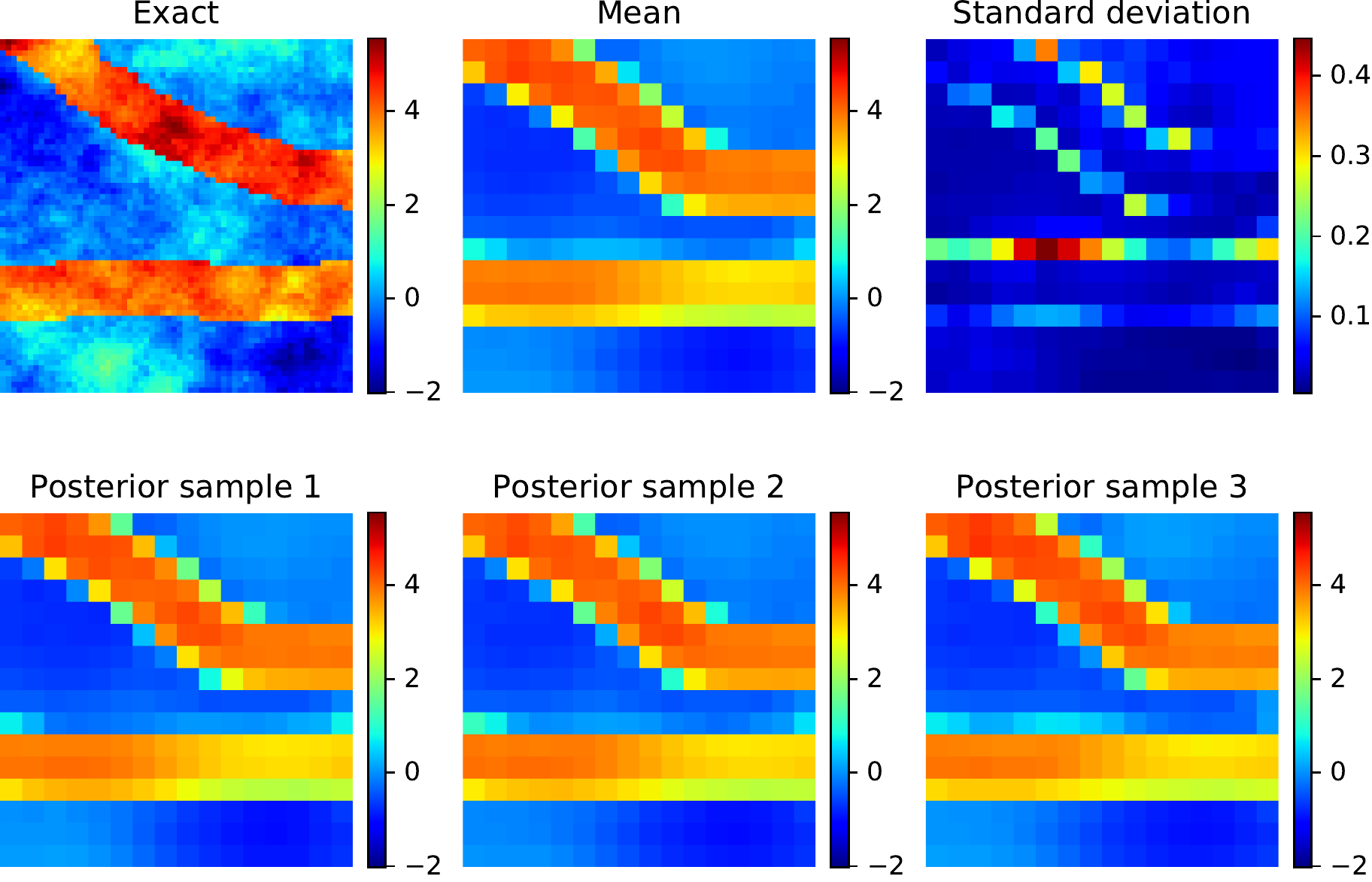}
	\caption{The estimation result with $16 \times 16$ grid. The results only involve inference in the coarsest scale with $16 \times 16$ grid, which needs to be corrected and refined in the finer-scale. Inference is performed with respect to  $\bz_1 \in \mathbb{R}^{16}$ using Algorithm~\ref{mcmc}.}
	\label{channel_16}
\end{figure}

\begin{figure}[h]
	\centering
	\includegraphics[width=0.8\linewidth, height = 0.5\linewidth]{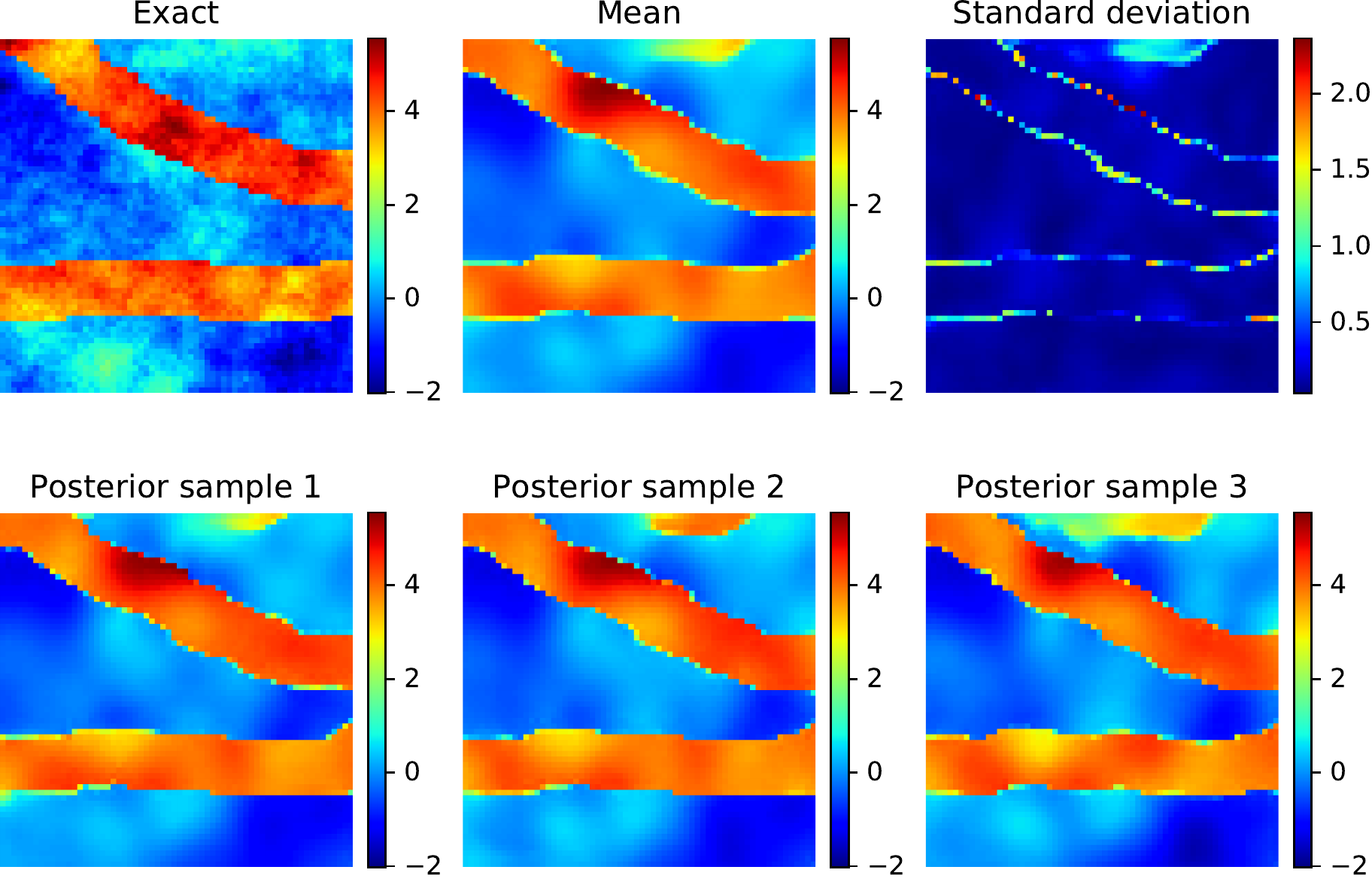}
	\caption{The two-scale estimation result with the desired $64 \times 64$ grid. The results involve  inference across scales. Inference is performed with respect to  $\bz_1 \in \mathbb{R}^{16}$ and $\bz_2^{\star} \in \mathbb{R}^{256}$ using Algorithm~\ref{multi_mcmc}.}
	\label{channel_16-64}
\end{figure}

\noindent \emph{Three scales $(16-32-64)$.}
To discuss the impact of the number of scales, we also studied the three-scale estimation with pre-trained three-scale generative models. The coarsest scale result was shown in Fig.~\ref{channel_16}. To refine the estimated results in the scale with $32 \times 32$ grid, we  run $7000$ iterations in Algorithm~\ref{multi_mcmc} with the pre-trained model $\bx_2 = \mu_{ \btheta_2}(\bz_2^{\star},\bz_1)$, where $\bx_2 \in \mathbb{R}^{32 \times 32}$, $\bz_2^{\star} \in \mathbb{R}^{64}$, and $\bz_1 \in \mathbb{R}^{16}$. Unlike the estimation in the coarsest-scale that  can only identify the location of the channels, the refined results with the $32 \times 32$ grid can also provide a good estimation regarding the spatial distribution of high- and low-values  as shown in Fig.~\ref{3scales_channel_32}. Since most salient features were captured in previous scales, the desired scale estimation becomes much easier. We only need to perform $5000$ iterations in Algorithm~\ref{multi_mcmc} with the pre-trained model $\bx_3 = \mu_{ \btheta_3}(\bz_3^{\star},\bz_2)$ to guarantee its convergence, where $\bx_3 \in \mathbb{R}^{64 \times 64}$, $\bz_3^{\star} \in \mathbb{R}^{256}$, and $\bz_2 \in \mathbb{R}^{80}$. Fig.~\ref{3scales_channel_64} summarizes the final three-scales inference results where we can notice that the estimation is more accurate than the previous two experiments in both capturing the channels and the local permeability details.

\begin{figure}[h]
	\centering
	\includegraphics[width=0.8\linewidth, height = 0.5\linewidth]{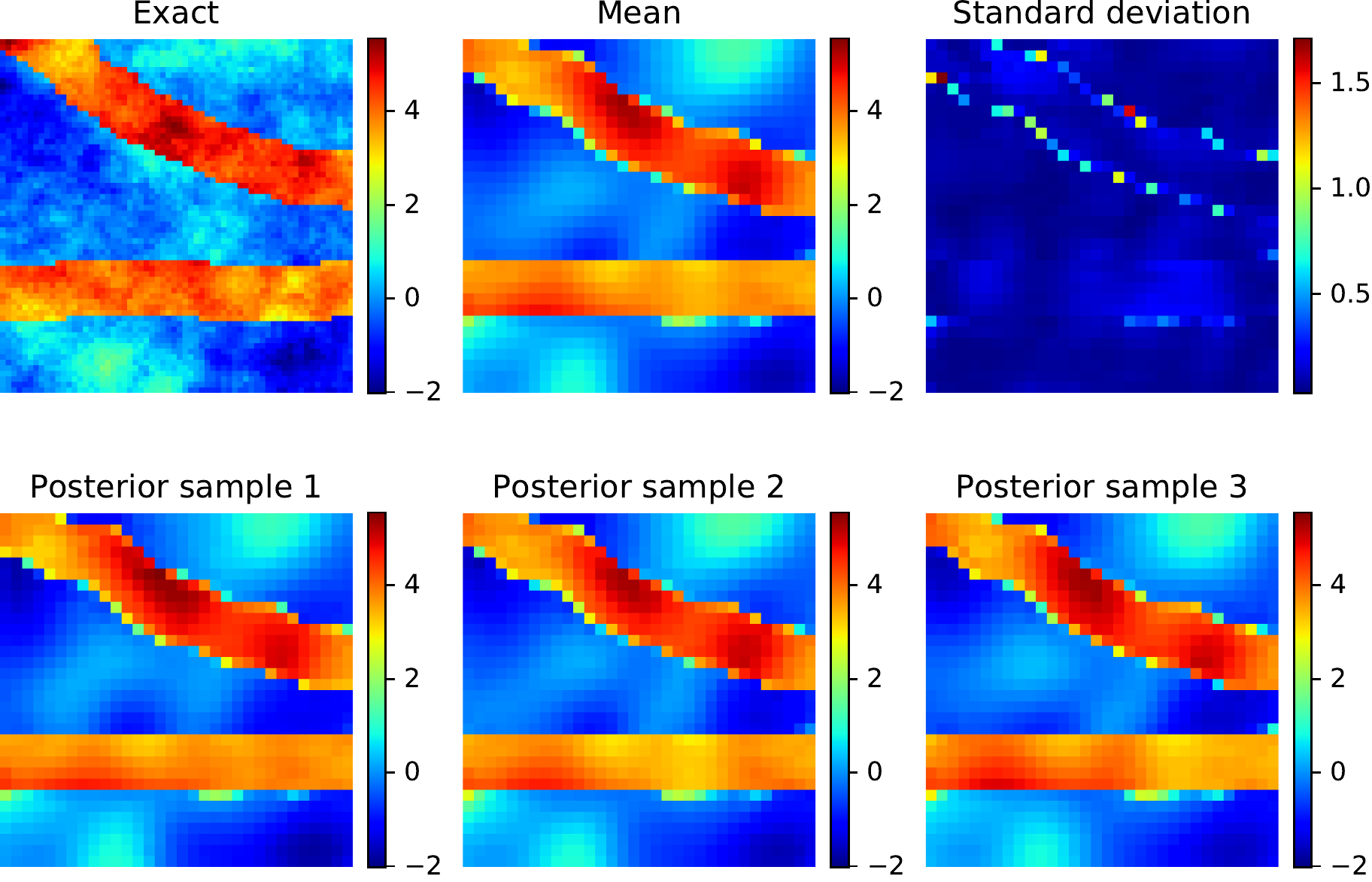}
	\caption{The three-scale estimation result with $32 \times 32$ grid, which needs to be corrected and refined in the finer-scale. The results involve   inference across scales. Inference is performed with respect to  $\bz_1 \in \mathbb{R}^{16}$ and $\bz_2^{\star} \in \mathbb{R}^{64}$ using Algorithm~\ref{multi_mcmc}.}
	\label{3scales_channel_32}
\end{figure}

\begin{figure}[h]
	\centering
	\includegraphics[width=0.8\linewidth, height = 0.5\linewidth]{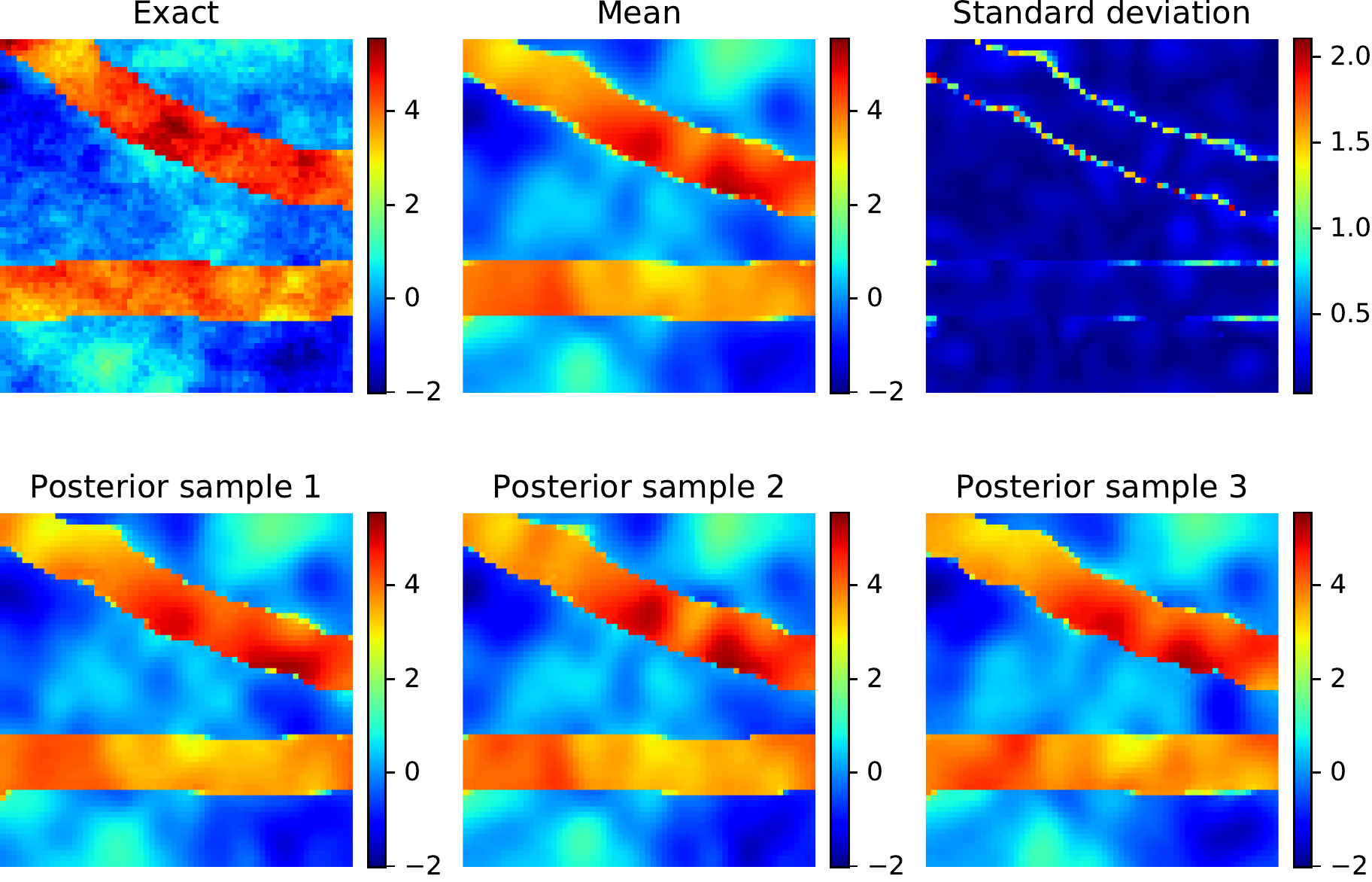}
	\caption{The three-scale estimation result with the desired $64 \times 64$ grid. The results involve  inference across scales. Inference was performed with respect to  $\bz_2 \in \mathbb{R}^{80}$ and $\bz_3^{\star} \in \mathbb{R}^{256}$ is performed using Algorithm~\ref{multi_mcmc}.}
	\label{3scales_channel_64}
\end{figure}

As with the previous examples, we are interested in the posterior estimation in the scale of the $64 \times 64$ grid. The estimated results with uncertainty from the left bottom corner to the right top corner of the log-permeability field are given in Fig.~\ref{Channel_error_bar}. It can be seen that the posterior means of all experiments are close to the exact value, while the local features estimation by the reference method is worse than the results by the multiscale method. For a channelized log-permeability, 
the location of the channels will greatly impact the predicted pressure values in the observation/sensor locations.  Unlike the parameter only decoded from one latent variable in the single scale method, the  latent variables played different roles in MDGM to exploit multiscale characteristics.  The channels were identified from the coarse-scale inference, whereas the fine-scale inference corrects and refines the coarse-scale estimation while exploring and learning local features.

\begin{figure}[h]
	\centering
	\subfloat[]{
    \includegraphics[width=0.3\textwidth]{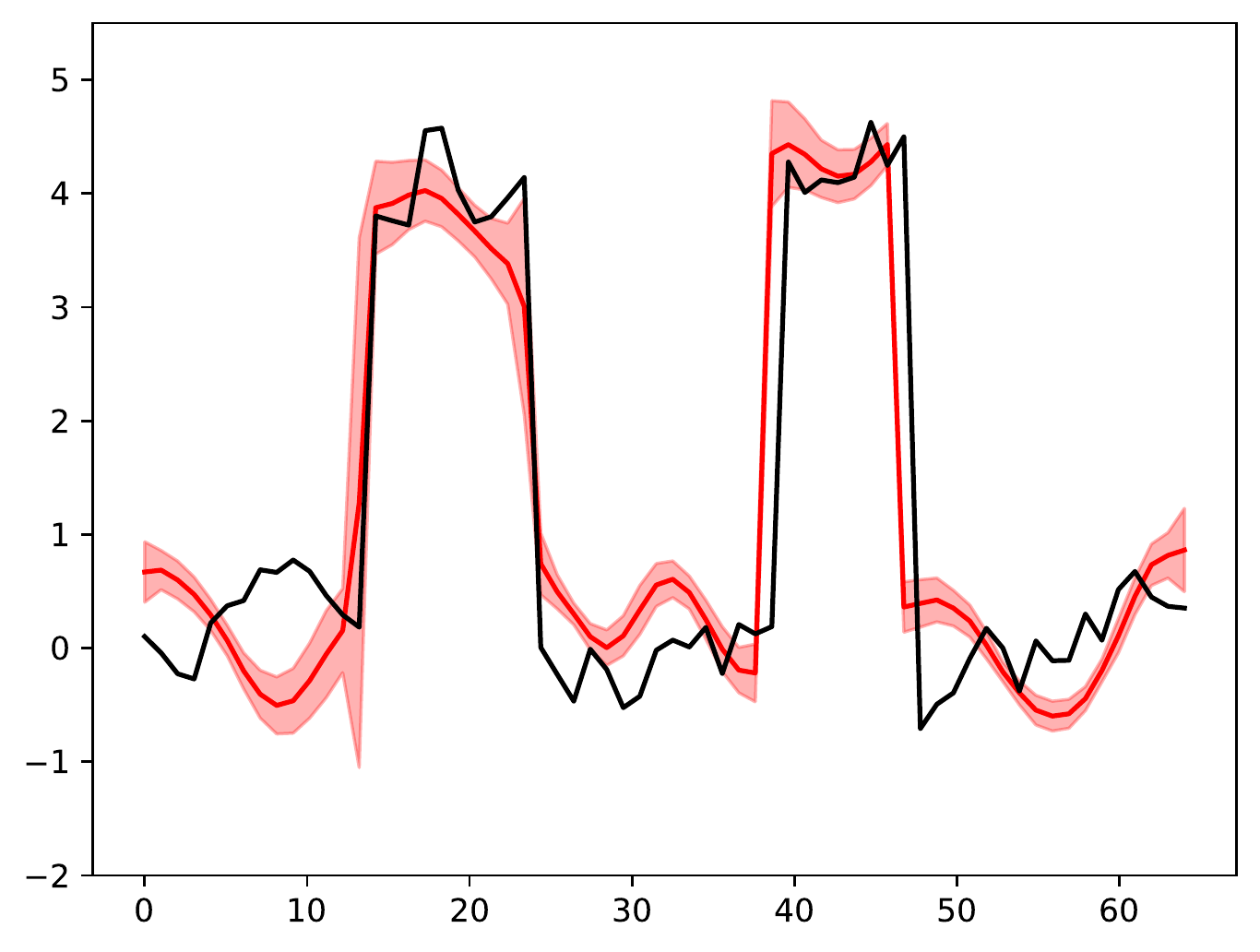}}
    \quad
    \subfloat[]{%
    \includegraphics[width=0.3\textwidth]{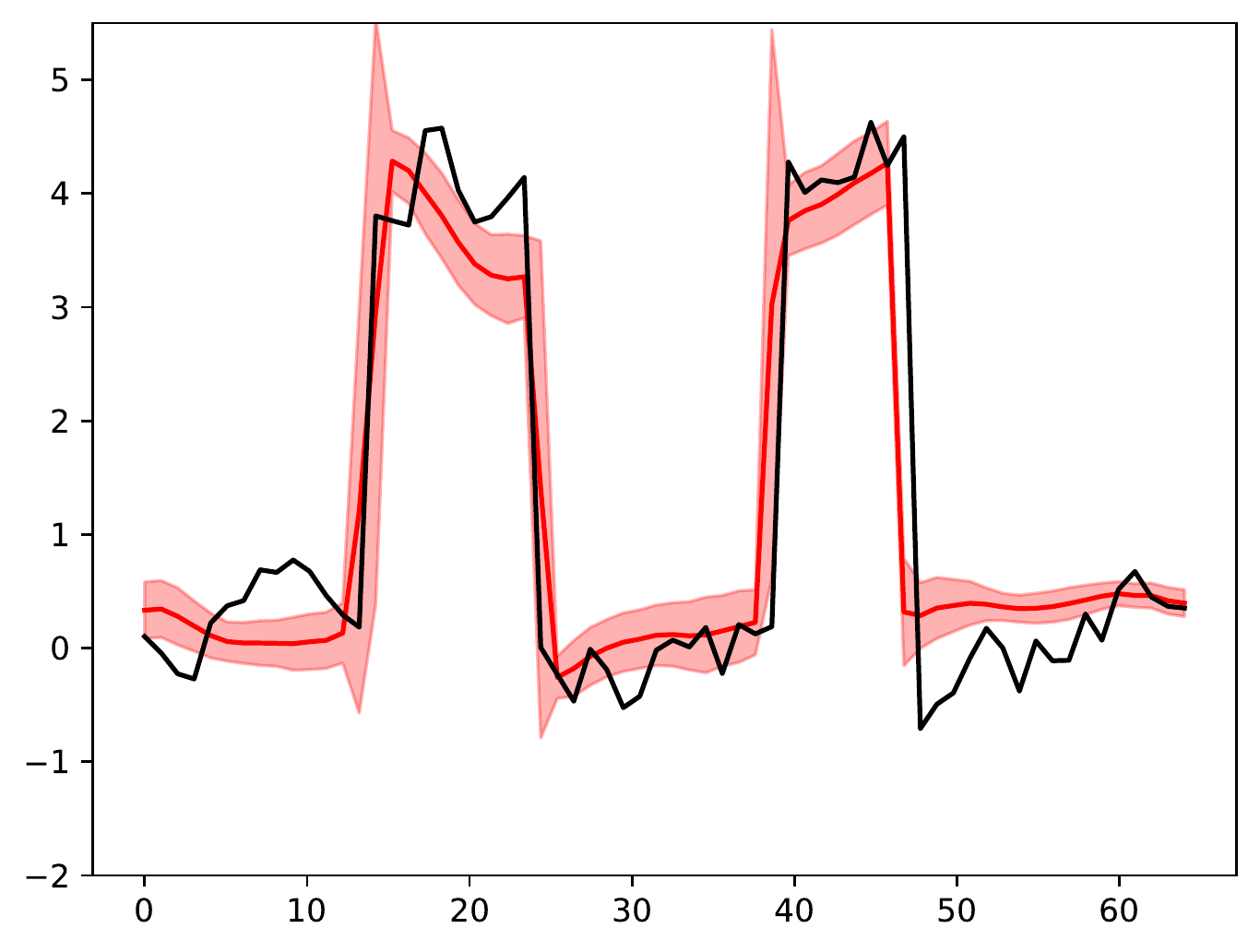}}
    \quad
    \subfloat[]{%
    \includegraphics[width=0.3\textwidth]{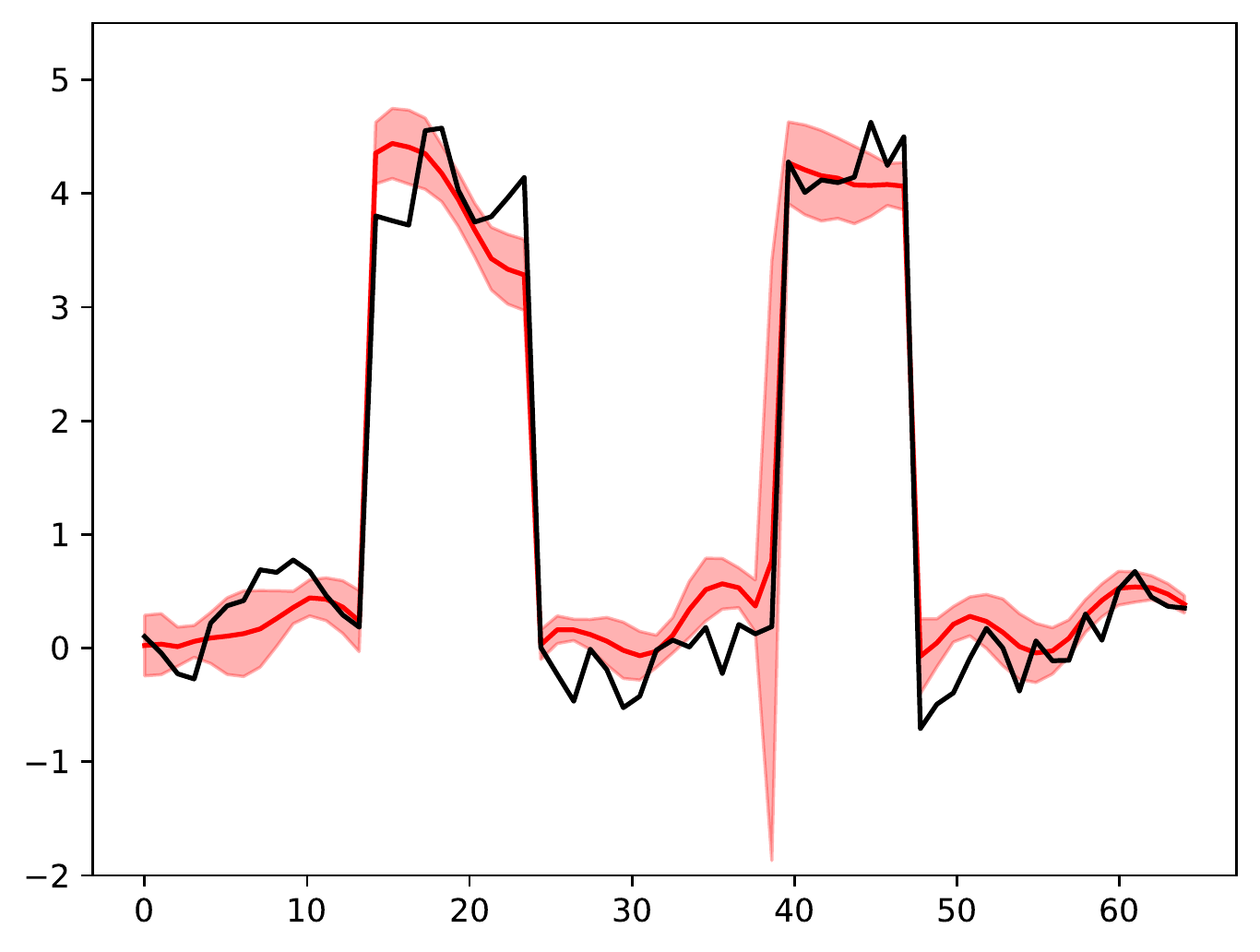}}
    \caption{The estimated log-permeability values in the line from the left bottom corner to the right top corner  on a  $64 \times 64$ grid. The results from left to right are obtained by  (a) reference $(64)$, (b) two-scales $(16-64)$, (c) three-scales $(16-32-64)$ experiment, respectively. The black and red lines show the true and the posterior mean log-permeability, respectively. The shaded region shows values within two standard deviations of the mean.}
    \label{Channel_error_bar}
\end{figure}

The computational cost of the forward model in each experiment is given in Table~\ref{Channel_cost}. For non-Gaussian random fields, the superiority of the proposed method is more prominent than the Gaussian random field case. The reference method takes about $2.9$  and $4.4$ more CPU time than  the two-scales and three-scales methods, respectively. The main difference among them is the computational cost in the desired scale. We designed the multiscale scheme to reduce the  computational cost in the fine-scale with some additional coarse-scale forward evaluations. As a consequence, this leads to significant computational savings during Bayesian inference. One can apply the proposed method to a computationally more intensive model. Correspondingly, more reduced computational time can be expected. The acceptance rate of each MCMC implementation is shown in Table~\ref{Channel_acceptance_ratio}. For the reference single-scale experiment, each dimension of the latent variable has equivalent importance in the generation of the log-permeability. At the same time, channels are the most salient features in this Bayesian inversion example. To infer a high-dimensional latent variable one needs to explore a high-dimensional state space, which will lead to a very low acceptance rate. For the multiscale method, the acceptance rate in the coarsest-scale is still very low. Based on the coarse-scale estimation, applying 
 Algorithm~\ref{multi_mcmc} to refine and correct the estimation will become much easier since two proposal distributions with different step sizes provide an informed and efficient exploration. Using a small step size for the low-dimensional latent variable to correct the global features and using a big step size for the high-dimensional latent variable to explore local features is goal-oriented, which not only improves the acceptance rate but also leads to a better estimation. 

\begin{table}[h]
	\caption{The iterations (its) and approximated cpu time in seconds(s) for solving the forward model in different experiments for the non-Gaussian random field  test problem.} 
	\centering	
	\begin{tabular}{ccccc}  
		\hline
		Experiment &  $16 \times 16$   & $32 \times 32$ &  $64 \times 64$ & Total\\ \hline
		\multirow{2}{*}{one scale ($64$)$^{\star}$} & - & - & $30000$ its & $30000$ its \\
		& - & - & $93000$ s & $93000$ s  \\\hline
		\multirow{2}{*}{two scales ($16-64$)} & $7000$ its & - & $10000$ its & $17000$ its \\
		& 910 s & - & $31000$ s & $31910$ s  \\\hline
		\multirow{2}{*}{three scales ($16-32-64$)} & $7000$ its & $7000$ its & $5000$ its & $19000$ its \\
		& $910$ s & $4690$ s & $15500$ s & $21100$ s \\\hline
	\end{tabular}
	\label{Channel_cost}
\end{table}

\begin{table}[h]
	\caption{The acceptance ratio of the MH algorithm in different experiments for the non-Gaussian random field test problem.} 
	\centering	
	\begin{tabular}{cccc}  
		\hline
		Experiment & $16 \times 16$   & $32 \times 32$ & $64 \times 64$\\ \hline
		one scale ($64$) & - & - &  5.47\%\\ \hline
		two scales ($16-64$) &2.13\% & - & 12.07\% \\\hline
		three scales ($16-32-64$) &2.13\%& 9.90\% & 40.28\% \\\hline
	\end{tabular}
	\label{Channel_acceptance_ratio}
\end{table}

As with the Gaussian case, we use three evaluation metrics to assess the convergence in the desired scale with the $64 \times 64$ grid. Fig.~\ref{convergence_Channel} indicates a better performance in the channelized log-permeability estimation. It is obvious that the reference single-scale method takes a long exploration to converge to the stationary distribution. As discussed earlier, the inferred latent variable of the single-scale inference is still high-dimensional, and each dimension keeps equivalent importance in the  log-permeability generation. In such scenarios, it is easy to get trapped in local modes. Fig.~\ref{Channel-state} provides the state evolution of the Markov chain to infer the log-permeability on the $64 \times 64$ grid. The single-scale method captures well the true solution at about the $15000$--th iteration, while the multiscale method only needs to refine the coarse-scale estimation for  fine-scale inference. By greatly reducing the fine-scale forward model evaluations, the computational burden  in Bayesian inverse problems is reduced. For a parameter like the  channelized log-permeability with obvious multiscale characteristics, it can readily be seen that the three-scales experiment performs better than the two-scales experiment with respect to stability, efficiency and accuracy.

\begin{figure}[h]
	\centering
	\subfloat[]{
    \includegraphics[width=0.3\textwidth]{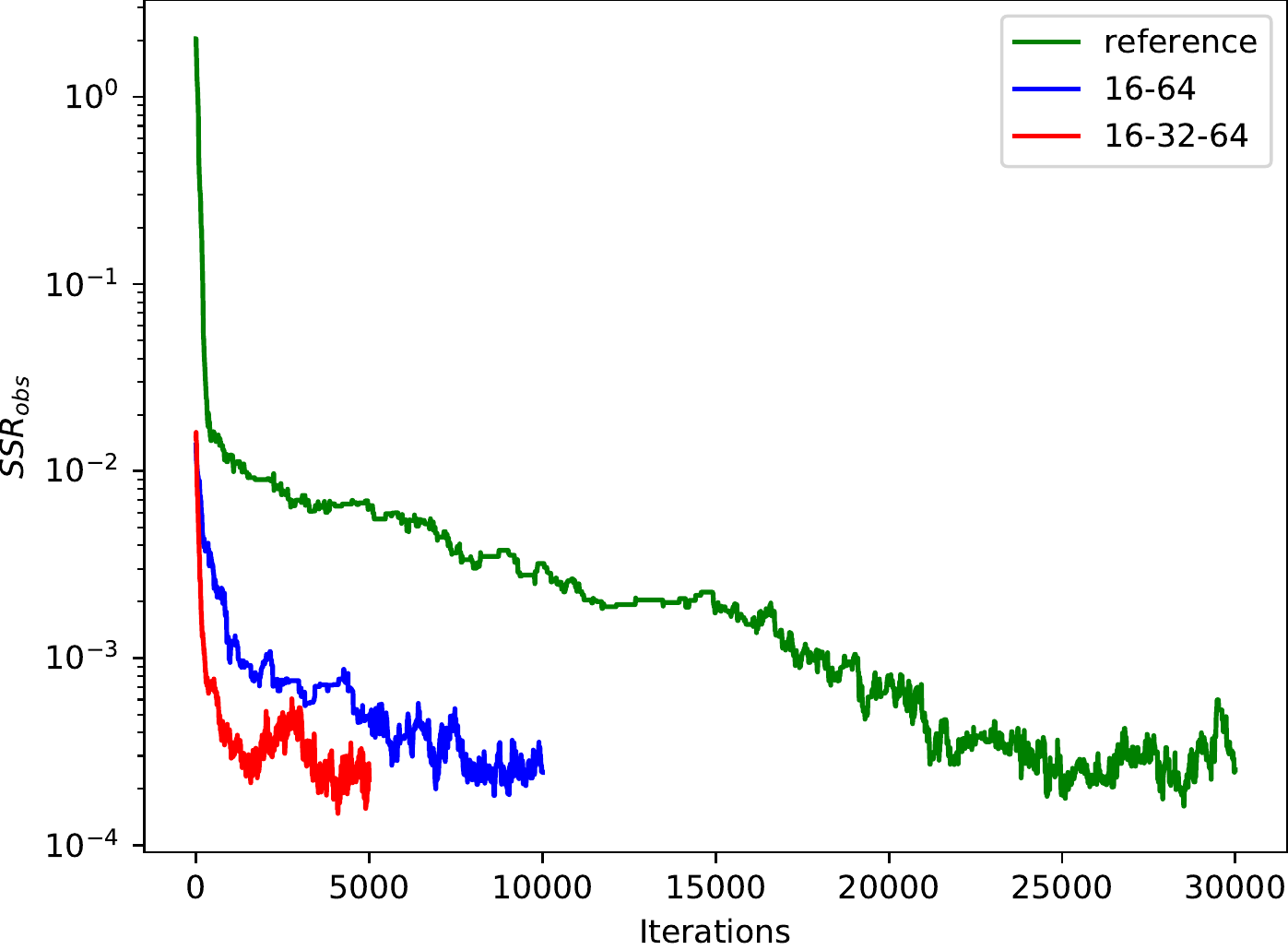}}
    \quad
    \subfloat[]{
    \includegraphics[width=0.3\textwidth]{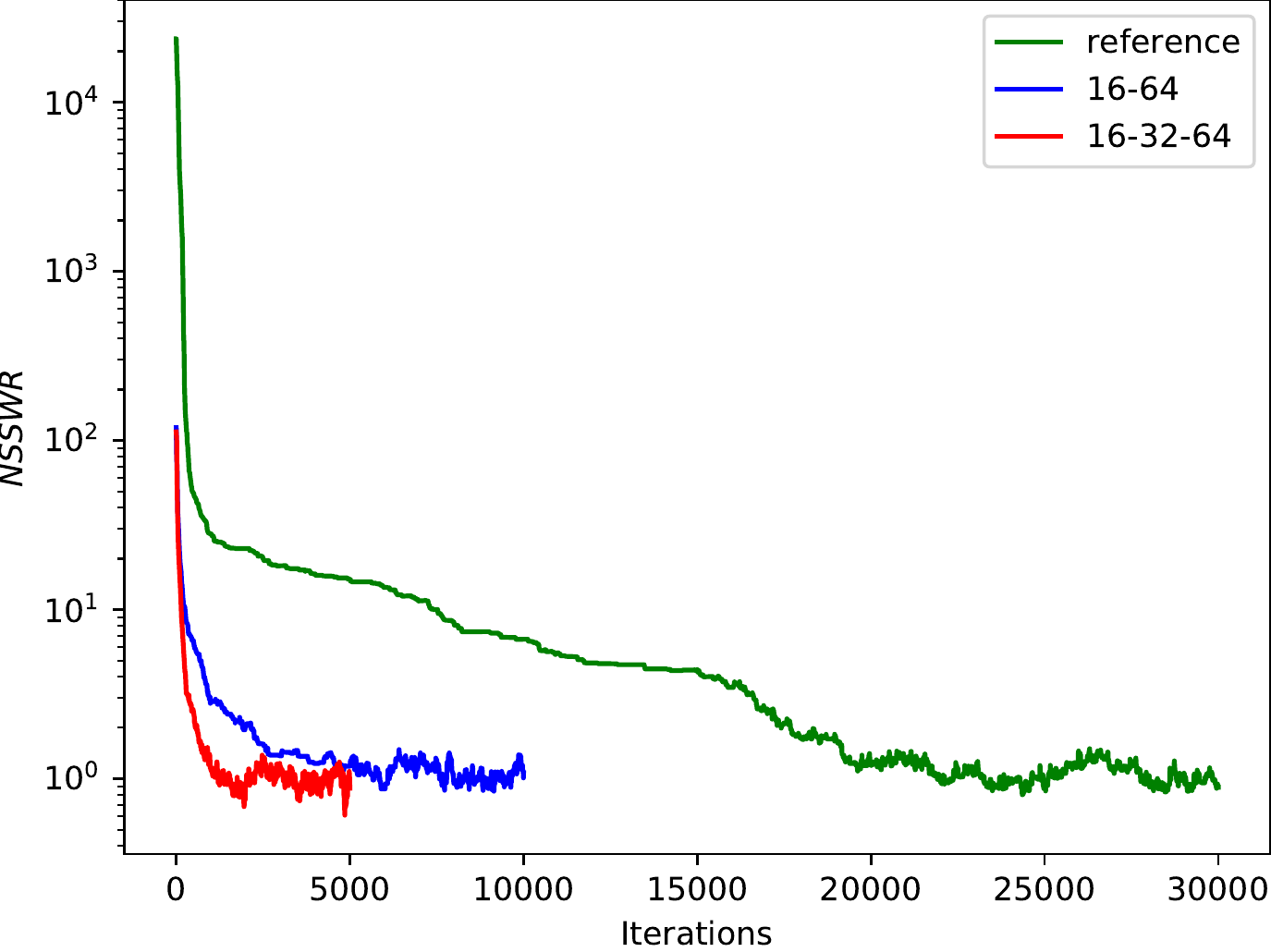}}
    \quad
    \subfloat[]{
    \includegraphics[width=0.3\textwidth]{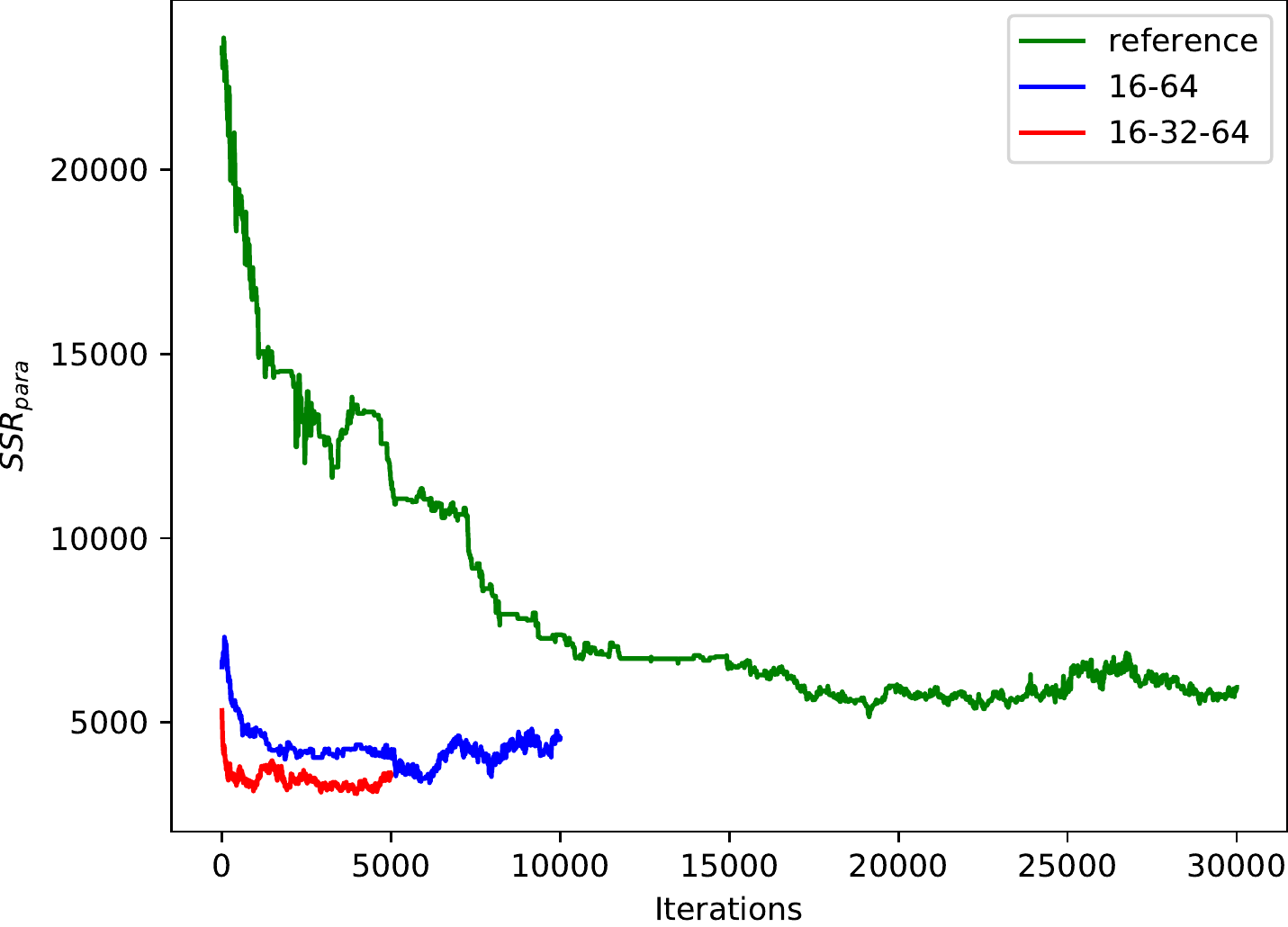}}
	\caption{The convergence of the Markov chain used for the non-Gaussian log-permeability field estimation with a $64 \times 64$ grid. The evaluation metrics from left to right are (a) the SSR of observable pressure values  (b) NSSWR values (c) the SSR of parameter field, respectively.}
	\label{convergence_Channel}
\end{figure}

\begin{figure}[h]
	\centering
	\includegraphics[width=1.0\linewidth, height = 0.6\linewidth]{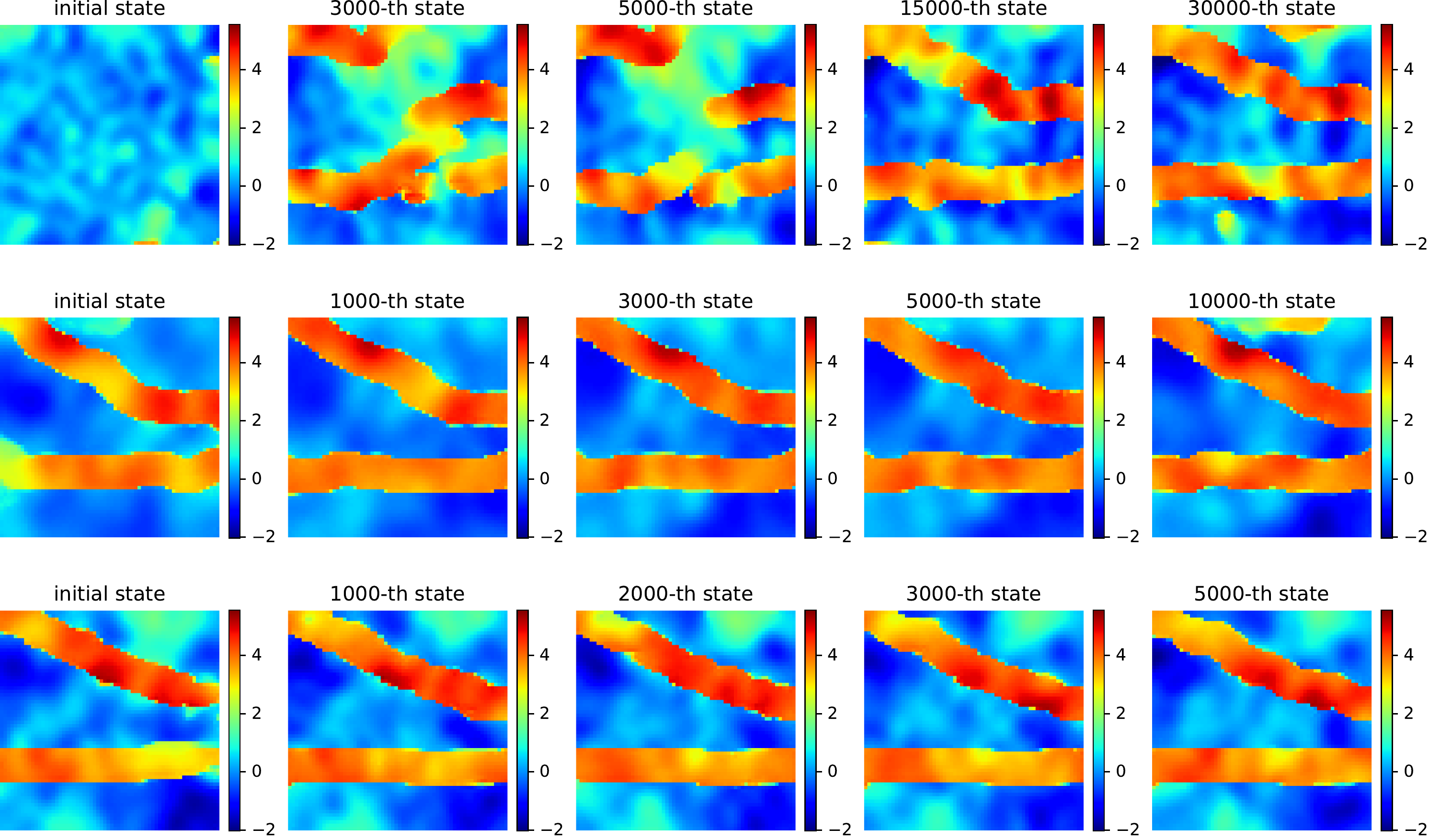}
	\caption{The states of the non-Gaussian log-permeability field with $64 \times 64$ grid in the Markov chain. The results from top to bottom row are obtained from (a) reference $(64)$, (b) two scales $(16-64)$, (c) three scales $(16-32-64)$ experiment, respectively. }
	\label{Channel-state}
\end{figure}

\section{Conclusions}\label{sec:Conclusions}

 In this work, we introduced a novel multiscale parameter estimation framework for Bayesian inverse problems based on a  multiscale deep generative model.   
 The deep generative model has been proven to be promising for the characterization of  complex spatially varying parameters. To exploit the multiscale characteristics, we extended the existing VAE-based deep generative model into a multiscale framework with multiple latent variables. Endowing the latent variables with different missions using training data at various scales, the low-dimensional latent variables can generate coarse-scales parameters and dominate the global features in finer-scale parameter generation, while the high-dimensional latent variables can enrich local details. We demonstrated the model with Gaussian and non-Gaussian parameter estimation. Combining pre-trained multiscale deep generative models with a multiscale inference strategy, we hierarchically performed inference from coarse- to fine-scale. 
 
 Benefited from the construction of the latent space in the multiscale generative model, the coarse-scale estimation explores in the low-dimensional latent space and searches for all possible global patterns by invoking the extremely cheap forward model. Using previous estimation results, the fine-scale estimation refines the parameters by correcting the global features and enriching the local features using the expensive fine-scale forward model. It was demonstrated that coarse-scale estimation information could pass across scales via the designed latent space, which plays an important role in accelerated convergence. In the two test cases, the proposed method shows superior performance over the reference single-scale method in computational cost and accuracy. We also discussed in the non-Gaussian case, the importance of the number of scales considered in the generative model and parameter estimation.

 Some challenges and extensions are worthy to explore in the future. The fundamental requirement for the proposed method is to train a stable and desired multiscale deep generative model, which involves different setups for various types of parameters, like the number of training data, the number of scales, hyperparameter selection, and so on. Further study of  the multiscale generative model has promising applications on  super resolution, multiscale uncertainty quantification, and so on. In addition, note that we use the simple Metropolis-–Hastings algorithm with pCN proposal distribution as the Bayesian inference method. Enhanced sampling techniques like sequential MC (SMC) that can realize parallel computation will result in accelerated exploration and high-efficiency.

\section*{Acknowledgements}
\noindent
N.Z. acknowledges support  from the Defense Advanced Research Projects Agency (DARPA) under the Physics of Artificial Intelligence (PAI) program (contract HR$00111890034$). 
Computing resources were provided by the AFOSR Office of Scientific Research through the DURIP program and by the University of Notre Dame's Center for Research Computing (CRC).

\appendix

\section{Neural network architectures for the encoder and decoder networks}\label{app:nn}
In this work, we use convolutional neural networks (CNNs)~\cite{krizhevsky2012imagenet} for the encoder and decoder models. CNNs are more effective in capturing multiscale features than fully-connected neural networks and allow modeling of the hierarchical nature of the features~\cite{zeiler2014visualizing}.  The implemented encoder and decoder neural networks~\cite{huang2017densely,zhang2018residual,he2016deep,mo2019integration}  are illustrated in Fig.~\ref{neural_networks}. The batch size is $64$ for all implementations.

\begin{figure}[!htp]
	\centering
	\includegraphics[width=1\linewidth, height = 0.89\linewidth]{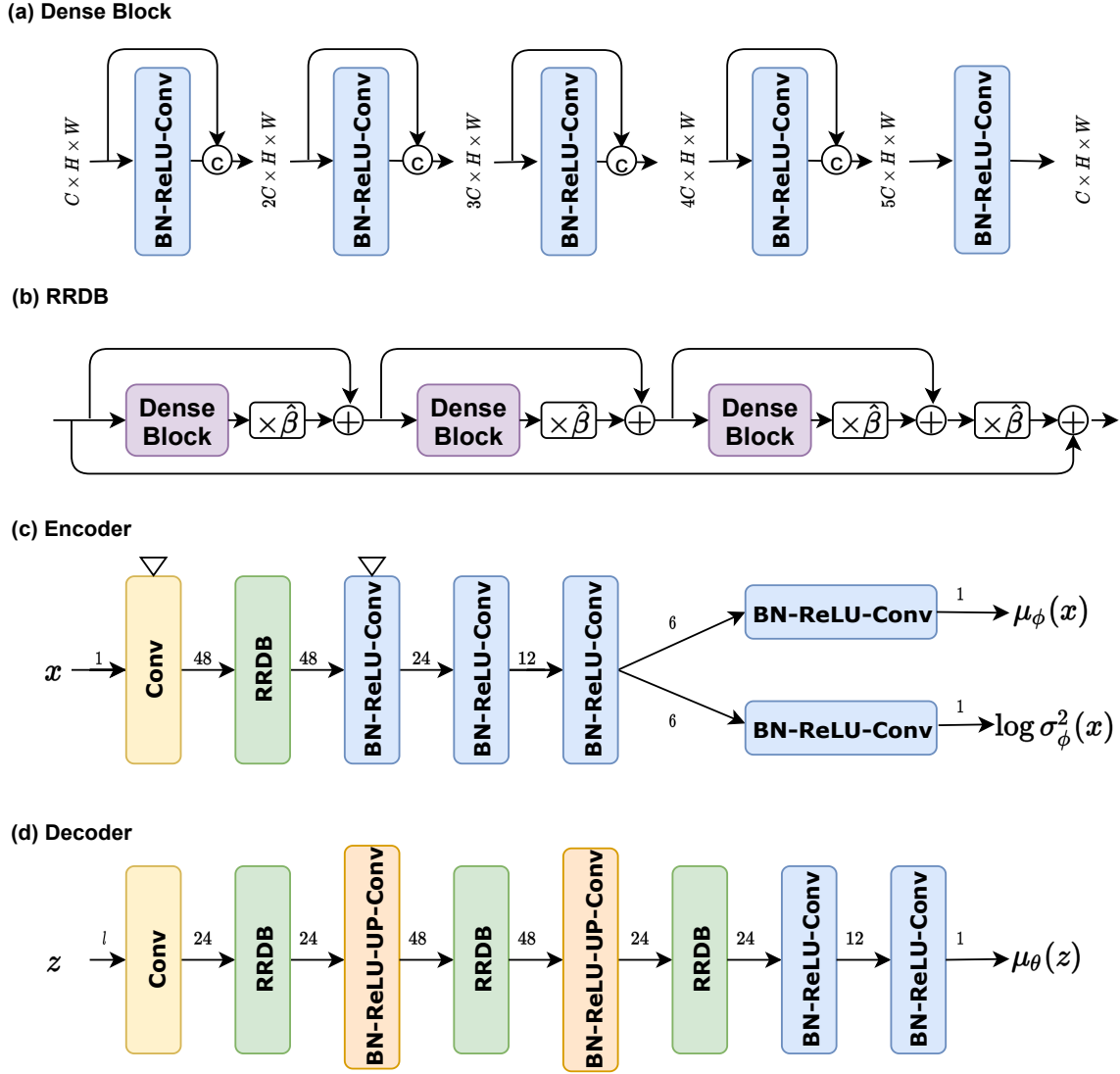}
	\caption{(a) Dense block with five layers. Its input consists of $C$ feature maps/channels with size $H \times W$. In each layer, its output is computed successively by three operators, i.e. Batch Normalization (BN), Rectified Linear Units (ReLU), and Convolution (Conv), where $C$ is specified above the arrows in the   encoder and decoder architectures in sub-figures (c) and (d). The output feature maps are  concatenated with the input feature maps. The concatenated feature maps are the input to the next layer.  (b) A residual-in-residual dense block (RRDB) using $3$ residual dense blocks. In each block, the output is multiplied by a constant $\hat{\beta}$ and then is added to the input with the result serving as the input for the next dense block. We let $\hat{\beta}$ be  $0.2$ in this paper.  (c) Encoder neural network architecture. The feature map size $H \times W$ is halved by Conv operator in $\triangledown$  with a stride $2$. (d) Decoder neural network architecture. The number of feature maps of the input is equal to the scale number $l$.  The feature map size $H \times W$ is doubled by applying the nearest upsampling (UP) operator. }
	\label{neural_networks}
\end{figure}

\section{Concatenation of latent variables}\label{app:cat}
In the MDGM, the $l$-th scale encoder network includes two parts i.e. the augmented encoder network and the $(l-1)$-th scale encoder network (see Fig.~\ref{multi_vae1}). Since the training  is recursive, the $(l-1)$-th encoder network also includes the augmented encoder network and the $(l-2)$-th scale encoder networks and so on. The latent variable in the $l$-th scale is $\bz_l = (\bz_1, \bz_2^{\star}, \dots, \bz_l^{\star})$. In this paper, all the augmented encoder  and decoder networks employ the same architectures  described in \ref{app:nn}, so the elements $(\bz_1, \bz_2^{\star}, \dots, \bz_l^{\star})$ in $\bz_l$ have proportional sizes  depending on their input size. For example, we can obtain $\bz_1 = q_{\bphi_1}(\bz_1 | \bx_1)$ and $\bz_2^{\star} = q_{\bphi_2^{\star}}(\bz_2^{\star} | \bx_2)$ using the encoder model, where $\bz_1 \in \mathbb{R}^{4\times4}$, $\bz_2^{\star} \in \mathbb{R}^{8\times8}$, $\bx_1 \in \mathbb{R}^{16\times16}$, and $\bx_2 \in \mathbb{R}^{32\times32}$.  The input size of the  $l$-th scale decoder network  should be $C \times H \times W$. For example, in the previous example, the size of $\bz_2$ is $2 \times 8 \times 8$), where the number of channels is equal to $C=l$ since $\bz_l$ is stacked by the outputs of $l$ encoders. Also, $H \times W$ is the size of $\bz_l^{\star}$, and $\bz_l$ must be reshaped as a tensor in such size.

To make  the other $(l-1)$ elements (i.e. $\bz_1, \bz_2^{\star}, \dots, \bz_{l-1}^{\star}$) in $\bz_l$ to 
be of consistent size with $\bz_l^{\star}$, we use the  \textsl{Upsample} operator \footnote{https://pytorch.org/docs/master/generated/torch.nn.Upsample.html} in the Pytorch library~\cite{paszke2019pytorch} over these elements and then concatenate\footnote{https://pytorch.org/docs/master/generated/torch.cat.html} all of them as input $\bz_l \in \mathbb{R}^{C \times H \times W}$ for the $l$-th scale decoder networks. The scale factor in the  \textsl{Upsample} operator depends on the output  and input sizes. Their sizes satisfy the following relationship:

\begin{equation}
    \begin{aligned}
    &H_{out} = H_{in} \times \text{scale factor}, \\
    &W_{out} = W_{in} \times \text{scale factor}, 
    \end{aligned}
\end{equation}
where $[H_{in} \times W_{in}]$ and $[H_{out} \times W_{out}]$ are the input and output sizes, respectively. We used the nearest mode in \textsl{Upsample} operator. A simple example is given in Fig.~\ref{cat_1} to illustrate this process   (the scale factor here is $2$).

\begin{figure}[H]
	\centering
	\includegraphics[width=0.42\linewidth, height = 0.2\linewidth]{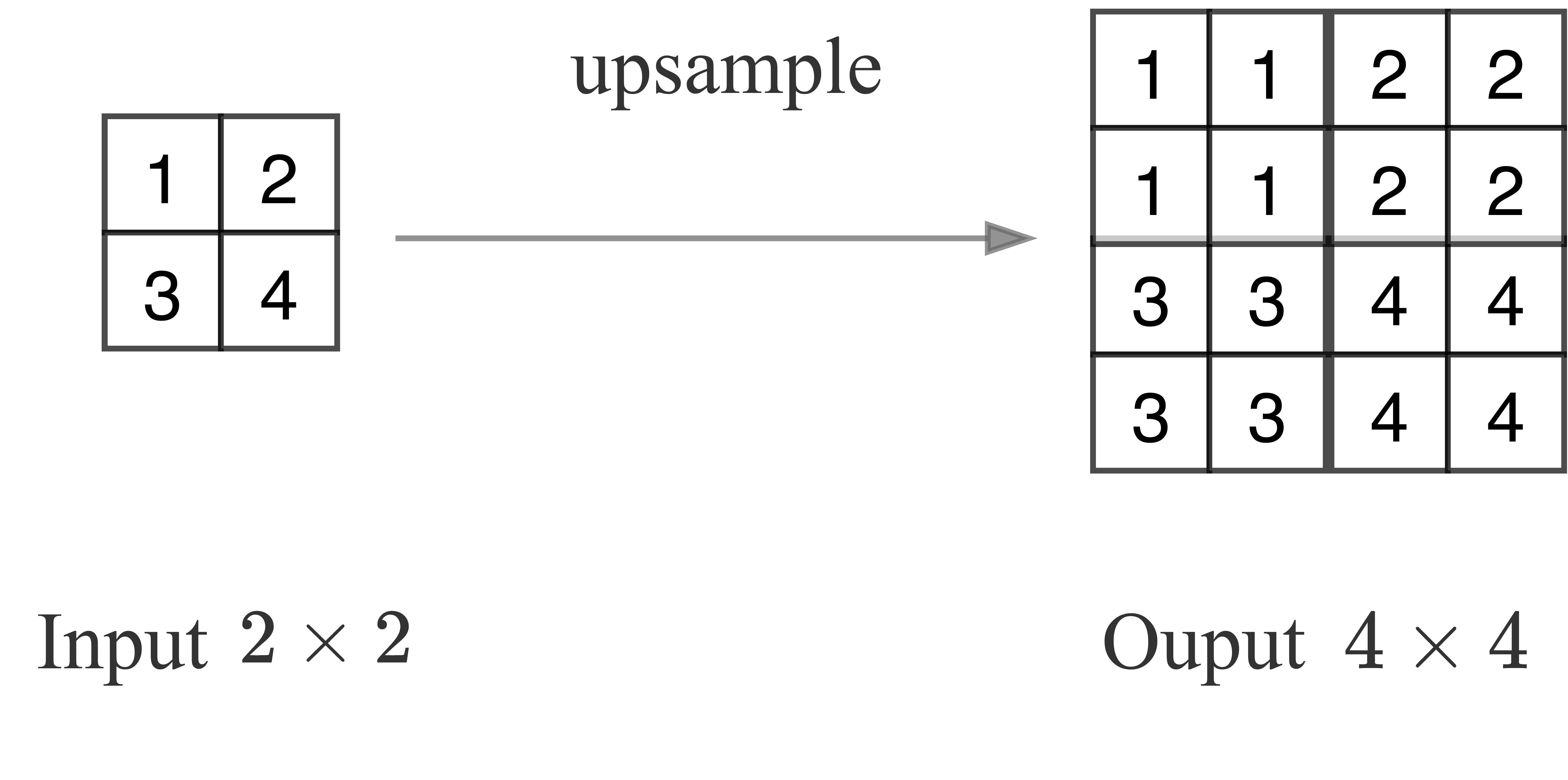}
	\caption{The \textsl{Upsample} operator example with nearest mode, where the input image size is $2 \times 2$  and the output size is $4 \times 4$.}
	\label{cat_1}
\end{figure}

\bibliography{Bibliography}

\end{document}

%% file: Definitions.tex
\newcommand\be{\begin{equation}}
\newcommand\ee{\end{equation}}

\newcommand\btheta{\boldsymbol{\theta}}

\newcommand\bphi{\boldsymbol{\phi}}

\newcommand\bz{\boldsymbol{z}}

\newcommand\bX{\boldsymbol{X}}
\newcommand\bx{\boldsymbol{x}}

\newcommand\bxi{\boldsymbol{x}^{(i)}}
\newcommand\bzi{\boldsymbol{z}^{(i)}}

\newcommand\inth{_{\btheta}}
\newcommand\inphi{_{\bphi}}